\documentclass[10pt,twocolumn,letterpaper]{article}

\usepackage[pagenumbers]{cvpr} 
\usepackage{algorithm}
\usepackage{algpseudocode}
\usepackage{amsmath}
\usepackage{float}
\usepackage{amssymb}      

\usepackage{graphicx}
\usepackage{multirow}
\usepackage{subcaption}
\usepackage{array}
\usepackage{ragged2e}
\usepackage{amsthm}

\usepackage{booktabs}     
\usepackage{caption}      

\usepackage{subfiles}

\numberwithin{equation}{section}
\newtheorem{lemma}{Lemma}[section]
\newtheorem{theorem}[lemma]{Theorem}
\newtheorem{proposition}[lemma]{Proposition}
\newtheorem{corollary}[lemma]{Corollary}
\newtheorem{remark}[lemma]{Remark}
\usepackage{longtable} 

\usepackage[dvipsnames]{xcolor} 
\usepackage{pifont}             

\newcommand{\cmark}{\textcolor{Green}{\ding{51}}} 
\newcommand{\xmark}{\textcolor{Red}{\ding{55}}}    

\usepackage[table]{xcolor}  
\definecolor{RankFirst}{RGB}{255,242,204}  
\definecolor{RankSecond}{RGB}{255,204,153} 
\definecolor{RankThird}{RGB}{221,235,247} 
\definecolor{coralPink}{HTML}{ED028C} 

\definecolor{cvprblue}{rgb}{0.21,0.49,0.74}
\usepackage[pagebackref,breaklinks,colorlinks,allcolors=cvprblue]{hyperref}

\def\paperID{*****} 
\def\confName{CVPR}
\def\confYear{2026}

\title{ChordEdit: One-Step Low-Energy Transport for Image Editing}

\author{
\href{https://orcid.org/0009-0006-2839-3901}{\textcolor{black}{Liangsi Lu}}$^{1}$, 
\href{https://orcid.org/0000-0001-6000-3914}{\textcolor{black}{Xuhang Chen}}$^{2}$, 
\href{https://orcid.org/0009-0001-2511-4763}{\textcolor{black}{Minzhe Guo}}$^{1}$, 
\href{https://orcid.org/0009-0000-9387-5233}{\textcolor{black}{Shichu Li}}$^{3}$, 
\href{https://orcid.org/0000-0002-0099-539X}{\textcolor{black}{Jingchao Wang}}$^{4}$, 
\href{https://orcid.org/0009-0009-3928-7495}{\textcolor{black}{Yang Shi}}$^{1\dagger}$\\
$^{1}$Guangdong University of Technology\quad
$^{2}$Huizhou University\\
$^{3}$Shenzhen University\quad
$^{4}$Peking University\\
{\small $^{\dagger}$ Corresponding author: sudo.shiyang@gmail.com}\\
{\small Project page: \textcolor{coralPink}{\url{https://chordedit.github.io}}}
}

\begin{document}

\twocolumn[{%
\maketitle

\vspace{-11mm}

\begin{center}
  \includegraphics[width=\textwidth]{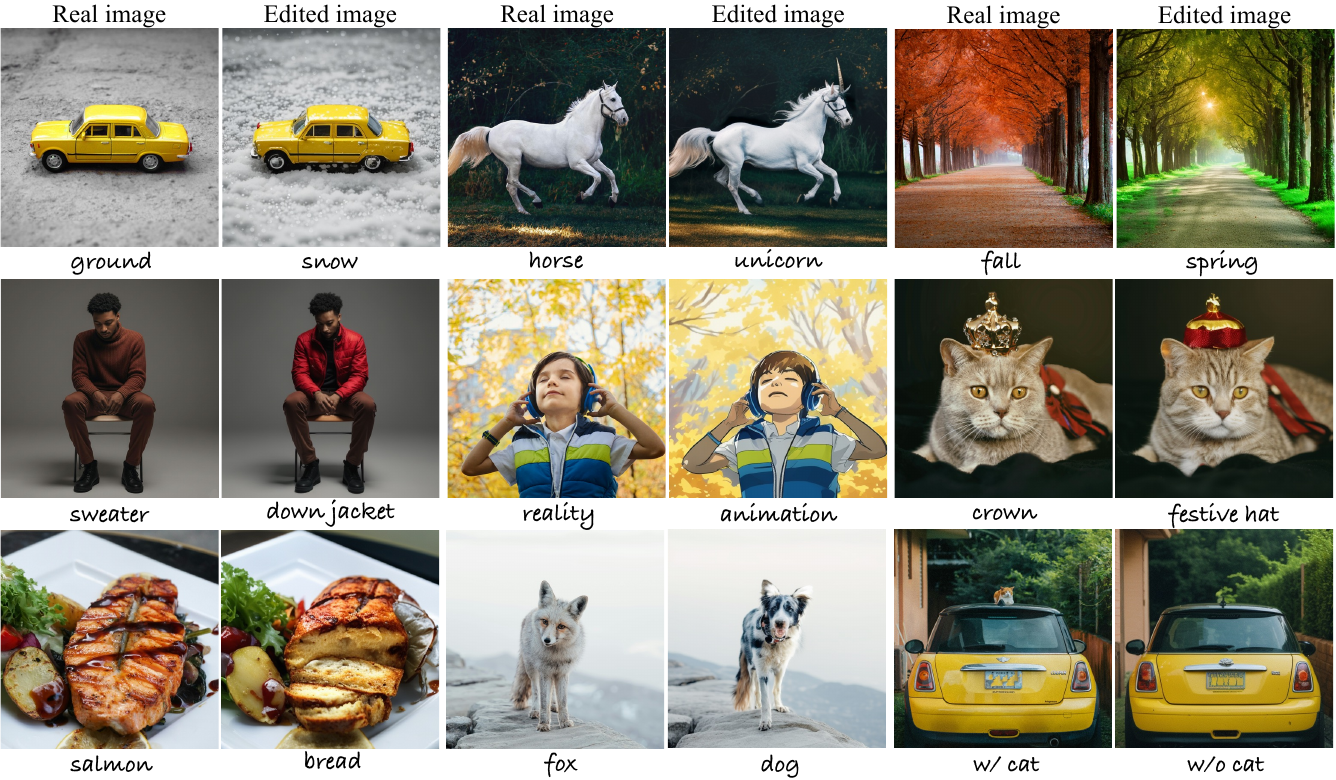}
  \vspace{-7mm}
  \captionof{figure}{\textbf{ChordEdit}. These examples demonstrate our model agnostic, training-free and inversion-free method operating on fast generative models. ChordEdit mitigates the failures of naive single-step editing by deriving a stable, low-energy control field based on optimal transport theory. This field's stability permits a single, large integration step, facilitating precise edits that preserve non-edited regions. Results shown use SD-Turbo (top two rows) and SwiftBrush-v2 (bottom row). Labels indicate the desired semantic change.}
  \label{fig:teaser}
\end{center}
}]

\begin{abstract}
The advent of one-step text-to-image (T2I) models offers unprecedented synthesis speed. However, their application to text-guided image editing remains severely hampered, as forcing existing training-free editors into a single inference step fails. This failure manifests as severe object distortion and a critical loss of consistency in non-edited regions, resulting from the high-energy, erratic trajectories produced by naive vector arithmetic on the models' structured fields. To address this problem, we introduce \textbf{ChordEdit}, a model agnostic, training-free, and inversion-free method that facilitates high-fidelity one-step editing. We recast editing as a transport problem between the source and target distributions defined by the source and target text prompts. Leveraging dynamic optimal transport theory, we derive a principled, low-energy control strategy. This strategy yields a smoothed, variance-reduced editing field that is inherently stable, facilitating the field to be traversed in a single, large integration step. A theoretically grounded and experimentally validated approach allows ChordEdit to deliver fast, lightweight and precise edits, finally achieving true real-time editing on these challenging models.
\end{abstract}

\section{Introduction}

The advent of one-step text-to-image (T2I) models, such as SD-Turbo~\cite{Sauer2024ADD}, SwiftBrush-v2~\cite{Dao2025SwiftBrushV2} and InstaFlow~\cite{liu2023instaflow}, has introduced a new paradigm of real-time image synthesis. By distilling large-scale diffusion models~\cite{song2023consistency,lin2024sdxl} into a compact, single-step inference pathway~\cite{liu2022flow,lipman2022flow,song2023consistency,luo2023latent,xu2023inversion}, these models offer unprecedented speed, promising truly interactive applications. This progress naturally raises the expectation that this real-time capability can be directly leveraged for the nuanced task of text-guided image editing.

However, this promise for flexible real-time text-guided image editing remains unmet. Existing one-step text-guided method~\cite{nguyen2025swiftedit} achieves fast performance by training dedicated networks, sacrificing model-agnostic flexibility and relying on precise inversion. A more flexible alternative, the training-free and inversion-free approach, typically computes an editing field by differencing the drifts conditioned on source and target prompts~\cite{xu2023inversion, kulikov2025flowedit}.
Despite its efficacy in traditional multi-step generators, this simple drifts approach fails when forced into the one-step models. The failure manifests as severe object distortion, where the edited entity is warped beyond recognition, and a critical loss of consistency in non-edited regions, causing the background and surrounding structures to disintegrate. These failure modes are visualized in Figure~\ref{fig:naive_show}. The root cause lies in the editing field computed via naive differencing. While one-step models are distilled to create stable, direct paths from noise to image, this distillation process yields a highly non-linear and sensitive mapping from the text condition to the vector field. Consequently, the naive editing field is inherently unstable, representing the arithmetic difference of two large-magnitude, divergent trajectories, resulting in an erratic, high-energy control field. Applying this volatile field in a single, large integration step tends to accumulate significant error, causing the observed distortions.

\begin{figure}[t]
    \centering
    \includegraphics[width=\linewidth]{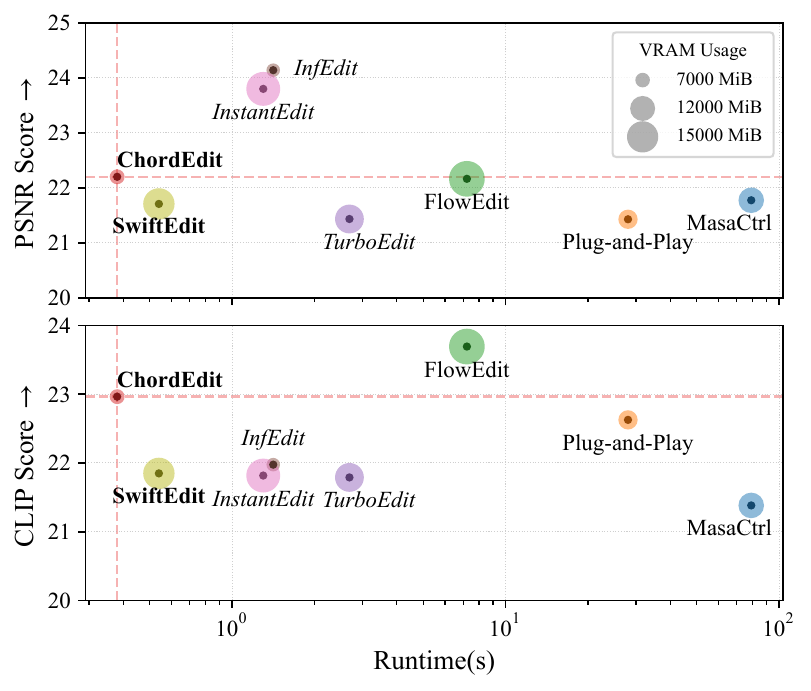}
    \vspace{-20pt}
\caption{Comparing ChordEdit (SD-Turbo) against \textbf{one-step}, \textit{few-step}, and multi-step editing methods on PIE-bench~\cite{ju2023direct}, evaluating performance on background consistency (PSNR), semantic alignment (CLIP, referring to CLIP-Edited)~\cite{radford2021learning}, and Runtime. Our method facilitates real-time text-guided editing while yielding highly competitive results.}
    \label{fig:sota_scatter}
    
\end{figure}

\begin{figure}[t]
\vspace{-8pt}
    \centering
    \includegraphics[width=0.75\linewidth]{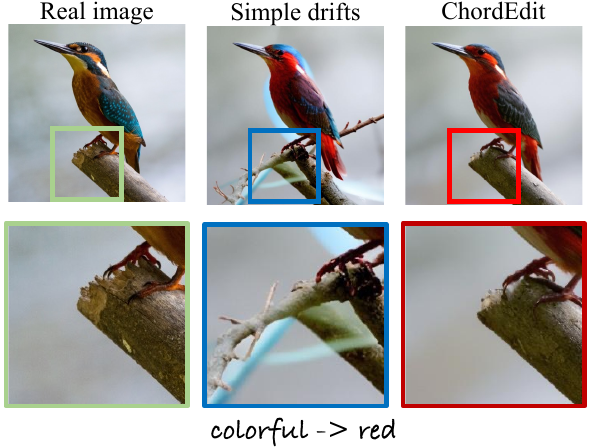}
  \vspace{-4pt}
    \caption{\textbf{One-Step Simple drift editing fails. ChordEdit preserves structure.} Simple drifts, a direct drift-difference from a one-step model, induce a high-energy, non-smooth vector field, yielding two disqualifying failures: (\textit{i}) severe \textbf{object distortion} and (\textit{ii}) \textbf{ background breakup} and spurious structures. Zoomed crops (bottom) highlight the distortions in \emph{Simple drifts} versus the faithful, photorealistic result of \emph{ChordEdit}.}
    \label{fig:naive_show}
\end{figure}

To overcome these challenges, we introduce \textbf{ChordEdit}, a training-free, inversion-free and lightweight method that facilitates high-fidelity on fast T2I models. We adopt a different perspective from simple vector arithmetic and recast the editing problem from the principled perspective of dynamic optimal transport (OT)~\cite{benamou2000computational}, seeking a low-energy chord to transport the source image distribution to the target. Our key contribution is the Chord Control Field, a theoretically-grounded, time-weighted average of the source and target drifts that replaces the instantaneous, erratic drift difference. This formulation acts as a potent temporal smoothing operator, yielding an inherently stable, low-energy field that can be traversed in a single, large integration step. Following this transport, a lightweight proximal refinement can be optionally applied to enhance target semantics. Our framework operates as a black-box by querying the model's velocity or equivalent field to compute our distinct control field, ensuring model agnosticism. Our experimental results on the PIE-bench~\cite{ju2023direct} benchmark demonstrate that our principled low-energy field addresses the core instability of one-step editing, achieving state-of-the-art efficiency while maintaining high background preservation and semantic fidelity.

\section{Related Work}

We provide a full discussion of related work in the Appendix. Prior work includes GAN-based editing~\cite{hertz2022prompt,tumanyan2023plug} and various diffusion/flow-based editors built on fast T2I models~\cite{song2023consistency, lin2024sdxl,Sauer2024ADD,liu2023instaflow, Dao2025SwiftBrushV2,kimconsistency}. These editors often require iterative, multi-step inversion~\cite{hertz2022prompt,tumanyan2023plug, mokady2023null,npi25,ju2023direct,proxedit24, edict23,aidi23,spdinv24,huberman2024edit} or few-step acceleration~\cite{deutch2024turboedit,gong2025instantedit,cora25,xu2025textvdb,si2025pix2pixzerocon,parmar2023zero}, making real-time interaction infeasible.

A central challenge lies in one-step editing. Training-free differential methods~\cite{meng2021sdedit, diffedit22}, such as InfEdit~\cite{xu2023inversion}, FlowEdit~\cite{kulikov2025flowedit}, are stable when averaged over multiple steps but collapse in the single-step limit due to high energy and variance. These failure modes are visualized in Figure~\ref{fig:naive_show}. Conversely, methods such as SwiftEdit~\cite{nguyen2025swiftedit} achieve one-step performance by training a dedicated inversion network, sacrificing model-agnostic flexibility~\cite{icd24, kawar2023imagic, zu2024cotflow}. \textbf{ChordEdit} operates in this challenging training-free, inversion-free, single-step regime. Instead of relying on multi-step averaging or trained inverters, we introduce a \emph{Chord control field} to stabilize the transport, achieving a low-energy, low-variance edit.

\begin{figure*}[t]
    \centering
    \includegraphics[width=0.92\linewidth]{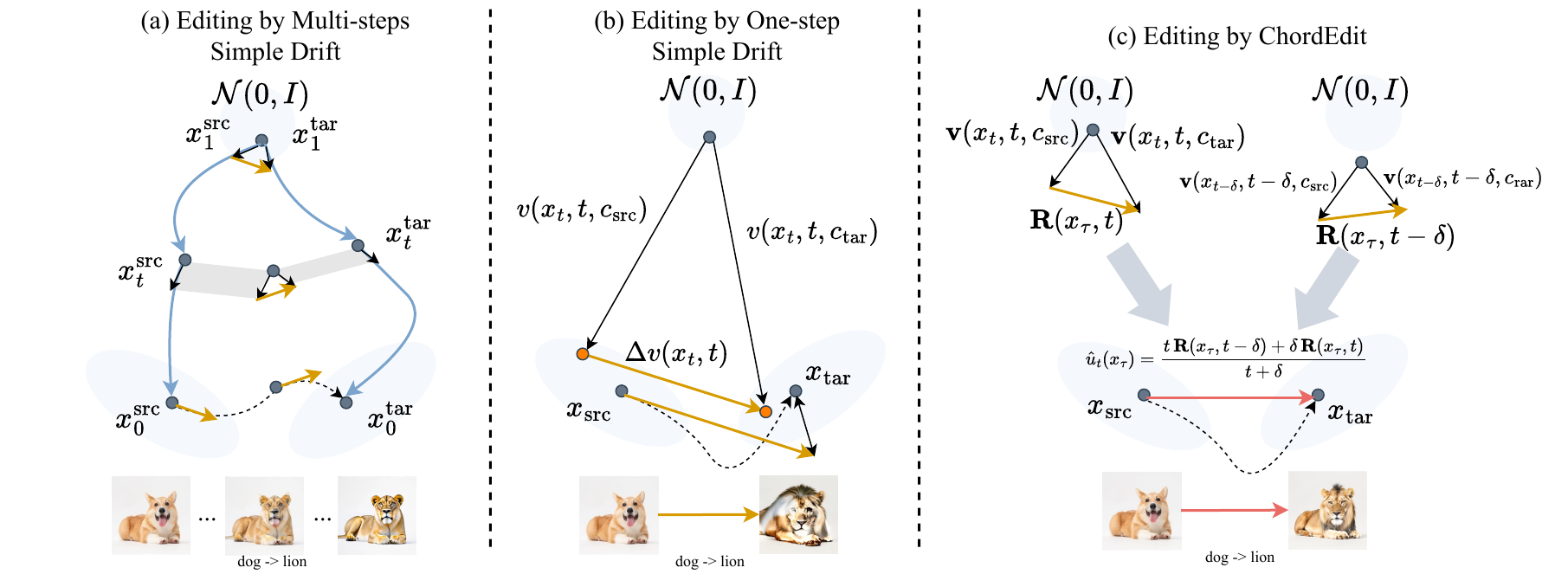}
    \vspace{-9pt}
\caption{\textbf{Comparison of editing field stability.} (a) \textbf{Multi-step Simple Drift}: In conventional multi-step diffusion, the iterative application of the simple drift $\Delta v$ ensures a stable trajectory. (b) \textbf{One-step Simple Drift}: In distilled models, the naive field $\Delta v(x_t, t)$ is high-energy and volatile. A single, large integration step (solid arrow) accumulates significant error and deviates significantly, as the erratic underlying path (dashed) confirms. (c) \textbf{Editing by ChordEdit (Ours)}: We derive a stable, low-energy \textbf{Chord Control Field} by time-averaging the \textit{observable} fields $\mathbf{R}(x_\tau, t)$ and $\mathbf{R}(x_\tau, t-\delta)$. This smoothed field facilitates an accurate, single-step transport (red arrow) that faithfully reaches the target $x_{\rm tar}$.} 
    \label{fig:method}
\end{figure*}

\section{Preliminaries}
\label{sec:preliminaries}
\subsection{Conditional Probability Flow and the Editing Problem}
\label{sec:preliminaries-flow}
Let $t\in[0,1]$ be the time, $x_t$ be the image state, and $c$ be the text condition.
A pre-trained text-to-image model induces a conditional probability flow with drift $v(x_t,t,c)$, defined as
\begin{equation}
\frac{d x_t}{dt}=v(x_t,t,c).
\end{equation}

We denote the distribution of $x_t$ as $p_t(x\mid c)$, where $p_1$ is the data distribution and $p_0$ is the prior distribution.
Given prompts $c_{\rm src}$ and $c_{\rm tar}$, we denote an initial image $x_{\rm src}:=x_1\sim p_{1}(x\mid c_{\rm src})$ and a edited image $x_{\rm tar}:=x_0\sim p_{0}(x\mid c_{\rm tar}).$
Editing amounts to transporting $x_1$ to $x_0$ by modifying the source flow with the instantaneous residual
\begin{equation}
\Delta v(x_t,t)\;=\;v(x_t,t,c_{\rm tar})-v(x_t,t,c_{\rm src}),
\end{equation}
which is the ideal continuous-time control aligning the two conditional dynamics.

\subsection{Observable Model at Noisy States}
Section~\ref{sec:preliminaries-flow} defines the ideal flow state $x_t$. We fix the editing anchor to the clean source state $x_\tau := x_1$. In practice we cannot access $x_t$ at arbitrary $t$, we therefore query the model at a synthetic noisy proxy $z \sim K_t(\cdot\mid x_\tau)$ drawn from a forward noising kernel $K_t(\cdot\mid x_\tau)$ that mimics the noise level at time $t$. Let $Q(z,t,c)$ denote the model’s observable output. For source/target prompts $c_{\rm src},c_{\rm tar}$ define the conditional residual $\Delta Q(z,t)\;=\;Q(z,t,c_{\rm tar})-Q(z,t,c_{\rm src})$.
Let $\mathcal{B}_t$ be a \emph{time-only linear map} from the codomain of $Q(\cdot,t,\cdot)$ to a fixed comparison space $U$ (e.g.\ drift/velocity units). The observable proxy field is
\begin{equation}
\mathbf{R}(x_\tau,t)\;=\;\mathbb{E}_{\,z\sim K_t(\cdot\mid x_\tau)}\!\big[\,\mathcal{B}_t\,\Delta Q(z,t)\,\big].
\label{eq:proxy-field}
\end{equation}

The expectation $\mathbb{E}[\cdot]$ is over the kernel randomness. In practice we use shared-noise Monte Carlo. If the model exposes the drift directly, we take $\mathcal{B}_t\equiv I$ and Eq.\eqref{eq:proxy-field} reduces to $\mathbb{E}[\Delta v(\cdot,t)]$. Although the theory is continuous in $t$, we evaluate $Q$ on a discrete grid $\{t,t-\delta\}$. This discretization does not alter the definitions above.

\subsection{Model Parameterizations and the Observable Output \texorpdfstring{$Q$}{Q}}
We instantiate $Q$ and $\mathcal{B}_t$ for common one-step models. In all cases $\mathcal{B}_t$ is linear and time-dependent only.
For noise-prediction models, such as SD-Turbo~\cite{Sauer2024ADD},
\begin{equation}
Q(z,t,c)=\hat\epsilon_\theta(z,t,c),\qquad
\mathcal{B}_t(\hat\epsilon)=A_t^{(\epsilon)}\,\hat\epsilon,
\end{equation}
where $A_t^{(\epsilon)}$, a function of the schedule $\alpha_t,\sigma_t$, maps to the drift/velocity comparison domain $U$ and $\hat\epsilon$ is predicted noise. Closed-form coefficients for $A_t^{(\epsilon)}$ are listed in the Appendix. For velocity models, such as InstaFlow~\cite{liu2023instaflow},
\begin{equation}
Q(z,t,c)=\mathbf{v}_\theta(z,t,c),\qquad \mathcal{B}_t\equiv I,
\end{equation}
where $\mathbf{v}_\theta$ is velocity prediction. Other parameterizations, such as score-to-drift, $x_0$ heads, or consistency models, fit the same template via a time-only linear $\mathcal{B}_t$. Details are deferred to the Appendix.

\section{ChordEdit}
\label{sec:methodology}

Our goal is a training-free, inversion-free editing scheme that remains stable under one-step inference. A Key challenge is that the ideal editing field is unknown, and its naive proxy is high-energy and irregular. As shown in Figure~\ref{fig:method}, we derive a low-energy estimator by integrating the dynamic optimal transport (OT) view with the observable model of Sec.~\ref{sec:preliminaries}.

\subsection{OT View: Editing as an Estimation Problem}
\label{sec:ot_view_revised}

We define the transport density $\rho_t(x)$ as the density evolving from the source boundary $\rho_1 = p_1(\cdot \mid c_{\rm src})$ to the target boundary $\rho_0 = p_0(\cdot \mid c_{\rm tar})$.
We define $u_t(x)$ as the editing vector field that drives this transport. The ideal $u_t$ is the one that solves the Benamou–Brenier dynamic OT problem:
\begin{equation}
\begin{aligned}
\min_{\rho,\,u}\quad
& \int_0^1\!\!\int \frac{1}{2}\,\|u_t(x)\|^2\,\rho_t(x)\,dx\,dt \\
\text{s.t.}\quad
& \partial_t \rho_t(x)+\nabla_x\!\cdot\!\big(\rho_t(x)\,u_t(x)\big)=0.
\end{aligned}
\label{eq:ot}
\end{equation}

This ideal field $u_t$ is unknown. We can only access it via the $\mathbf{R}(x_\tau, t)$, which was defined in Eq.~\eqref{eq:proxy-field}. We posit a measurement model where the observable $\mathbf{R}$ is the true field $u_t$ corrupted by a zero-mean noise term $\varepsilon_t$:
\begin{equation}
\mathbf{R}(x_\tau,t)\;=\;u_t(x_\tau)\;+\;\varepsilon_t, \quad \mathbb{E}[\varepsilon_t]=0.
\label{eq:measurement_model}
\end{equation}

A naive approach would use $u_{\rm nai} = \mathbf{R}$ as the control, but this noisy measurement's high-energy nature renders it unstable for single-step integration.

\subsection{Chord Control: A Low-Energy Local Estimator}
\label{sec:chord_control_revised}
To resolve the instability of the naive field, we seek a locally smoothed, low-energy estimator $\hat u_t$ for the true field $u_t$.
Fix a short window $[t-\delta,t]$ and an anchor $x_\tau$. We determine a locally constant estimate $u \in \mathbb{R}^d$ by minimizing a strictly convex quadratic surrogate $\Phi_t(u;x_\tau)$:
\begin{equation}
\Phi_t(u;x_\tau)\;=\;t\,\|u-\hat u_{t-\delta}(x_\tau)\|^2+\int_{t-\delta}^{t}\!\|u-\mathbf{R}(x_\tau,\xi)\|^2\,d\xi.
\label{eq:phi}
\end{equation}

This objective which derived in the Appendix balances a recursive energy prior $\hat u_{t-\delta}$ against agreement with the new measurements $\mathbf{R}$.
Setting $\nabla_u\Phi_t=0$ yields the unique minimizer:
\begin{equation}
u_t^\star(x_\tau)\;=\;\frac{t}{t+\delta}\,\hat u_{t-\delta}(x_\tau)\;+\;\frac{1}{t+\delta}\int_{t-\delta}^{t}\mathbf{R}(x_\tau,\xi)\,d\xi.
\label{eq:optimal_integral}
\end{equation}

Leveraging first-order causal approximations, namely $\hat u_{t-\delta} \approx \mathbf{R}(x_\tau, t-\delta)$ and $\int_{t-\delta}^{t}\mathbf{R}(x_\tau,\xi)d\xi \approx \delta \mathbf{R}(x_\tau, t)$, we obtain the practical \emph{Chord Control Field}:
\begin{equation}
\hat u_t(x_\tau)\;=\;\frac{t\,\mathbf{R}(x_\tau,t-\delta)+\delta\,\mathbf{R}(x_\tau,t)}{t+\delta}\;.
\label{eq:chord_main}
\end{equation}

Eq.~\eqref{eq:chord_main} is a \emph{causal} one–sided kernel smoothing ($\hat u = K_\delta * \mathbf R$) of the naive field $\mathbf R$, where the kernel satisfies $K_\delta\!\ge\!0$, $\int K_\delta=1$, and $\operatorname{supp} K_\delta\subset[0,\delta]$. As proven in Appendix, this averaging provides critical numerical stability. By Jensen's inequality, it is an $L^2$–contraction, $\int\!\|\hat u\|^2 \le \int\!\|\mathbf R\|^2$, suppressing high-energy spikes.
Since differentiation commutes with the temporal convolution, the $L^\infty$-norms (supremum norms) of the field, its time derivative, and its spatial gradient are all contracted, i.e., $\|\hat u\|_\infty \le \|\mathbf R\|_\infty$, $\|\partial_t \hat u\|_\infty \le \|\partial_t \mathbf R\|_\infty$, and $\|\nabla_x \hat u\|_\infty \le \|\nabla_x \mathbf R\|_\infty$.
This directly tightens the standard consistency proxy for explicit Euler
\begin{equation}
\mathcal C(u):=\|\partial_t u\|_\infty+\|\nabla_x u\|_\infty\,\|u\|_\infty,
\end{equation}
which, when applied to our chord field $\hat u$ and the naive field $\mathbf R$, yields $\mathcal C_{\rm cho}\le \mathcal C_{\rm nai}$. This reduces the local truncation error $M_f$, which is bounded by $\mathcal C(u)$, and thus tightens the global $O(h)$ error bound for our $h=1$ step (Appendix Thm. C.6). Furthermore, the step-size stability margin, governed by $\|\nabla_x u\|_\infty$, is preserved or improved (Appendix Prop. D.7). Eq.~\eqref{eq:chord_main} thus enforces a smoother, lower-energy path, mitigating the numerical fragility of the naive approach, as shown in Figure~\ref{fig:particle}.

\begin{figure}[t]
    \centering
    \includegraphics[width=0.80\linewidth]{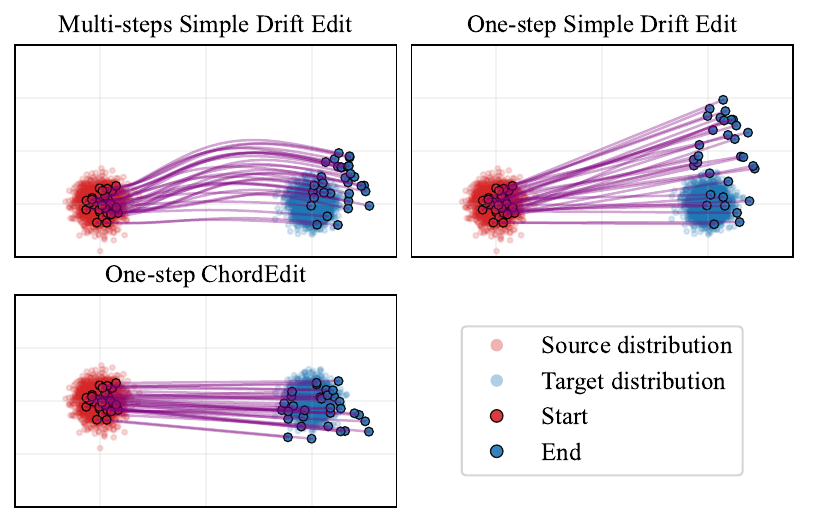}
    \vspace{-10pt}
    \caption{\textbf{2D Toy Example of Distribution transport.} Naive residual fields are high-energy and unstable under coarse discretization. ChordEdit computes a low-energy field (Eq.~\eqref{eq:chord_main}) that drives particles straight to the target with minimal deviation, facilitating reliable one-step transport.}
    \label{fig:particle}
\end{figure}

\begin{figure}
    \centering
    \includegraphics[width=0.9\linewidth]{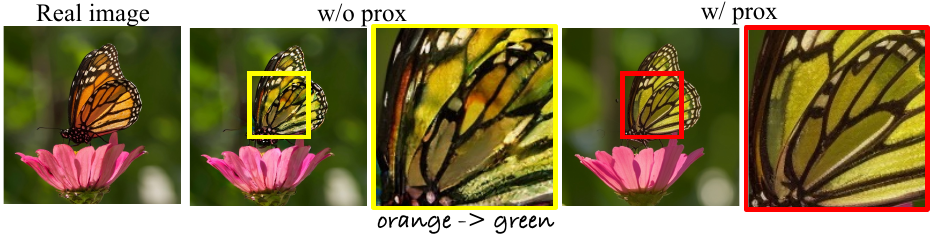}
 \vspace{-8pt}
 \caption{\textbf{Effect of the Proximal Refinement.} The refinement step enhances the target editing semantics.}
    \label{fig:ablation_prox_show}
\end{figure}

\subsection{Proximal Refinement}
Given the prediction $x^{\rm pred}$ from one-step transport, we introduce an \emph{optional} proximal refinement. This step, implemented as a single forward pass using only $c_{\rm tar}$, serves to amplify target semantics for challenging edits:
\begin{equation}
\operatorname{prox}\!\left(x^{\rm pred},t_{\rm c},c_{\rm tar}\right)
\;=\;\mathcal{B}_{t_{\rm c}}\,
Q\!\left(x^{\rm pred},t_{\rm c},c_{\rm tar}\right),
\label{eq:prox}
\end{equation}
implemented as one native “predict-$x_0$” call with a fixed noise draw. This plug-and-play step requires no re-inversion, and is not part of the transport. All energy metrics are computed before it. This approach separates structure-preserving transport from semantic enhancement, similar to strategies in multi-step methods~\cite{kulikov2025flowedit}. We visualize the effect of this step in Figure~\ref{fig:ablation_prox_show}, and provide a detailed ablation study in the Appendix.

\subsection{ChordEdit Algorithm}
Algorithm~\ref{alg:chordedit} presents a single-noise ($n=1$) implementation of ChordEdit in VAE latent space. This highly efficient $n=1$ configuration is the default setting used for all quantitative benchmarks and qualitative comparisons in our experiments. We provide a detailed analysis and ablation study on the effect of using multiple noise samples in Section~\ref{sec:noise-analysis}. The full implementation details for the multi-noise variant are deferred to the Appendix. All intermediate variables for $\hat{u}$ are parallel-computable within one batch, rendering the transport 1-NFE (Number of Function Evaluations).

\begin{algorithm}[t]
\caption{Simplified algorithm for ChordEdit}
\label{alg:chordedit}
\begin{algorithmic}[1]
\State \textbf{Inputs:} source image $x_{\rm src}$; prompts $c_{\rm src}, c_{\rm tar}$; step time $t$; window $\delta$; step scale $\lambda$; Proximal Refinement time $t_{\rm c}$.
\State \textbf{Output:} edited image $x_{\rm tar}$.
\State \textbf{Init:} $x_{\rm in} \leftarrow x_{\rm src}$
\State \(
\hat{u} \leftarrow \frac{t\,\mathbf{R}(x_{\rm in},\,t-\delta) + \delta\,\mathbf{R}(x_{\rm in},\,t)}{t+\delta}
\) 
\State $x^{\rm pred} \leftarrow x_{\rm in} + \lambda\,\hat{u}$
\State $x_{\rm tar} \leftarrow \operatorname{prox}\left(x^{\mathrm{pred}},\, t_{\mathrm c},\, c_{\mathrm{tar}}\right)$ \Comment{Optional}
\State \textbf{Return} $x_{\rm tar}$
\end{algorithmic}
\end{algorithm}

\begin{table*}[t]
\centering
\caption{\textbf{Quantitative comparison on PIE-bench}~\cite{ju2023direct}.
\textbf{T-free}: Training-free. \textbf{I-free}: Inversion-free.
The best/second/third results in each numeric column are highlighted with
\colorbox{RankFirst}{yellow}/\colorbox{RankSecond}{orange}/\colorbox{RankThird}{blue} backgrounds, respectively.
A comprehensive table with extended metrics (e.g., SSIM, Structure Distance) is available in Appendix.}
\label{tab:main_comparison_final}
\resizebox{\textwidth}{!}{%
\begin{tabular}{@{}l l | ccc | cc | cc | cccc@{}}
\toprule
\multirow{2}{*}{\textbf{Type}} & \multirow{2}{*}{\textbf{Method}} &
\multicolumn{3}{c|}{\textbf{Consistency}} &
\multicolumn{2}{c|}{\textbf{CLIP Semantics}} &
\multicolumn{2}{c|}{\textbf{Properties}} &
\multicolumn{4}{c}{\textbf{Efficiency}} \\
\cmidrule(lr){3-5} \cmidrule(lr){6-7} \cmidrule(lr){8-9} \cmidrule(lr){10-13}
& &
\textbf{PSNR}$\uparrow$ &
\textbf{MSE}${}_{\text{10}^3}\downarrow$ &
\textbf{LPIPS}${}_{\text{10}^3}\downarrow$ &
\textbf{Whole}$\uparrow$ &
\textbf{Edited}$\uparrow$ &
\textbf{T-free} &
\textbf{I-free} &
\textbf{Step}$\downarrow$ &
\textbf{NFE}$\downarrow$ &
\textbf{Runtime(s)}$\downarrow$ &
\textbf{VRAM(MiB)}$\downarrow$ \\
\midrule
\multirow{5}{*}{\shortstack[c]{Multi-step \\ ($\ge$ 30 steps)}} &
DDIM + MasaCtrl~\cite{songdenoising,cao2023masactrl} &
21.25 & 8.58 & 106.59 & 24.13 & 21.13 & \cmark & \xmark & 50 & 100 & 55.20 & 12272 \\
& Direct Inversion + MasaCtrl~\cite{ju2023direct,cao2023masactrl} &
21.78 & 7.99 & \cellcolor{RankThird}87.38 & 24.42 & 21.38 & \cmark & \xmark & 50 & 100 & 79.10 & 12272 \\
& DDIM + PnP~\cite{songdenoising,tumanyan2023plug} &
21.26 & 8.42 & 113.58 & 25.45 & 22.54 & \cmark & \xmark & 50 & 100 & 28.01 & \cellcolor{RankThird}9262 \\
& Direct Inversion + PnP~\cite{ju2023direct,tumanyan2023plug} &
21.43 & 8.10 & 106.26 & \cellcolor{RankThird}25.48 & \cellcolor{RankThird}22.63 & \cmark & \xmark & 50 & 100 & 28.03 & \cellcolor{RankThird}9262 \\
& FlowEdit (SD3)~\cite{kulikov2025flowedit} &
22.17 & 7.69 & 104.81 & \cellcolor{RankFirst}26.64 & \cellcolor{RankFirst}23.69 & \cmark & \cmark & \cellcolor{RankThird}33 & 33 & 7.22 & 17140 \\
\midrule
\multirow{3}{*}{\shortstack[c]{Few-step \\ (4 steps)}} &
TurboEdit (SDXL-Turbo)~\cite{deutch2024turboedit} &
21.44 & 9.49 & 108.60 & 24.66 & 21.79 & \cmark & \cmark &
\cellcolor{RankSecond}4 & \cellcolor{RankThird}4 & 2.69 & 13826 \\
& InfEdit (SD1.4)~\cite{xu2023inversion} &
\cellcolor{RankFirst}24.14 & \cellcolor{RankThird}6.82 & \cellcolor{RankFirst}55.69 &
24.89 & 21.88 & \cmark & \cmark &
\cellcolor{RankSecond}4 & \cellcolor{RankThird}4 & 1.41 & \cellcolor{RankFirst}6502 \\
& InstantEdit (PeRFlow-SD1.5)~\cite{gong2025instantedit} &
\cellcolor{RankThird}23.80 & \cellcolor{RankFirst}4.21 & \cellcolor{RankSecond}60.92 &
24.97 & 21.82 & \cmark & \xmark &
\cellcolor{RankSecond}4 & 8 & 1.30 & 16270 \\
\midrule
\multirow{4}{*}{\shortstack[c]{One-step}} &
SwiftEdit (SwiftBrush-v2)~\cite{nguyen2025swiftedit} &
21.71 & 8.22 & 91.22 & 24.93 & 21.85 & \xmark & \xmark &
\cellcolor{RankFirst}1 & \cellcolor{RankSecond}2 & \cellcolor{RankThird}0.54 & 15060 \\
& \textbf{ChordEdit (SwiftBrush-v2)} &
22.04 & 7.13 & 111.22 & 25.12 & 22.58 & \cmark & \cmark &
\cellcolor{RankFirst}1 & \cellcolor{RankSecond}2 & \cellcolor{RankSecond}\textbf{0.38} & \cellcolor{RankSecond}6988 \\
& \textbf{ChordEdit (w/o prox, SD-Turbo)} &
\cellcolor{RankSecond}23.89 & \cellcolor{RankSecond}5.05 & 88.36 & 24.97 & 21.87 & \cmark & \cmark &
\cellcolor{RankFirst}1 & \cellcolor{RankFirst}1 & \cellcolor{RankFirst}\textbf{0.20} & \cellcolor{RankSecond}6988 \\
& \textbf{ChordEdit (SD-Turbo)} &
22.20 & 6.84 & 128.25 & \cellcolor{RankSecond}25.58 & \cellcolor{RankSecond}22.96 & \cmark & \cmark &
\cellcolor{RankFirst}1 & \cellcolor{RankSecond}2 & \cellcolor{RankSecond}\textbf{0.38} & \cellcolor{RankSecond}6988 \\
\bottomrule
\end{tabular}%
}
\end{table*}

\begin{figure*}[t]
    \centering
    \includegraphics[width=0.94\linewidth]{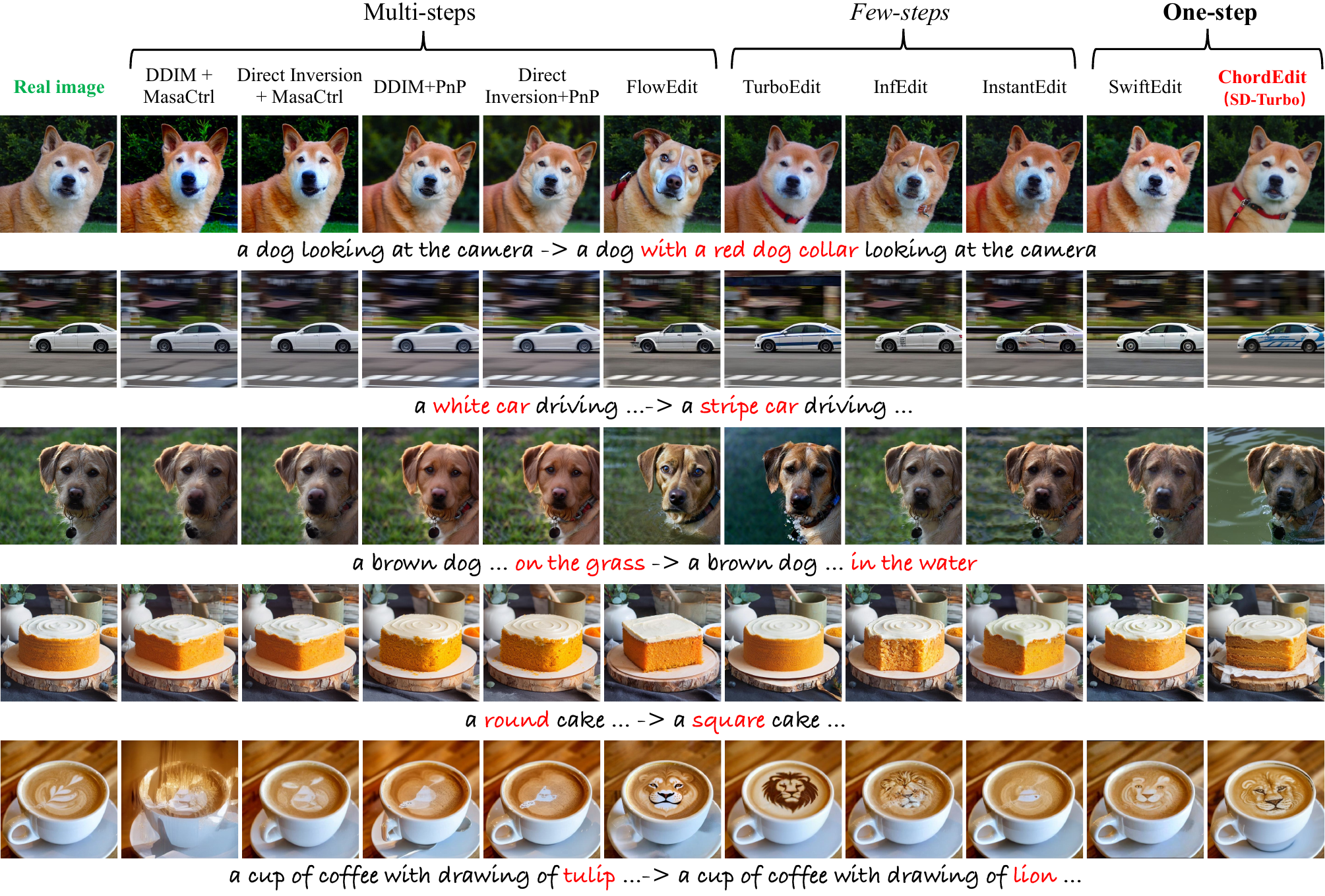}
    \vspace{-7pt}
    \caption{\textbf{Comparison of edited results.} Real images are in the first column. Prompts are noted under each row.}
    \label{fig:sota_grid}
\end{figure*}

\section{Experiments}

\subsection{Experimental Setup}

\paragraph{Dataset and evaluation metrics.} We conduct our empirical evaluation on the PIE-bench~\cite{ju2023direct} benchmark, a standard dataset for instruction-based image 
editing on 512$\times$512 images. This benchmark comprises 700 samples distributed across 10 distinct editing categories, where each instance provides a source image, textual prompts, and a precise ground-truth mask delineating the edit region. We assess performance along two critical axes: background fidelity and semantic alignment. Background fidelity is quantified using Peak Signal-to-Noise Ratio (PSNR) and Mean Squared Error (MSE), computed exclusively on the non-edited regions. Semantic alignment with the target instruction is measured via CLIP-Whole and CLIP-Edited scores~\cite{radford2021learning}, which evaluate the textual-visual consistency of the entire image and the modified region, respectively.

\paragraph{Implementation details.} Code is available in supplementary material. All experiments are conducted on a single NVIDIA Titan 24GB GPU. Our framework consists of the 1-NFE Chord transport step and an optional 1-NFE proximal refinement. To present the best overall performance, our default ChordEdit (NFE=2) includes this refinement (parameters: $n=1$, $t=0.90$, $\delta=0.15$, $\lambda=1.00$, $t_c=0.30$). These defaults reflect a clear trade-off, as $\delta$ balances stability against semantic strength, while $t$ and $\lambda$ jointly control the transport's semantic intensity, and $t_c$ amplifies the final edit. We also report the transport-only ChordEdit (w/o prox) in Table~\ref{tab:main_comparison_final}. For a fair comparison of background fidelity, all methods were similarly evaluated without the use of any internal or external protective masks. We compare against representative multi-/few-/one-step editors. We note that many state-of-the-art editors, especially few-step methods, are architecturally coupled with specific models, such as TurboEdit with SDXL-Turbo, InstantEdit with PeRFlow. Therefore, to fairly evaluate each method's optimal performance, we follow standard practice and report results using their officially specified models. To isolate the gains from our method, a direct comparison on a unified model SwiftBrush-v2 is provided in the One-step category of Table~\ref{tab:main_comparison_final}.

\subsection{Comparison with Prior Methods}

\paragraph{Quantitative Results.}
In Table~\ref{tab:main_comparison_final}, we compare ChordEdit with multi-step, few-step, and one-step editors on PIE-bench. Overall, ChordEdit's training-free, inversion-free design achieves state-of-the-art efficiency, requiring less than half the VRAM of SwiftEdit on the same model while maintaining competitive editing quality. Compared to multi-step methods, ChordEdit shows superior background preservation and highly competitive semantics, yet is significantly faster, such as 19$\times$ faster than FlowEdit and over 208$\times$ faster than Direct Inversion. Against few-step editors, ChordEdit leads in semantic fidelity while being at least 3.4$\times$ faster than the fastest alternative.
In the One-step category, our transport-only ChordEdit (w/o prox) validates our core innovation: it achieves a high PSNR at NFE=1, demonstrating our Chord Control Field's stability. Building on this, our full ChordEdit adds an optional refinement to further enhance semantic alignment, achieving the best overall balance without model-specific training, inversion, or protective masking.

\begin{figure*}[t]
    \centering
    \includegraphics[width=0.98\linewidth]{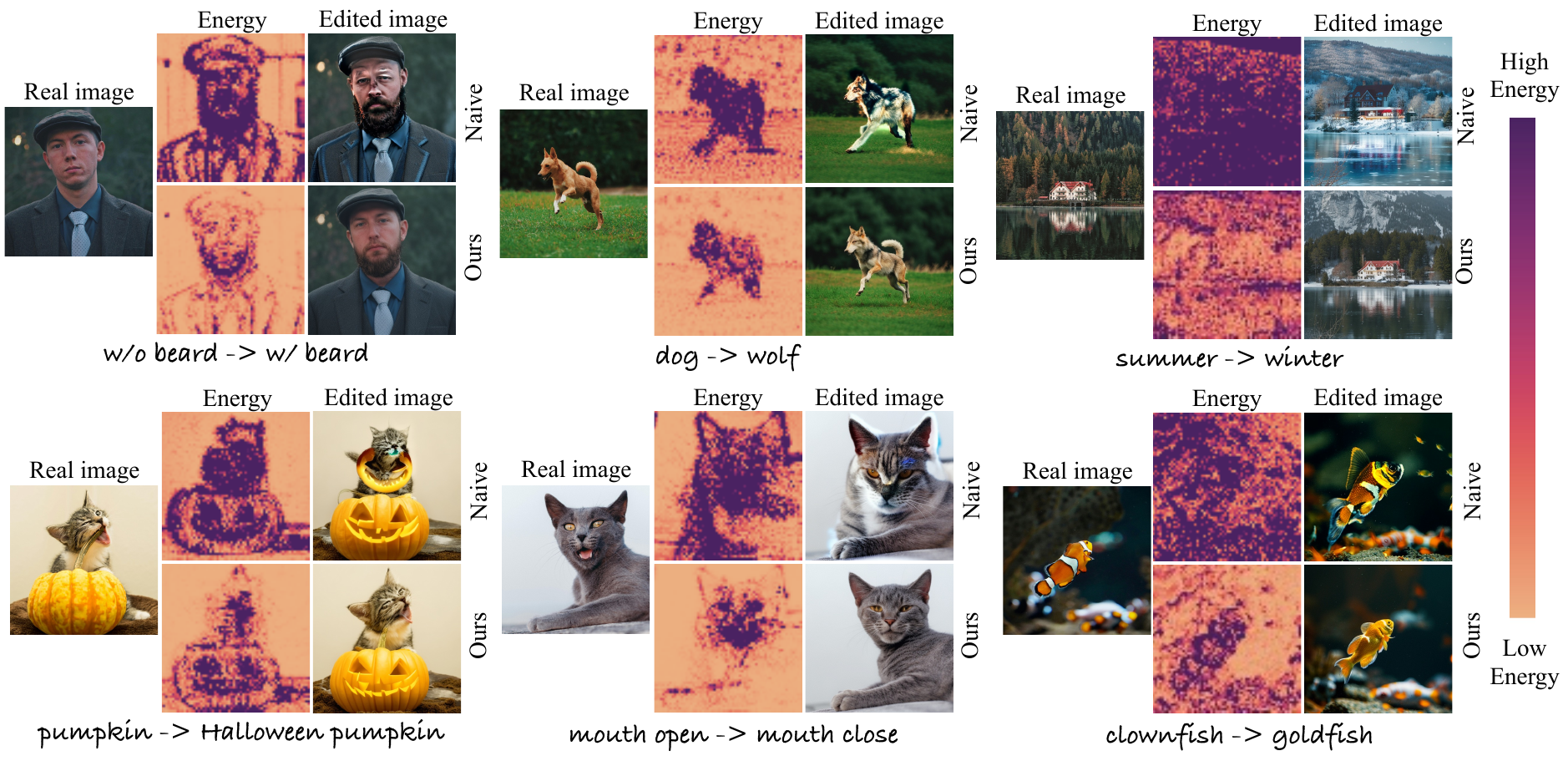}
    \vspace{-12pt}
\caption{\textbf{Qualitative Comparison and Energy Visualization.} 
We compare ChordEdit (Ours) against the naive baseline ($\delta=0$, Naive).
The naive method's high-energy field leads to artifacts and background corruption.
Our ChordEdit derives a stable, low-energy field, resulting in high-fidelity edits that preserve object identity and non-edited regions.
Results shown used SwiftBrush-v2 (first column) and SD-Turbo (second and third columns). Energy plots are computed as $E=\frac{1}{S\,C}\sum_{s=1}^{S}\sum_{\text{channel}}(\hat u_{t_s})^2$.}
\label{fig:energy}
\end{figure*}

\begin{figure}[t]
  \centering
  \begin{minipage}[t]{0.43\linewidth}
    \centering
    \includegraphics[width=\linewidth]{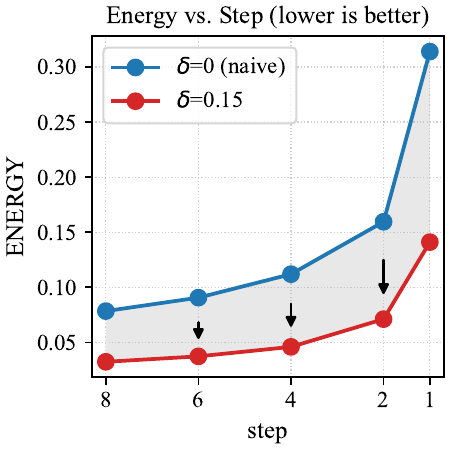}    
  \end{minipage}
  \begin{minipage}[t]{0.43\linewidth}
    \centering
    \includegraphics[width=\linewidth]{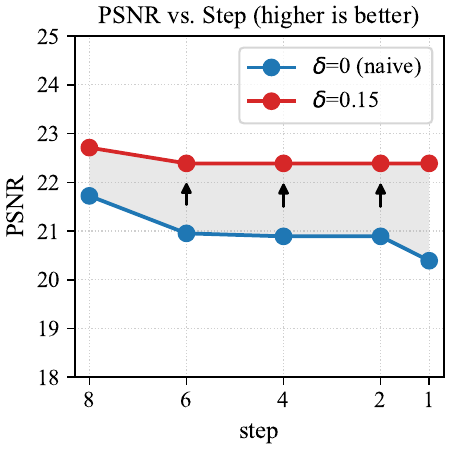}
  \end{minipage}
    \begin{minipage}[t]{0.43\linewidth}
    \centering
    \includegraphics[width=\linewidth]{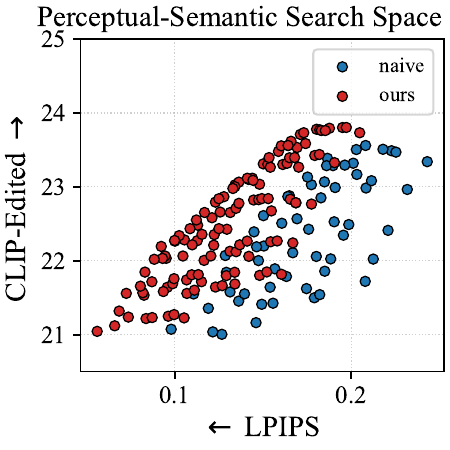}    
  \end{minipage}
  \begin{minipage}[t]{0.43\linewidth}
    \centering
    \includegraphics[width=\linewidth]{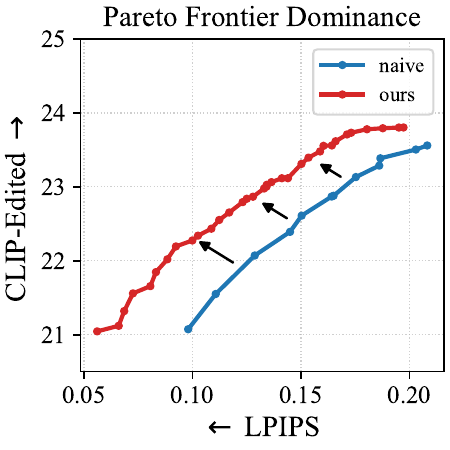}
  \end{minipage}
  \vspace{-7pt}
\caption{\textbf{(top) ChordEdit Stability as a function of Integration Steps.} We compare ChordEdit ($\delta=0.15$, red) against the naive baseline ($\delta=0$, blue). The naive field's energy spikes as step $S \to 1$, confirming its unsuitability for large steps. ChordEdit's energy remains low. Consequently, the naive method's background consistency (PSNR) collapses, while ChordEdit maintains high PSNR. \textbf{(bottom) Analysis of the Perceptual-Semantic Trade-off.} Our method ($\delta \ne 0$, red) strictly Pareto-dominates the naive baseline ($\delta=0$, blue), consistently achieving superior semantic alignment for any given level of perceptual distortion.}
  \label{fig:ccf_effort}
\end{figure}

\paragraph{Qualitative Results.}
In Figure~\ref{fig:sota_grid}, visual comparisons confirm our quantitative findings. ChordEdit consistently adheres to the prompt with exceptional background preservation, avoiding the artifacts or identity failures seen in multi-step methods such as Direct Inversion+PnP. It also demonstrates a strong balance of semantics and consistency compared to other few-step methods. A user study (details in Appendix) further supports this, with participants overwhelmingly preferring our method for both editing semantics (42.5\%) and background preservation (48.3\%).

\section{Ablation Study}

We conduct ablation studies to validate ChordEdit's core components. This section analyzes: our Chord Control Field against the naive $\delta=0$ baseline, the impact of noise samples, the decoupled contributions of our transport and refinement steps, and model-agnostic performance. A detailed hyperparameter analysis for $t$, $\delta$, $\lambda$, and $t_c$ is deferred to the Appendix.

\subsection{Analysis of the Chord Control Field}

We validate our Chord Control Field (CCF) by ablating its temporal smoothing interval $\delta$. Setting $\delta=0$ degenerates our CCF into the naive baseline. We report its unweighted, discrete Benamou–Brenier kinetic energy
$\bar E=\frac{1}{S\,C\,H\,W}\sum_{s=1}^{S}\sum_{\text{dims}}(\hat u_{t_s})^2$, 
where $S$ is the total step count and the inner sum $\sum_{\text{dims}}$ is over the channel $C$, height $H$, and width $W$ dimensions.
The quantitative results in Figure~\ref{fig:ccf_effort} (top) confirm our hypothesis: as the step count $S \to 1$, the naive field's energy spikes and its PSNR collapses, while our CCF ($\delta=0.15$) remains stable. This stability provides a superior perceptual-semantic trade-off, as our method strictly Pareto-dominates the naive baseline (Figure~\ref{fig:ccf_effort}, bottom). Figure~\ref{fig:energy} qualitatively confirms this low-energy path prevents artifacts and keeps non-edited regions.

\begin{figure}
    \centering
    \includegraphics[width=\linewidth]{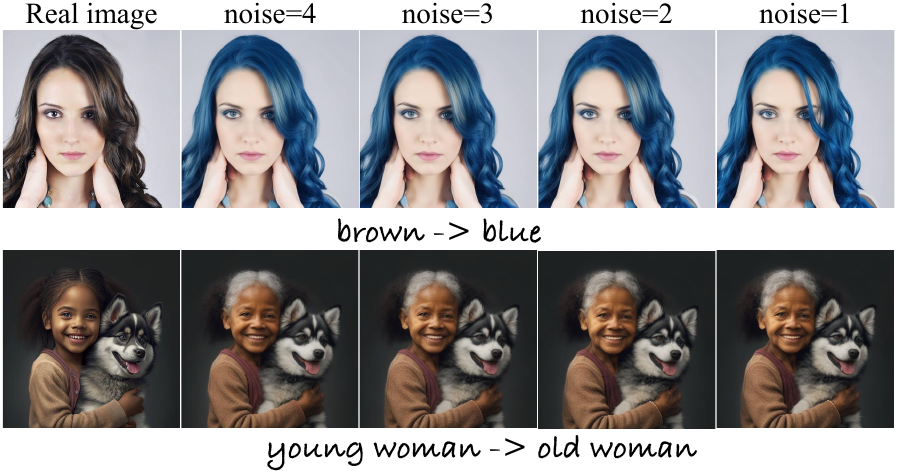}
    \vspace{-23pt}
   \caption{\textbf{Effect of the number of noise samples.} This figure qualitatively confirms that increasing the number of noise samples yields negligible marginal returns.}
   \vspace{-12pt}
    \label{fig:noise_ab_show}
\end{figure}

\begin{figure}[t]
  \centering
  \begin{minipage}[t]{0.42\linewidth}
    \centering
    \includegraphics[width=\linewidth]{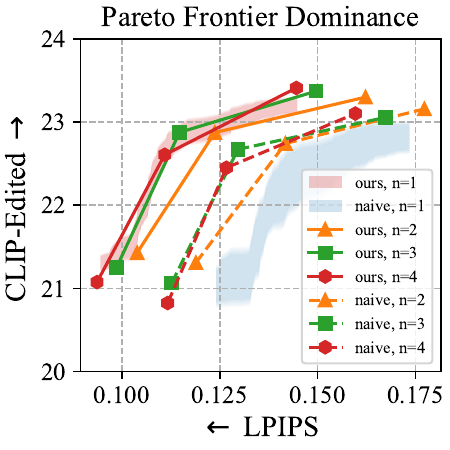}    
  \end{minipage}
  \begin{minipage}[t]{0.56\linewidth}
    \centering
    \includegraphics[width=\linewidth]{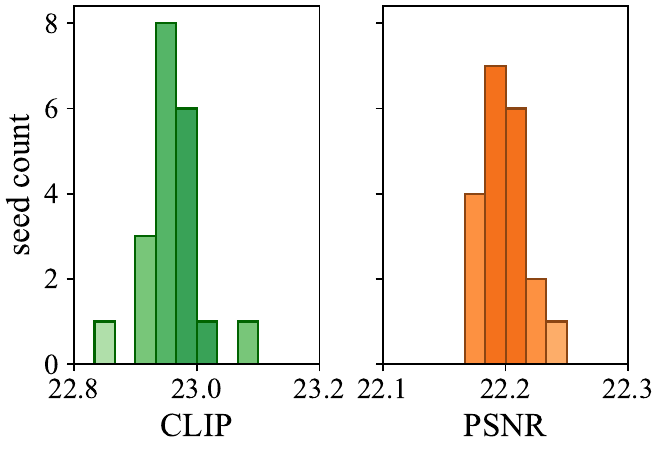}
  \end{minipage}
  \vspace{-13pt}
\caption{\textbf{Pareto dominance and Seed robustness.} Left: LPIPS–CLIP Pareto fronts~\cite{zhang2018unreasonable} comparing ChordEdit (solid) to the naive baseline (dashed). Shaded regions denote the envelope across seeds. Fronts for ChordEdit with $n=1\ldots4$ are nearly overlapping and dominate the naive counterparts, indicating negligible marginal returns from multi-noise. Right: histograms of CLIP-Edited and PSNR across 20 seeds for single-noise ($n{=}1$). Both distributions are tight (CLIP CoV $0.20\%$, PSNR CoV $0.07\%$), confirming that ChordEdit one noise is effectively insensitive to the random seed.}
  \label{fig:pareto_noise_and_count}
\end{figure}

\subsection{Analysis of Noise}
\label{sec:noise-analysis}

We analyze the effect of Monte-Carlo noise samples $n$. Classical practice requires $n > 1$ to reduce estimator variance, which scales inversely with the number of samples. We hypothesize that Chord Control Field's smoothed, differential construction possesses an intrinsically low variance, rendering additional samples unnecessary.

Our empirical results validate this hypothesis:
(i) Increasing $n$ yields negligible marginal returns for ChordEdit. As shown in Figure~\ref{fig:pareto_noise_and_count} (left), the Pareto fronts for $n=1, 2, 3, 4$ are nearly indistinguishable, a finding qualitatively confirmed by Figure~\ref{fig:noise_ab_show}.
(ii) ChordEdit's $n=1$ performance is stable, whereas the naive exhibits high variance, relying on $n>1$ samples to stabilize. Nonetheless, ChordEdit's $n=1$ front strictly Pareto-dominates the naive even at naive's $n=4$ samples.
(iii) ChordEdit is robust to seed variation at $n=1$. The histograms in Figure~\ref{fig:pareto_noise_and_count} (right) show tight performance distributions, confirming our $n=1$ performance is stable and reliable.

These results establish that ChordEdit achieves state-of-the-art, seed-robust performance using only a single noise sample. Our method's geometric control and intrinsic variance reduction achieve the same precision as traditional methods that rely on costly MC averaging.

\subsection{Analysis of Transport and Refinement}
\label{sec:ablation_transport}

\begin{table}[t]
\centering
\caption{\textbf{Ablation of Chord transport field and refinement.} Our Chord field drives consistency (PSNR), while the prox step boosts semantics (CLIP-Edited). Full metrics are in the Appendix.}
\label{tab:ablation_components}
\vspace{-5pt}
\resizebox{0.9\linewidth}{!}{%
\begin{tabular}{l | cc | cc | c}
\toprule
\multirow{2}{*}{\textbf{Method}} & \multicolumn{2}{c|}{\textbf{Naive ($\delta=0$)}} & \multicolumn{2}{c|}{\textbf{Ours ($\delta=0.15$)}} & \multirow{2}{*}{\textbf{NFE}} \\
\cmidrule(lr){2-3} \cmidrule(lr){4-5}
 & \textbf{PSNR} $\uparrow$ & \textbf{CLIP-Edited} $\uparrow$ & \textbf{PSNR} $\uparrow$ & \textbf{CLIP-Edited} $\uparrow$ & \\
\midrule
w/o prox & 21.89 & 20.83 & \textbf{23.89} & 21.87 & 1 \\
w/ prox & 21.38 & 21.96 & 22.20 & \textbf{22.96} & 2 \\
\bottomrule
\end{tabular}%
}
\end{table}

We ablate our framework in Table~\ref{tab:ablation_components} to show its modularity. The Chord field (w/o prox) prioritizes consistency, achieving a high 23.89 PSNR. This conservative transport, with a 21.87 CLIP-Edited score, is then semantically amplified by the optional prox step, which boosts the score to 22.96. This design effectively separates high-fidelity transport from semantic amplification.

\subsection{Analysis of T2I Models}

\begin{table}[t]
\centering
\caption{\textbf{Quantitative comparison on different T2I models.} Our method (Ours) consistently outperforms the naive baseline across all tested models. Full details are provided in the Appendix.}
\label{tab:model_comparison}
\vspace{-5pt}
\resizebox{0.7\linewidth}{!}{
\begin{tabular}{l c c c c}
\toprule
\multirow{2}{*}{\textbf{T2I Method}} & \multicolumn{2}{c}{\textbf{PSNR} $\uparrow$} & \multicolumn{2}{c}{\textbf{CLIP-Edited} $\uparrow$} \\
\cmidrule(lr){2-3} \cmidrule(lr){4-5}
 & {\small Naive} & {\small Ours} & {\small Naive} & {\small Ours} \\
\midrule
InstaFlow~\cite{liu2023instaflow} & 22.05 & 23.05 & 20.19 & 21.39 \\
SwiftBrush-v2~\cite{Dao2025SwiftBrushV2} & 20.52 & 22.04 & 21.06 & 22.58 \\
SD-Turbo~\cite{Sauer2024ADD} & 21.38 & \textbf{22.20} & 21.96 & \textbf{22.96} \\
\bottomrule
\end{tabular}%
}
\end{table}

We validate ChordEdit's model-agnostic claim. Table~\ref{tab:model_comparison} confirms our method's robust applicability, consistently outperforming the naive baseline. On SD-Turbo, for instance, it boosts PSNR from 21.38 to 22.20 and the CLIP-Edited score from 21.96 to 22.96.

\section{Conclusion}
\label{sec:conclusion}

We introduced ChordEdit, a training-free, inversion-free framework that solves the instability of one-step image editing. Our method replaces the naive, high-energy drift difference with a principled Chord Control Field. This field's temporal smoothing creates a stable, low-energy transport path, facilitating a single, large integration step while preserving non-edited regions. ChordEdit achieves state-of-the-art efficiency with a runtime of 0.38s and a low VRAM footprint. This speed is achieved not by sacrificing quality, but by maintaining high fidelity and strong semantic alignment, outperforming the naive baseline and other one-step methods. This high-fidelity performance is robustly model-agnostic and seed-insensitive, even with a single noise sample. ChordEdit achieves true real-time, high-fidelity, and consistent generative image editing. We acknowledge the potential for misuse and we intend our work for creative and assistive applications. A full discussion on societal impacts is in the Appendix.

{
    \small
    \bibliographystyle{ieeenat_fullname}
    \bibliography{main}

@String(CVPR= {IEEE Conf. Comput. Vis. Pattern Recog.})

@String(ICCV= {Int. Conf. Comput. Vis.})

@String(ECCV= {Eur. Conf. Comput. Vis.})

@String(ICLR = {Int. Conf. Learn. Represent.})

@String(AAAI = {AAAI})

@String(CVPR  = {CVPR})

@String(ICCV  = {ICCV})

@String(ECCV  = {ECCV})

@String(ICLR  = {ICLR})

@inproceedings{song2023consistency,
  title={Consistency Models},
  author={Song, Yang and Dhariwal, Prafulla and Chen, Mark and Sutskever, Ilya},
  booktitle={International Conference on Machine Learning},
  pages={32211--32252},
  year={2023},
  organization={PMLR}
}

@misc{luo2023latent,
      title={Latent Consistency Models: Synthesizing High-Resolution Images with Few-Step Inference}, 
      author={Simian Luo and Yiqin Tan and Longbo Huang and Jian Li and Hang Zhao},
      year={2023},
      eprint={2310.04378},
      archivePrefix={arXiv},
      primaryClass={cs.CV}
}

@inproceedings{liu2023instaflow,
  title={Instaflow: One step is enough for high-quality diffusion-based text-to-image generation},
  author={Liu, Xingchao and Zhang, Xiwen and Ma, Jianzhu and Peng, Jian and Liu, Qiang},
  booktitle={International Conference on Learning Representations},
  year={2024}
}

@article{hertz2022prompt,
  title = {Prompt-to-Prompt Image Editing with Cross Attention Control},
  author = {Hertz, Amir and Mokady, Ron and Tenenbaum, Jay and Aberman, Kfir and Pritch, Yael and Cohen-Or, Daniel},
  journal = {arXiv preprint arXiv:2208.01626},
  year = {2022},
}

@article{benamou2000computational,
  title={A computational fluid mechanics solution to the Monge-Kantorovich mass transfer problem},
  author={Benamou, Jean-David and Brenier, Yann},
  journal={Numerische Mathematik},
  volume={84},
  number={3},
  pages={375--393},
  year={2000},
  publisher={Springer-Verlag Berlin/Heidelberg}
}

@inproceedings{cao2023masactrl,
  title={Masactrl: Tuning-free mutual self-attention control for consistent image synthesis and editing},
  author={Cao, Mingdeng and Wang, Xintao and Qi, Zhongang and Shan, Ying and Qie, Xiaohu and Zheng, Yinqiang},
  booktitle={Proceedings of the IEEE/CVF international conference on computer vision},
  pages={22560--22570},
  year={2023}
}

@inproceedings{parmar2023zero,
  title={Zero-shot image-to-image translation},
  author={Parmar, Gaurav and Kumar Singh, Krishna and Zhang, Richard and Li, Yijun and Lu, Jingwan and Zhu, Jun-Yan},
  booktitle={ACM SIGGRAPH 2023 conference proceedings},
  pages={1--11},
  year={2023}
}

@inproceedings{tumanyan2023plug,
  title={Plug-and-play diffusion features for text-driven image-to-image translation},
  author={Tumanyan, Narek and Geyer, Michal and Bagon, Shai and Dekel, Tali},
  booktitle={Proceedings of the IEEE/CVF conference on computer vision and pattern recognition},
  pages={1921--1930},
  year={2023}
}

@inproceedings{songdenoising,
  title={Denoising Diffusion Implicit Models},
  author={Song, Jiaming and Meng, Chenlin and Ermon, Stefano},
  booktitle={International Conference on Learning Representations}
}

@inproceedings{mokady2023null,
  title={Null-text inversion for editing real images using guided diffusion models},
  author={Mokady, Ron and Hertz, Amir and Aberman, Kfir and Pritch, Yael and Cohen-Or, Daniel},
  booktitle={Proceedings of the IEEE/CVF conference on computer vision and pattern recognition},
  pages={6038--6047},
  year={2023}
}

@inproceedings{kulikov2025flowedit,
  title={Flowedit: Inversion-free text-based editing using pre-trained flow models},
  author={Kulikov, Vladimir and Kleiner, Matan and Huberman-Spiegelglas, Inbar and Michaeli, Tomer},
  booktitle={Proceedings of the IEEE/CVF International Conference on Computer Vision},
  pages={19721--19730},
  year={2025}
}

@inproceedings{nguyen2025swiftedit,
  title={Swiftedit: Lightning fast text-guided image editing via one-step diffusion},
  author={Nguyen, Trong-Tung and Nguyen, Quang and Nguyen, Khoi and Tran, Anh and Pham, Cuong},
  booktitle={Proceedings of the Computer Vision and Pattern Recognition Conference},
  pages={21492--21501},
  year={2025}
}

@inproceedings{deutch2024turboedit,
  title={Turboedit: Text-based image editing using few-step diffusion models},
  author={Deutch, Gilad and Gal, Rinon and Garibi, Daniel and Patashnik, Or and Cohen-Or, Daniel},
  booktitle={SIGGRAPH Asia 2024 Conference Papers},
  pages={1--12},
  year={2024}
}

@article{xu2023inversion,
  title={Inversion-Free Image Editing with Natural Language},
  author={Xu, Sihan and Huang, Yidong and Pan, Jiayi and Ma, Ziqiao and Chai, Joyce},
  journal={CoRR},
  year={2023}
}

@inproceedings{gong2025instantedit,
  title={InstantEdit: Text-Guided Few-Step Image Editing with Piecewise Rectified Flow},
  author={Gong, Yiming and Zhu, Zhen and Zhang, Minjia},
  booktitle={Proceedings of the IEEE/CVF International Conference on Computer Vision},
  pages={16808--16817},
  year={2025}
}

@article{ju2023direct,
  title={PnP Inversion: Boosting Diffusion-based Editing with 3 Lines of Code},
  author={Ju, Xuan and Zeng, Ailing and Bian, Yuxuan and Liu, Shaoteng and Xu, Qiang},
  journal={International Conference on Learning Representations ({ICLR})},
  year={2024}
}

@article{icd24,
  author  = {Nikita Starodubcev and Mikhail Khoroshikh and Artem Babenko and Dmitry Baranchuk},
  title   = {Invertible Consistency Distillation for Text-Guided Image Editing in Around 7 Steps},
  journal = {arXiv:2406.14539},
  year    = {2024},
  url     = {https://arxiv.org/abs/2406.14539}
}

@article{diffedit22,
  author  = {Guillaume Couairon and Jakob Verbeek and Holger Schwenk and Matthieu Cord},
  title   = {DiffEdit: Diffusion-Based Semantic Image Editing with Mask Guidance},
  journal = {arXiv:2210.11427},
  year    = {2022},
  url     = {https://arxiv.org/abs/2210.11427}
}

@inproceedings{edict23,
  author    = {Bram Wallace and Akash Gokul and Nikhil Naik},
  title     = {EDICT: Exact Diffusion Inversion via Coupled Transformations},
  booktitle = {CVPR},
  year      = {2023},
  url       = {https://openaccess.thecvf.com/content/CVPR2023/papers/Wallace_EDICT_Exact_Diffusion_Inversion_via_Coupled_Transformations_CVPR_2023_paper.pdf}
}

@inproceedings{aidi23,
  author    = {Zhizhong Pan and Yijun Li and Xue Bai and Zhuowen Tu and Ming-Hsuan Yang},
  title     = {Effective Real Image Editing with Accelerated Iterative Diffusion Inversion},
  booktitle = {ICCV},
  year      = {2023},
  url       = {https://openaccess.thecvf.com/content/ICCV2023/html/Pan_Effective_Real_Image_Editing_with_Accelerated_Iterative_Diffusion_Inversion_ICCV_2023_paper.html}
}

@inproceedings{npi25,
  author    = {Daiki Miyake and Akihiro Iohara and Yu Saito and Toshiyuki Tanaka},
  title     = {Negative-Prompt Inversion: Fast Image Inversion for Editing with Text-Guided Diffusion Models},
  booktitle = {WACV},
  year      = {2025},
  url       = {https://openaccess.thecvf.com/content/WACV2025/html/Miyake_Negative-Prompt_Inversion_Fast_Image_Inversion_for_Editing_with_Text-Guided_Diffusion_WACV_2025_paper.html}
}

@inproceedings{proxedit24,
  author    = {Bo Han and Yujun Shen and Yingqing He and Lei Zhang and Jingren Zhou},
  title     = {ProxEdit: Improving Tuning-Free Real Image Editing with Proximal Guidance},
  booktitle = {WACV},
  year      = {2024},
  url       = {https://openaccess.thecvf.com/content/WACV2024/papers/Han_ProxEdit_Improving_Tuning-Free_Real_Image_Editing_With_Proximal_Guidance_WACV_2024_paper.pdf}
}

@inproceedings{spdinv24,
  author    = {Ruibin Li and Yujun Shen and others},
  title     = {Source Prompt Disentangled Inversion for Boosting Image Editability with Diffusion Models},
  booktitle = {ECCV},
  year      = {2024},
  url       = {https://www.ecva.net/papers/eccv_2024/papers_ECCV/papers/03804.pdf}
}

@inproceedings{cora25,
  author    = {Amir Alimohammadi and Aryan Mikaeili and Sauradip Nag and Negar Hassanpour and Andrea Tagliasacchi and Ali Mahdavi Amiri},
  title     = {Cora: Correspondence-Aware Image Editing Using Few-Step Diffusion},
  booktitle = {SIGGRAPH},
  year      = {2025},
  url       = {https://cora-edit.github.io/paper.pdf}
}

@inproceedings{Sauer2024ADD,
  author    = {Axel Sauer and Dominik Lorenz and Andreas Blattmann and Robin Rombach},
  title     = {Adversarial Diffusion Distillation},
  booktitle = {Computer Vision -- ECCV 2024},
  series    = {Lecture Notes in Computer Science},
  volume    = {15144},
  pages     = {87--103},
  publisher = {Springer},
  year      = {2024},
  doi       = {10.1007/978-3-031-73016-0_6},
  url       = {https://arxiv.org/abs/2311.17042}
}

@incollection{Dao2025SwiftBrushV2,
  author    = {Trung Dao and Thuan Hoang Nguyen and Thanh Le and Duc Vu and Khoi Nguyen and Cuong Pham and Anh Tran},
  title     = {SwiftBrush V2: Make Your One-Step Diffusion Model Better Than Its Teacher},
  booktitle = {Computer Vision -- ECCV 2024},
  series    = {Lecture Notes in Computer Science},
  volume    = {15140},
  pages     = {176--192},
  publisher = {Springer},
  year      = {2025},   
  doi       = {10.1007/978-3-031-73007-8_11},
  url       = {https://arxiv.org/abs/2408.14176}
}

@article{zu2024cotflow,
  title   = {COT Flow: Learning Optimal-Transport Image Sampling and Editing by Contrastive Pairs},
  author  = {Zu, Xinrui and Tao, Qian},
  journal = {arXiv preprint arXiv:2406.12140},
  year    = {2024},
  url     = {https://arxiv.org/abs/2406.12140}
}

@inproceedings{kawar2023imagic,
  title     = {Imagic: Text-Based Real Image Editing with Diffusion Models},
  author    = {Kawar, Bahjat and Zada, Shiran and Lang, Oran and Tov, Omer and Chang, Huiwen and Dekel, Tali and Mosseri, Inbar and Irani, Michal},
  booktitle = {Proceedings of the IEEE/CVF Conference on Computer Vision and Pattern Recognition (CVPR)},
  year      = {2023},
  url       = {https://openaccess.thecvf.com/content/CVPR2023/papers/Kawar_Imagic_Text-Based_Real_Image_Editing_With_Diffusion_Models_CVPR_2023_paper.pdf}
}

@inproceedings{xu2025textvdb,
  title     = {Textualize Visual Prompt for Image Editing via Diffusion Bridge},
  author    = {Xu, Pengcheng and Fan, Qingnan and Kou, Fei and Qin, Shuai and Gu, Hong and Zhao, Ruoyu and Ling, Charles and Wang, Boyu},
  booktitle = {Proceedings of the AAAI Conference on Artificial Intelligence},
  year      = {2025},
  url       = {https://ojs.aaai.org/index.php/AAAI/article/view/35483}
}

@article{si2025pix2pixzerocon,
  title   = {Contrastive Learning Guided Latent Diffusion Model for Image-to-Image Translation},
  author  = {Si, Qi and Wang, Bo and Zhang, Zhao},
  journal = {arXiv preprint arXiv:2503.20484},
  year    = {2025},
  url     = {https://arxiv.org/abs/2503.20484}
}

@article{liu2022flow,
  title={Flow straight and fast: Learning to generate and transfer data with rectified flow},
  author={Liu, Xingchao and Gong, Chengyue and Liu, Qiang},
  journal={arXiv preprint arXiv:2209.03003},
  year={2022}
}

@article{lipman2022flow,
  title={Flow matching for generative modeling},
  author={Lipman, Yaron and Chen, Ricky TQ and Ben-Hamu, Heli and Nickel, Maximilian and Le, Matt},
  journal={arXiv preprint arXiv:2210.02747},
  year={2022}
}

@inproceedings{radford2021learning,
  title={Learning transferable visual models from natural language supervision},
  author={Radford, Alec and Kim, Jong Wook and Hallacy, Chris and Ramesh, Aditya and Goh, Gabriel and Agarwal, Sandhini and Sastry, Girish and Askell, Amanda and Mishkin, Pamela and Clark, Jack and others},
  booktitle={International conference on machine learning},
  pages={8748--8763},
  year={2021},
  organization={PmLR}
}

@inproceedings{zhang2018unreasonable,
  title={The unreasonable effectiveness of deep features as a perceptual metric},
  author={Zhang, Richard and Isola, Phillip and Efros, Alexei A and Shechtman, Eli and Wang, Oliver},
  booktitle={Proceedings of the IEEE conference on computer vision and pattern recognition},
  pages={586--595},
  year={2018}
}

@article{meng2021sdedit,
  title={Sdedit: Guided image synthesis and editing with stochastic differential equations},
  author={Meng, Chenlin and He, Yutong and Song, Yang and Song, Jiaming and Wu, Jiajun and Zhu, Jun-Yan and Ermon, Stefano},
  journal={arXiv preprint arXiv:2108.01073},
  year={2021}
}

@article{lin2024sdxl,
  title={Sdxl-lightning: Progressive adversarial diffusion distillation},
  author={Lin, Shanchuan and Wang, Anran and Yang, Xiao},
  journal={arXiv preprint arXiv:2402.13929},
  year={2024}
}

@inproceedings{kimconsistency,
  title={Consistency Trajectory Models: Learning Probability Flow ODE Trajectory of Diffusion},
  author={Kim, Dongjun and Lai, Chieh-Hsin and Liao, Wei-Hsiang and Murata, Naoki and Takida, Yuhta and Uesaka, Toshimitsu and He, Yutong and Mitsufuji, Yuki and Ermon, Stefano},
  booktitle={The Twelfth International Conference on Learning Representations}
}

@inproceedings{huberman2024edit,
  title={An edit friendly ddpm noise space: Inversion and manipulations},
  author={Huberman-Spiegelglas, Inbar and Kulikov, Vladimir and Michaeli, Tomer},
  booktitle={Proceedings of the IEEE/CVF Conference on Computer Vision and Pattern Recognition},
  pages={12469--12478},
  year={2024}
}
}


\maketitlesupplementary
\appendix

\tableofcontents

\section{Full Related Work}

\paragraph{GAN-based image editing.}
Prior to the dominance of diffusion models, generative adversarial networks (GANs) provided strong latent-space controllability. This included latent-space traversal methods~\cite{hertz2022prompt, tumanyan2023plug}, encoder-based inversion for real image editing, text-driven manipulation, and CLIP-guided zero-shot domain adaptation. While offering intuitive control, these methods often face challenges in domain generalization and high-resolution reconstruction for real-world images.

\paragraph{Fast One-Step T2I Backbones.}
The primary enabler for real-time editing is the advent of high-speed generators. These models, often distilled from large-scale diffusion models~\cite{song2023consistency, lin2024sdxl}, can synthesize high-fidelity images in a single step. Key examples include adversity-matching or rectified-flow-based generators like SD-Turbo~\cite{Sauer2024ADD}, InstaFlow~\cite{liu2023instaflow}, and SwiftBrush-v2~\cite{Dao2025SwiftBrushV2}. These backbones provide the foundation for our work, but their fast, non-linear dynamics pose unique challenges for control.

\paragraph{Text-guided editing with diffusion/flow models.}
Prior work on editing with these backbones falls into several categories with distinct trade-offs:
\begin{itemize}
    \item \textbf{Training-free, Inversion-Required.} This is a common multi-step paradigm. Methods first reconstruct a latent representation of the source image via inversion, then steer the generation process using attention control or guidance, such as in PnP~\cite{hertz2022prompt, tumanyan2023plug}, NPI~\cite{mokady2023null, npi25}, or ProxEdit~\cite{ju2023direct, proxedit24, edict23, aidi23, spdinv24}. While effective, their reliance on iterative inversion and multi-step sampling (e.g., 30-50 steps) makes real-time application infeasible.

    \item \textbf{Training-free, Inversion-Free.} These methods avoid costly per-image inversion, often relying on short trajectories~\cite{meng2021sdedit, diffedit22}. Differential update strategies, such as FlowEdit~\cite{kulikov2025flowedit} and InfEdit~\cite{xu2023inversion}, are stabilized by sample averaging over multiple steps. However, when forced into the one-step limit, this approach collapses, concentrating high energy and variance into a single, unstable transport step, leading to the failures.

    \item \textbf{Accelerated and Few-Step Editing.} Leveraging fast backbones like SDXL-Turbo~\cite{Sauer2024ADD}, methods such as TurboEdit~\cite{deutch2024turboedit} and InstantEdit~\cite{gong2025instantedit, cora25} significantly reduce latency. However, they still require 4-8 steps and often rely on inversion, falling short of true one-step, instant interaction~\cite{xu2025textvdb, si2025pix2pixzerocon}.
    
    \item \textbf{Training-based One-Step Editing.} To achieve true one-step editing, methods like SwiftEdit~\cite{icd24, nguyen2025swiftedit, kawar2023imagic} are proposed. SwiftEdit highlights that one-step editing depends on accurate one-step inversion. It achieves this by training a \emph{dedicated inversion network} to predict the noise, enabling a one-step reconstruction and edit. This reliance on extra training, however, sacrifices the model-agnostic, training-free nature that is critical for broad applicability. COT Flow~\cite{zu2024cotflow} uses post-training for one-step transport but is limited to low-resolution 64$\times$64 editing.
\end{itemize}

\paragraph{Our Method.}
Different from GAN-based editing and multi-step diffusion editors, \textbf{ChordEdit} operates in the challenging training-free, inversion-free, single-step regime. 
Instead of relying on multi-step differential updates (i.e. InfEdit/ FlowEdit) or trained inversion networks (i.e. SwiftEdit), we introduce a \emph{Chord control field}. This field is constructed directly in the observable residual domain to average and stabilize the high-energy control signal. Paired with a proximal refinement, our method achieves a low-energy, low-variance transport, finally facilitating consistent editing on fast backbones without training or inversion.

\section{From Dynamic Optimal Transport to the Chord Control Field}
\label{sec:chordedit_ot}

\paragraph{Standing primitives (from the preliminaries).}

Given prompts $c_{\rm src}$ and $c_{\rm tar}$, we draw the source and target endpoint images as
\[
x_{\rm src}:=x_1 \sim p_{1}(x\mid c_{\rm src}),
\qquad
x_{\rm tar}:=x_0 \sim p_{0}(x\mid c_{\rm tar}).
\]
Throughout, $x\in\mathbb{R}^d$ denotes the spatial variable and $t\in[0,1]$ is the (diffusion/editing) time, with the convention that $t=1$ corresponds to the source endpoint and $t=0$ to the target endpoint.

\subsection{Definitions used in the derivation}

\paragraph{(D1) Editing density path.}
We denote by $\{\rho_t\}_{t\in[0,1]}$ a time-indexed family of densities with boundary conditions
\begin{equation}
\label{eq:bb-bc}
\rho_{1}(\cdot)=p_{1}(\cdot\mid c_{\rm src}),
\qquad
\rho_{0}(\cdot)=p_{0}(\cdot\mid c_{\rm tar}).
\end{equation}

\paragraph{(D2) Editing vector field.}
The \emph{editing vector field} $u_t:\mathbb{R}^d\!\to\!\mathbb{R}^d$ is the (complete) probability flow that transports $\rho_1$ to $\rho_0$ via the continuity equation
\begin{equation}
\label{eq:continuity}
\partial_t \rho_t(x)+\nabla\!\cdot\!\big(\rho_t(x)\,u_t(x)\big)=0
\quad\text{for }t\in(0,1).
\end{equation}
We emphasize that $u_t$ is \emph{not} an additive residual on top of another reference field; it is the unique driver of the editing transport in our formulation.

\paragraph{(D3) Observable surrogate of $u_t$ at an anchor.}
Fix an \emph{anchor} $x_\tau$ (in practice we take $x_\tau:=x_{\rm src}$). Let $Q(z,t,c)$ be the model’s observable at noisy state $z$ and time $t$ (e.g., noise/velocity prediction).
Let $\mathcal{B}_t$ be a time-dependent linear map that converts the model’s units to velocity units, and let
$K_t(\cdot\mid x_\tau)$ denote the corruption/noising kernel that produces $z$ conditioned on $x_\tau$ at time $t$.
Define the observable surrogate
\begin{equation}
\label{eq:R-def}
\mathbf{R}(x_\tau,t)
\;:=\;
\mathbb{E}_{z\sim K_t(\cdot\mid x_\tau)}
\!\big[\;\mathcal{B}_t\big(Q(z,t,c_{\rm tar})-Q(z,t,c_{\rm src})\big)\big].
\end{equation}

\paragraph{(A1) Local observability (measurement model).}
At the anchor and within a short temporal window, the surrogate is an unbiased noisy measurement of the editing field:
\begin{equation}
\label{eq:meas}
\mathbf{R}(x_\tau,t)\;=\;u_t(x_\tau)\;+\;\varepsilon_t,
\qquad
\mathbb{E}[\varepsilon_t]=0.
\end{equation}

\paragraph{(A2) Short-window local homogeneity.}
Fix a small window $[t-\delta,t]$ with $\delta>0$.
Within the anchor’s neighborhood, $u_\xi(x)\approx u$ is (approximately) constant w.r.t.\ both $x$ and $\xi\in[t-\delta,t]$; the local mass factor $\rho_\xi(x)$ can be absorbed as a (positive) scalar weight.

\paragraph{(A3) Recursive energy prior.}
The kinetic energy accumulated on $[0,t-\delta]$ induces a quadratic prior around the previous estimate $\hat u_{t-\delta}(x_\tau)$ with weight proportional to the elapsed time $t$.

\subsection{Dynamic OT objective with $t$ as progress}

The (unregularized) Benamou--Brenier dynamic OT problem directly in the time variable $t$ reads
\begin{equation}
\label{eq:bb}
\begin{aligned}
\min_{\rho,\,u}\quad 
& \int_0^1\!\!\int \frac{1}{2}\,\|u_t(x)\|^2\,\rho_t(x)\,dx\,dt \\[2pt]
\text{s.t.}\quad 
& \partial_t \rho_t(x)+\nabla\!\cdot\!\big(\rho_t(x)\,u_t(x)\big)=0,\\
& \rho_{1}=p_1(\cdot\mid c_{\rm src}),\ \ \rho_{0}=p_0(\cdot\mid c_{\rm tar}).
\end{aligned}
\end{equation}
This formulation treats $u_t$ as the \emph{complete} field that transports $\rho_1$ to $\rho_0$; no reference flow is introduced.

\subsection{Local, myopic surrogate of \eqref{eq:bb} and MAP objective}

To obtain a causal, single-step estimator of $u_t$ at the anchor, we combine (A1)--(A3) over the short window $[t-\delta,t]$ into the following convex quadratic objective:
\begin{equation}
\label{eq:Phi}
\Phi_t(u;\,x_\tau)\;=\;
t\,\big\|u-\hat u_{t-\delta}(x_\tau)\big\|^2
\;+\;
\int_{t-\delta}^{t}\!\big\|u-\mathbf{R}(x_\tau,\xi)\big\|^2\,d\xi.
\end{equation}
The first term encodes the recursive energy prior (A3); the second term enforces local agreement with the measurements (A1) under local homogeneity (A2).
All global density factors can be absorbed into the (time) weights without changing the closed form below.

\subsection{Closed-form minimizer on the window}

Differentiating \eqref{eq:Phi} w.r.t.\ $u$ and setting the gradient to zero gives
\begin{equation}
\label{eq:normal-eq}
2t\,(u-\hat u_{t-\delta})\;+\;2\!\int_{t-\delta}^{t}\!(u-\mathbf{R}(x_\tau,\xi))\,d\xi\;=\;0,
\end{equation}
whence the unique minimizer is
\begin{equation}
\label{eq:window-opt}
u_t^\star(x_\tau)\;=\;
\frac{t}{t+\delta}\,\hat u_{t-\delta}(x_\tau)
\;+\;
\frac{1}{t+\delta}\,\int_{t-\delta}^{t}\mathbf{R}(x_\tau,\xi)\,d\xi.
\end{equation}
Equation \eqref{eq:window-opt} is exact under (A1)--(A3).

\subsection{Causal first-order approximation (Chord estimator)}

For an online single-step implementation, we apply two standard first-order causal approximations on the window $[t-\delta,t]$:
\begin{equation}
\label{eq:causal-approx}
\begin{aligned}
&\int_{t-\delta}^{t}\mathbf{R}(x_\tau,\xi)\,d\xi \;\approx\; \delta\,\mathbf{R}(x_\tau,t), \\
&\hat u_{t-\delta}(x_\tau)\;\approx\; \mathbf{R}(x_\tau,t-\delta).
\end{aligned}
\end{equation}
Substituting \eqref{eq:causal-approx} into \eqref{eq:window-opt} yields the \emph{Chord} estimate used in our implementation:
\begin{equation}
\label{eq:chord}
\boxed{\quad
\hat u_t(x_\tau)
\;=\;
\frac{\,t\,\mathbf{R}(x_\tau,t-\delta)\;+\;\delta\,\mathbf{R}(x_\tau,t)\,}{t+\delta}.
\quad}
\end{equation}

\paragraph{Remarks on interpretation and accuracy.}
(i) By construction, $u_t$ (and its estimator $\hat u_t$) is the \emph{full} editing flow that drives \eqref{eq:continuity} from the source boundary $\rho_1=p_1(\cdot\mid c_{\rm src})$ to the target boundary $\rho_0=p_0(\cdot\mid c_{\rm tar})$ in \eqref{eq:bb-bc}. 
(ii) The approximation error of \eqref{eq:chord} relative to \eqref{eq:window-opt} is $O(\delta)$ under standard smoothness, while measurement noise enters via \eqref{eq:meas}.
(iii) If desired, one may replace the scalar time-weights $(t,\delta)$ in \eqref{eq:Phi} by density-weighted effective durations without changing the closed-form structure in \eqref{eq:window-opt}--\eqref{eq:chord}.

\section{Unified Comparison-Domain Map and Closed-Form Coefficients}
\label{sec:Bt}

Our framework's model-agnosticism hinges on a linear, time-dependent map $\mathcal{B}_t$, which projects the output $Q$ of any given model into a unified comparison domain $U$ (specifically, the domain of the velocity field $u_t$). This section provides the closed-form derivations for $\mathcal{B}_t$ under various common model parameterizations.

\subsection{General Formulation}
We begin with the general forward noising path, which maps a clean image $x_0$ to a noisy state $x_t$ using a noise sample $\epsilon \sim \mathcal{N}(0, I)$:
\begin{equation}\label{eq:path}
x_t=\alpha(t)\,x_0+\sigma(t)\,\epsilon.
\end{equation}
The corresponding continuous-time velocity (or drift) $u_t$ is the time-derivative of this path:
\begin{equation}
u_t:=\dot x_t=\dot\alpha(t)\,x_0+\dot\sigma(t)\,\epsilon.
\end{equation}
Our goal is to find the map $\mathcal{B}_t$ such that for any model output $\Delta Q = Q(z_t,t,c_{\rm tar}) - Q(z_t,t,c_{\rm src})$, we have $\mathcal{B}_t(\Delta Q) \approx \Delta u_t$. We compute this difference using a shared-noise sample $z_t$, which implies a fixed $x_t$. This "fixed $x_t$" constraint is key, as it implies $\Delta x_t = 0$:
\begin{equation}
\Delta(\alpha(t)\,x_0+\sigma(t)\,\epsilon) = 0 \implies \alpha(t)\,\Delta x_0+\sigma(t)\,\Delta\epsilon=0.
\end{equation}
Assuming $\alpha(t) \neq 0$, this provides a direct linear relationship between the change in the predicted data $\Delta x_0$ and the change in the predicted noise $\Delta \epsilon$:
\begin{equation}
\Delta x_0 = -(\sigma(t)/\alpha(t))\Delta\epsilon.
\end{equation}

\subsection{Noise-Prediction models}
This is the most common parameterization, used by models like SD-Turbo. The model directly predicts the noise sample: $Q(z_t,t,c) = \hat\epsilon_\theta(z_t,t,c)$. We therefore have $\Delta Q = \Delta \hat\epsilon$, and we assume $\Delta \hat\epsilon \approx \Delta \epsilon$.

To find the map $\mathcal{B}_t$, we express the velocity difference $\Delta u_t$ purely in terms of $\Delta \epsilon$ by substituting the constraint for $\Delta x_0$:
\begin{align}
\Delta u_t &= \dot\alpha(t)\,\Delta x_0+\dot\sigma(t)\,\Delta\epsilon \\
 &= \dot\alpha(t)\,\Big(-\frac{\sigma(t)}{\alpha(t)}\Delta\epsilon\Big) + \dot\sigma(t)\,\Delta\epsilon \\
 &= \Big(\dot\sigma(t)-\frac{\dot\alpha(t)}{\alpha(t)}\sigma(t)\Big)\Delta\epsilon.
\end{align}
This gives us a scalar coefficient $A_t^{(\epsilon)}$ that defines the map $\mathcal{B}_t(\Delta Q) = A_t^{(\epsilon)} \Delta Q$:
\begin{equation}\label{eq:Aeps}
\boxed{\ A_t^{(\epsilon)}=\dot\sigma(t)-\frac{\dot\alpha(t)}{\alpha(t)}\,\sigma(t)\ }.
\end{equation}

For the common Variance-Preserving (VP) schedule, where $\alpha(t)^2+\sigma(t)^2 \equiv 1$, we have $2\alpha\dot\alpha + 2\sigma\dot\sigma = 0$, which implies $\dot\sigma(t) = -(\alpha(t)/\sigma(t))\dot\alpha(t)$. Substituting this into Eq.~\eqref{eq:Aeps} yields a simplified form:
\begin{align}
A_t^{(\epsilon)} &= \Big(-\frac{\alpha(t)}{\sigma(t)}\dot\alpha(t)\Big) - \frac{\dot\alpha(t)}{\alpha(t)}\sigma(t) \\
 &= -\frac{\dot\alpha(t)}{\alpha(t)\sigma(t)}\big(\alpha(t)^2 + \sigma(t)^2\big) \\
 &= -\frac{\dot\alpha(t)}{\alpha(t)\sigma(t)}.
\end{align}
\begin{equation}\label{eq:Aeps-VP}
\boxed{\ A_t^{(\epsilon)}=-\,\frac{\dot\alpha(t)}{\alpha(t)\,\sigma(t)}\ }.
\end{equation}
If the schedule is further parameterized by a continuous-time $\beta(t)$ such that $\alpha(t)=\exp\!\big(-\frac12\int_0^t\beta(s)\,ds\big)$, then $\dot\alpha(t)=-\tfrac12\beta(t)\alpha(t)$. This gives the final form:
\begin{equation}\label{eq:Aeps-beta}
A_t^{(\epsilon)}=\frac{\beta(t)}{2\,\sigma(t)}.
\end{equation}

\subsection{Velocity and Flow-Matching models}
This is the most direct case, used by models like InstaFlow. The model output is designed to directly predict the velocity: $Q(z_t,t,c) = u_\theta(z_t,t,c) \approx u_t$.

Therefore, the output difference $\Delta Q$ is already in the target comparison domain, and the map $\mathcal{B}_t$ is simply the identity:
\begin{equation}
\boxed{\ \mathbf{R}(x_\tau,t)=\mathbb{E}_{z\sim K_t(\cdot\mid x_\tau)}[\Delta u_\theta(z,t)],\qquad \mathcal{B}_t\equiv I\ }.
\end{equation}

\subsection{Other Common Parameterizations}
We can derive coefficients for $x_0$-prediction and $v$-prediction models using the same principles, assuming a VP schedule for simplicity.

\paragraph{$x_0$-Prediction} Here, $Q = \hat x_0$, so $\Delta Q = \Delta \hat x_0 \approx \Delta x_0$. We map $\Delta x_0$ to $\Delta u_t$ using the constraint $\Delta\epsilon = -(\alpha(t)/\sigma(t))\Delta x_0$:
\begin{align}
\Delta u_t &= \dot\alpha(t)\,\Delta x_0+\dot\sigma(t)\,\Delta\epsilon \\
 &= \dot\alpha(t)\,\Delta x_0+\dot\sigma(t)\,\Big(-\frac{\alpha(t)}{\sigma(t)}\Delta x_0\Big) \\
 &= \Big(\dot\alpha(t) - \frac{\dot\sigma(t)\alpha(t)}{\sigma(t)}\Big) \Delta x_0.
\end{align}
Using the VP relations $\dot\sigma = -(\alpha/\sigma)\dot\alpha$ and $\alpha^2+\sigma^2=1$:
\begin{align}
\Delta u_t &= \Big(\dot\alpha(t) - \frac{(-\alpha(t)/\sigma(t))\dot\alpha(t)\,\alpha(t)}{\sigma(t)}\Big) \Delta x_0 \\
 &= \Big(\dot\alpha(t) + \frac{\alpha(t)^2 \dot\alpha(t)}{\sigma(t)^2}\Big) \Delta x_0 \\
 &= \dot\alpha(t)\,\Big(\frac{\sigma(t)^2 + \alpha(t)^2}{\sigma(t)^2}\Big) \Delta x_0.
\end{align}
This yields the map $\mathcal{B}_t(\Delta Q) = A_t^{(x_0)} \Delta Q$:
\begin{equation}
A_t^{(x_0)} = \boxed{\ \frac{\dot\alpha(t)}{\sigma(t)^2}\ }.
\end{equation}

\paragraph{$v$-Prediction} Here, $Q = \hat v$, where $v := \alpha\epsilon - \sigma x_0$. Under a VP schedule, the difference $\Delta v$ relates to $\Delta \epsilon$ as:
\begin{align}
\Delta v &= \alpha \Delta \epsilon - \sigma \Delta x_0 \\
 &= \alpha \Delta \epsilon - \sigma \Big(-\frac{\sigma}{\alpha}\Delta\epsilon\Big) \\
 &= \Big(\frac{\alpha^2+\sigma^2}{\alpha}\Big) \Delta \epsilon = \frac{1}{\alpha} \Delta \epsilon.
\end{align}
Thus, $\Delta \epsilon = \alpha \Delta v$. We map $\Delta v$ to $\Delta u_t$ using the coefficient from Eq.~\eqref{eq:Aeps-VP}:
\begin{equation}
\Delta u_t = A_t^{(\epsilon)} \Delta \epsilon = \Big(-\frac{\dot\alpha(t)}{\alpha(t)\sigma(t)}\Big) (\alpha(t) \Delta v).
\end{equation}
This gives the map $\mathcal{B}_t(\Delta Q) = A_t^{(v)} \Delta Q$:
\begin{equation}
A_t^{(v)} = \boxed{\ -\,\frac{\dot\alpha(t)}{\sigma(t)}\ }.
\end{equation}

\paragraph{Score-to-drift}
If we interpret $u_t$ as the drift of the reverse-time SDE associated with the VP schedule, then score-based generative models parameterize this drift by predicting a score field $s_\theta(x_t,t) \approx \nabla_x \log p_t(x_t)$ and converting it to a drift via a scalar time-dependent factor.
For a VP forward SDE with noise-rate schedule $\beta(t)$, the reverse dynamics can be written as
\begin{equation}
  d x_t = \Big(f_t(x_t) + \beta(t)\,s_\theta(x_t,t)\Big)\,dt + \sqrt{\beta(t)}\,dW_t,
\end{equation}
where $f_t$ collects the parts of the drift that do not depend on the learned score and $\beta(t)$ is the same schedule as in the VP path.
For a fixed noisy state $x_t$ the change in drift between two conditioning signals is therefore
\begin{equation}
  \Delta u_t = \beta(t)\,\Delta s_\theta(x_t,t).
\end{equation}
Thus these models fit our template with $Q = \hat s$ and $\Delta Q \approx \Delta s_\theta$, and the map $\mathcal{B}_t$ is again time-only:
\begin{equation}
  \mathcal{B}_t(\Delta Q) = A_t^{(\mathrm{score})}\,\Delta Q,
  \qquad
  A_t^{(\mathrm{score})} = \boxed{\ \beta(t)\ }.
\end{equation}
In particular, no dependence on $x_t$ appears in $A_t^{(\mathrm{score})}$ beyond the scalar schedule $\beta(t)$.

\paragraph{Consistency models}
Consistency models operate on the probability-flow ODE of an underlying diffusion process and learn a ``consistency function'' $f_\theta(z_t,t)$ that maps any point $z_t$ on an ODE trajectory back to its origin $x_0$~\cite{song2023consistency}.
Under the parameterization in Eq.~\eqref{eq:path}, this means we can simply interpret the model output as an $x_0$-prediction,
\begin{equation}
  Q = \hat x_0 = f_\theta(z_t,t), \qquad \Delta Q \approx \Delta x_0.
\end{equation}
Therefore the same derivation as in the $x_0$-prediction case applies, and we can reuse the coefficient $A_t^{(x_0)} = \dot\alpha(t)/\sigma(t)^2$:
\begin{equation}
  \Delta u_t = A_t^{(x_0)}\,\Delta x_0
  = \Big(\frac{\dot\alpha(t)}{\sigma(t)^2}\Big)\Delta x_0.
\end{equation}
Again, the dependence on the particular model enters only through $\Delta Q$; the map $\mathcal{B}_t$ itself remains a linear, time-only scaling.

\subsection{Implementation and Numerical Stability}
In a practical implementation, the continuous-time derivatives $\dot\alpha(t)$ and $\dot\sigma(t)$ are approximated using first-order finite differences, e.g., $\dot\alpha(t)\approx (\alpha(t)-\alpha(t-\delta))/\delta$.

A critical consideration is that the maps involving $\alpha(t)$ in the denominator (e.g., $A_t^{(\epsilon)}$) become numerically unstable as $\alpha(t) \to 0$, i.e., at $t \approx 1$. Our ChordEdit method queries the field at times $t$ and $t-\delta$ (e.g., $t=0.90, \delta=0.15$), which are bounded away from $t=1$, ensuring $\alpha(t)$ is non-negligible and the map $\mathcal{B}_t$ is well-conditioned.

\section{Energy Contraction Property of the Chord Control Field}
\label{sec:energy-contraction}

In this section, we provide the formal justification for the low-energy property of the Chord Control Field. We first demonstrate that the chord estimator, as a form of temporal smoothing, acts as an $L^2$-energy contraction on the underlying observable proxy field. We then show how this leads to the pointwise energy bound presented in the main text.

We begin by formalizing the chord field $\hat u$ as a generalized temporal smoothing operation on the observable proxy field $\mathbf{R}$. The specific estimator derived in Eq.~\eqref{eq:umin} is one such example.
Let the time-dependent proxy field for a fixed anchor $x_\tau$ be denoted $\mathbf{R}(t) := \mathbf{R}(x_\tau, t)$. We define the corresponding chord field $\hat u(t) := \hat u_t(x_\tau)$ via a causal convolution with a smoothing kernel $K_\delta(s)$:
\begin{equation}
\hat u(t) = (K_\delta * \mathbf{R})(t) := \int_{-\infty}^{\infty} K_\delta(s) \mathbf{R}(t-s) \,ds.
\end{equation}
To match the derivation in Sec.~\ref{sec:chordedit_ot}, this kernel $K_\delta$ is assumed to be \textbf{non-negative} ($K_\delta(s) \ge 0$), have \textbf{unit mass} ($\int K_\delta(s)\,ds = 1$), and be \textbf{causal} (e.g., $\operatorname{supp}(K_\delta) \subset [0, \delta]$ for the integral form in Eq.~\eqref{eq:umin}, or a discrete recursive form). This frames $\hat u(t)$ as a weighted average, or expectation, of the recent history of $\mathbf{R}$.

\begin{proposition}[$L^2$-Energy Contraction]
\label{prop:l2-contraction}
Let the observable proxy field $\mathbf{R}(t)$ be in $L^2([0,1]; \mathbb{R}^d)$, and let the chord field $\hat u(t)$ be generated by convolution with any non-negative, unit-mass kernel $K_\delta$ as defined above. The total temporal kinetic energy of the chord field is strictly less than or equal to that of the proxy field:
\begin{equation}
\int_0^1 \|\hat u(t)\|^2 \,dt \le \int_0^1 \|\mathbf{R}(t)\|^2 \,dt.
\end{equation}
Furthermore, the inequality is \emph{strict} if $K_\delta$ is not a Dirac delta function (i.e., it performs non-trivial averaging) and $\mathbf{R}(t)$ is not almost-everywhere constant on $[0,1]$.
\end{proposition}

\begin{proof}
The proof relies on the strict convexity of the squared $\ell_2$-norm and Jensen's inequality.

\noindent\textbf{1. Pointwise Jensen's Inequality.}
For any fixed time $t$, we recognize $\hat u(t)$ as the expectation of a vector-valued random variable $Z_s := \mathbf{R}(t-s)$, where the probability measure is $d\mathbb{P}(s) = K_\delta(s)\,ds$.
\begin{equation}
\hat u(t) = \int_{\mathbb{R}} \mathbf{R}(t-s) K_\delta(s)\,ds = \mathbb{E}_{s \sim K_\delta}[\mathbf{R}(t-s)].
\end{equation}
Let $\varphi(z) = \|z\|^2$. This function is strictly convex on $\mathbb{R}^d$. By Jensen's inequality:
\begin{equation}
\label{eq:proof-jensen-pointwise}
\|\hat u(t)\|^2 = \varphi(\mathbb{E}[Z_s]) \le \mathbb{E}[\varphi(Z_s)] = \int_{\mathbb{R}} \|\mathbf{R}(t-s)\|^2 K_\delta(s)\,ds.
\end{equation}
Equality holds if and only if the random variable $Z_s$ is almost-everywhere constant, i.e., $\mathbf{R}(t-s)$ is constant for $s$ in the support of $K_\delta$.

\noindent\textbf{2. Integration and Fubini's Theorem.}
We integrate this pointwise inequality over the time interval $t \in [0, 1]$. For simplicity, we can consider all functions to be zero-padded outside $[0,1]$ and integrate over $\mathbb{R}$.
\begin{align*}
\int_0^1 \|\hat u(t)\|^2 \,dt
&\le \int_{\mathbb{R}} \left( \int_{\mathbb{R}} \|\mathbf{R}(t-s)\|^2 K_\delta(s)\,ds \right) \,dt \\
&= \int_{\mathbb{R}} K_\delta(s) \left( \int_{\mathbb{R}} \|\mathbf{R}(t-s)\|^2 \,dt \right) \,ds
\end{align*}
where we have exchanged the order of integration by Tonelli's theorem (as the integrand is non-negative).
By substituting $\tau = t-s$ (a simple shift), the inner integral becomes $\int_{\mathbb{R}} \|\mathbf{R}(\tau)\|^2 \,d\tau = \int_0^1 \|\mathbf{R}(t)\|^2 \,dt$.
\begin{align*}
\int_0^1 \|\hat u(t)\|^2 \,dt
&\le \int_{\mathbb{R}} K_\delta(s) \left( \int_0^1 \|\mathbf{R}(t)\|^2 \,dt \right) \,ds \\
&= \left( \int_{\mathbb{R}} K_\delta(s) \,ds \right) \left( \int_0^1 \|\mathbf{R}(t)\|^2 \,dt \right).
\end{align*}
Since the kernel $K_\delta$ has unit mass ($\int K_\delta(s)\,ds = 1$), we arrive at the desired result:
\begin{equation}
\int_0^1 \|\hat u(t)\|^2 \,dt \le \int_0^1 \|\mathbf{R}(t)\|^2 \,dt.
\end{equation}

\noindent\textbf{3. Strict Inequality.}
The inequality is strict if the pointwise Jensen inequality in Eq.~\eqref{eq:proof-jensen-pointwise} is strict on a set of $t$ with positive measure. This occurs if $\mathbf{R}(t-s)$ is not constant w.r.t. $s$ on the support of $K_\delta$. If $K_\delta$ is not a Dirac delta (i.e., its support has positive measure) and $\mathbf{R}(t)$ is not almost-everywhere constant, this condition will be met, guaranteeing a strict reduction in total energy.
\end{proof}

\begin{remark}[Contraction of Benamou–Brenier Energy]
This proposition extends directly to the full Benamou–Brenier energy functional. The proof above applies pointwise for every $x \in \mathbb{R}^d$.
\begin{equation}
\int_0^1 \frac{1}{2}\|\hat u(x,t)\|^2 \,dt \le \int_0^1 \frac{1}{2}\|\mathbf{R}(x,t)\|^2 \,dt.
\end{equation}
We can then multiply by the non-negative density $\rho_t(x)$ and integrate over $x$:
\begin{align*}
\mathcal{E}[\hat u; \rho]
&= \int_0^1 \! \int \frac{1}{2}\|\hat u(x,t)\|^2 \rho_t(x) \,dx \,dt \\
&\le \int_0^1 \! \int \frac{1}{2}\|\mathbf{R}(x,t)\|^2 \rho_t(x) \,dx \,dt = \mathcal{E}[\mathbf{R}; \rho].
\end{align*}
Thus, the Chord Control Field $\hat u$ generates a dynamic flow with strictly lower (or equal) kinetic energy than the naive proxy field $\mathbf{R}$.
\end{remark}

This general $L^2$-contraction property is instantiated in our specific estimator. The general minimizer from Eq.~\eqref{eq:umin} is a convex combination of the prior $\hat u_{t-\delta}$ and the observations $\mathbf{R}(\xi)$ over the window.

\begin{corollary}[Pointwise Energy Bound]
\label{cor:pointwise-bound}
The Chord Control Field estimator $\hat u_t(x_\tau)$ from the general solution
\begin{equation}
\hat u_t(x_\tau) = \frac{W_t\,\hat u_{t-\delta}(x_\tau)+\int_{t-\delta}^{t}\mathbf{R}(x_\tau,\xi)\,d\xi}{W_t+\delta I}
\label{eq:umin}
\end{equation}
(assuming $W_t$ is a scalar multiple of identity, $W_t = w_t I$) satisfies the pointwise energy bound:
\begin{equation}
\|\hat u_t(x_\tau)\|^2 \le \frac{w_t\,\|\hat u_{t-\delta}(x_\tau)\|^2+\int_{t-\delta}^{t}\|\mathbf{R}(x_\tau,\xi)\|^2\,d\xi}{w_t+\delta}.
\end{equation}
Furthermore, applying the first-order approximations from the main text (namely $w_t=t$, $\hat u_{t-\delta}(x_\tau) \approx \mathbf{R}(x_\tau, t-\delta)$, and $\int_{t-\delta}^{t}\mathbf{R}(\cdot,\xi)d\xi \approx \delta \mathbf{R}(\cdot,t)$) yields the final bound:
\begin{equation}
\|\hat u_t(x_\tau)\|^2 \le \frac{t\,\|\mathbf{R}(x_\tau,t-\delta)\|^2+\delta\,\|\mathbf{R}(x_\tau,t)\|^2}{t+\delta}.
\end{equation}
\end{corollary}

\begin{proof}
The estimator $\hat u_t(x_\tau)$ is a convex combination (a weighted average) of the vectors $\{\hat u_{t-\delta}(x_\tau)\} \cup \{\mathbf{R}(x_\tau,\xi)\}_{\xi \in [t-\delta, t]}$. The first inequality follows directly from applying Jensen's inequality to the strictly convex function $\varphi(z) = \|z\|^2$. The second inequality is a direct application of the same principle to the final approximated estimator $\hat u_t(x_\tau) = \frac{t}{t+\delta}\mathbf{R}(x_\tau, t-\delta) + \frac{\delta}{t+\delta}\mathbf{R}(x_\tau, t)$, which is a convex combination of the two endpoint proxy fields.
\end{proof}

\begin{lemma}[Local truncation error of explicit Euler under a edit control field]\label{lem:euler-local-truncation}
Consider the ODE over one step $[t_n,t_{n+1}]$ with $t_{n+1}=t_n+h$,
\begin{equation}\label{eq:ode}
\dot x(t)=u(x(t),t),\quad x(t)\in\mathbb{R}^d .
\end{equation}
Assume there exists a set $\mathcal{U}\subset\mathbb{R}^d\times[t_n,t_{n+1}]$ that contains the exact trajectory and on which
\begin{equation}\label{eq:bounds}
\|\partial_t u\|_{L^\infty(\mathcal{U})},\ \|\partial_x u\|_{L^\infty(\mathcal{U})},\ 
\|u\|_{L^\infty(\mathcal{U})} \;<\; \infty .
\end{equation}

Let the one–step local truncation error be
\begin{equation}\label{eq:tau}
\tau_{n+1} \;:=\; x(t_{n+1})-\bigl(x(t_n)+h\,f(x(t_n),t_n)\bigr).
\end{equation}
Then
\begin{equation}\label{eq:lte-bound}
\|\tau_{n+1}\| \;\le\; \frac{h^2}{2}\,M_f,
\end{equation}
where
\begin{equation}\label{eq:Mf-def}
M_f \;:=\; \sup_{(x,t)\in\mathcal{U}} \bigl\|\,\partial_t f(x,t)+\partial_x f(x,t)\,f(x,t)\,\bigr\|.
\end{equation}
In particular, since $f=u$,
\begin{equation}\label{eq:Mf-upper}
M_f \;\le\; \|\partial_t u\|_{L^\infty(\mathcal{U})}
+ \|\partial_x u\|_{L^\infty(\mathcal{U})}\,\|u\|_{L^\infty(\mathcal{U})}.
\end{equation}
\end{lemma}

\begin{proof}
By the variation-of-constants formula,
\begin{equation}\label{eq:varconst}
x(t_{n+1}) \;=\; x(t_n) + \int_{t_n}^{t_{n+1}} f\bigl(x(s),s\bigr)\,ds .
\end{equation}
Using the chain rule and $\dot x(s)=f(x(s),s)$,
\begin{equation}\label{eq:chain}
\begin{aligned}
\frac{d}{ds} f\bigl(x(s),s\bigr) 
&= \partial_t f\bigl(x(s),s\bigr) + \partial_x f\bigl(x(s),s\bigr)\,\dot x(s)\\
&= \partial_t f\bigl(x(s),s\bigr) + \partial_x f\bigl(x(s),s\bigr)\,f\bigl(x(s),s\bigr).
\end{aligned}
\end{equation}
Integrating \eqref{eq:chain} from $t_n$ to $s\in[t_n,t_{n+1}]$ gives
\begin{equation}\label{eq:int-identity}
\begin{aligned}
f\bigl(x(s),s\bigr) 
&= f\bigl(x(t_n),t_n\bigr)
 + \int_{t_n}^{s} \bigl(\partial_t f + \partial_x f\,f\bigr)\bigl(x(r),r\bigr)\,dr .
\end{aligned}
\end{equation}
Insert \eqref{eq:int-identity} into \eqref{eq:varconst} and subtract the explicit Euler update:
\begin{equation}\label{eq:double-int}
\begin{aligned}
\tau_{n+1}
&= \int_{t_n}^{t_{n+1}}\!\!\Bigl( f\bigl(x(s),s\bigr)-f\bigl(x(t_n),t_n\bigr)\Bigr)\, ds\\
&= \int_{t_n}^{t_{n+1}}\!\!\int_{t_n}^{s} 
\bigl(\partial_t f + \partial_x f\,f\bigr)\bigl(x(r),r\bigr)\,dr\,ds .
\end{aligned}
\end{equation}
Taking norms and using the definition of $M_f$,
\begin{equation}\label{eq:lte-final}
\begin{aligned}
\|\tau_{n+1}\|
&\le \int_{t_n}^{t_{n+1}}\!\!\int_{t_n}^{s} 
\bigl\| \bigl(\partial_t f + \partial_x f\,f\bigr)\bigl(x(r),r\bigr)\bigr\|\,dr\,ds\\
&\le \int_{t_n}^{t_{n+1}}\!\!\int_{t_n}^{s} M_f\,dr\,ds
= \frac{h^2}{2}\,M_f ,
\end{aligned}
\end{equation}
which proves \eqref{eq:lte-bound}. Since $f=u$, we directly have
\begin{equation}\label{eq:mf-upper-derivation}
\bigl\|\partial_t f + \partial_x f\,f\bigr\|
\;\le\; \|\partial_t u\| + \|\partial_x u\|\,\|u\|,
\end{equation}
from which \eqref{eq:Mf-upper} follows by taking the supremum over $\mathcal{U}$.
\end{proof}

\begin{proposition}[Consistency bound for the chord control field]\label{prop:consistency-chord}
Let
\begin{equation}\label{eq:fields}
\begin{aligned}
f_{\mathrm{nai}}(x,t) &= \mathbf{R}(x,t),\\
f_{\mathrm{cho}}(x,t) &= (K_\delta * \mathbf{R})(x,t),
\end{aligned}
\end{equation}
where the convolution acts only in time,
\begin{equation}\label{eq:time-conv}
(K_\delta * \mathbf{R})(x,t) \;=\; \int_{\mathbb{R}} K_\delta(t-s)\,\mathbf{R}(x,s)\,ds,
\end{equation}
with a nonnegative kernel $K_\delta\in L^1(\mathbb{R})$ satisfying $\int_{\mathbb{R}} K_\delta(s)\,ds=1$. 
Assume there exists a set $\mathcal{U}\subset\mathbb{R}^d\times[t_n,t_{n+1}]$ that contains the exact trajectory over one step and on which
\begin{equation}\label{eq:bounds-prop}
\begin{aligned}
&\|\partial_t u\|_{L^\infty(\mathcal{U})},\ \|\partial_x u\|_{L^\infty(\mathcal{U})} \;<\; \infty,\\
&\|\partial_t \mathbf{R}\|_{L^\infty(\mathcal{U})},\ \|\partial_x \mathbf{R}\|_{L^\infty(\mathcal{U})},\ 
\|\mathbf{R}\|_{L^\infty(\mathcal{U})} \;<\; \infty.
\end{aligned}
\end{equation}
Define the computable consistency proxy for $f=u$ by
\begin{equation}\label{eq:C-proxy}
\mathcal{C}(u;\mathcal{U})
:= \|\partial_t u\|_{L^\infty(\mathcal{U})}
 + \|\partial_x u\|_{L^\infty(\mathcal{U})}\,\|u\|_{L^\infty(\mathcal{U})}.
\end{equation}

Then the chord control field does not increase the consistency bound:
\begin{equation}\label{eq:consistency-ineq}
\mathcal{C}(K_\delta*\mathbf{R};\mathcal{U})
\;\le\;
\mathcal{C}(\mathbf{R};\mathcal{U}).
\end{equation}

Consequently, the local truncation error constant $M_f$ from Lemma~\ref{lem:euler-local-truncation} satisfies
\begin{equation}\label{eq:Mf-compare}
M_{f_{\mathrm{cho}}} \le \mathcal{C}(K_\delta*\mathbf{R};\mathcal{U})
\le \mathcal{C}(\mathbf{R};\mathcal{U}).
\end{equation}
and hence the explicit Euler local truncation error admits the bound
\begin{equation}\label{eq:lte-upper}
\|\tau^{\mathrm{cho}}_{n+1}\|
\;\le\; \frac{h^2}{2}\,\mathcal{C}\!\left(K_\delta*\mathbf{R};\,\mathcal{U}\right)
\;\le\; \frac{h^2}{2}\,\mathcal{C}\!\left(\mathbf{R};\,\mathcal{U}\right).
\end{equation}

\end{proposition}

\begin{proof}
Since the convolution acts only in time and $K_\delta$ does not depend on $x$, the spatial and temporal derivatives commute with convolution:
\begin{equation}\label{eq:commute}
\partial_t (K_\delta * \mathbf{R}) \;=\; K_\delta * (\partial_t \mathbf{R}), 
\qquad
\partial_x (K_\delta * \mathbf{R}) \;=\; K_\delta * (\partial_x \mathbf{R}).
\end{equation}
By Young’s $L^1$–$L^\infty$ inequality with $\|K_\delta\|_{L^1}=1$,
\begin{equation}\label{eq:nonexp}
\|K_\delta * Z\|_{L^\infty(\mathcal{U})} \;\le\; \|Z\|_{L^\infty(\mathcal{U})},
\qquad
Z \in \{\mathbf{R},\,\partial_t \mathbf{R},\,\partial_x \mathbf{R}\}.
\end{equation}
Apply \eqref{eq:commute}–\eqref{eq:nonexp} termwise in \eqref{eq:C-proxy} with $u=K_\delta*\mathbf{R}$ to obtain
\begin{equation}\label{eq:proxy-terms}
\begin{aligned}
\|\partial_t (K_\delta*\mathbf{R})\|_{L^\infty(\mathcal{U})}
&\le \|\partial_t \mathbf{R}\|_{L^\infty(\mathcal{U})},\\
\|\partial_x (K_\delta*\mathbf{R})\|_{L^\infty(\mathcal{U})}
&\le \|\partial_x \mathbf{R}\|_{L^\infty(\mathcal{U})},\\
\|K_\delta*\mathbf{R}\|_{L^\infty(\mathcal{U})}
&\le \|\mathbf{R}\|_{L^\infty(\mathcal{U})}.
\end{aligned}
\end{equation}
Substituting \eqref{eq:proxy-terms} into \eqref{eq:C-proxy} yields \eqref{eq:consistency-ineq}. 
For the connection to the local truncation error, recall from Lemma~\ref{lem:euler-local-truncation} that
\begin{equation}\label{eq:Mf-from-lemma}
M_f \;=\; \sup_{(x,t)\in\mathcal{U}} 
\bigl\|\,\partial_t f(x,t)+\partial_x f(x,t)\,f(x,t)\,\bigr\|
\;\le\; \mathcal{C}(u;\mathcal{U}),
\end{equation}
with $f=u$. Taking $u=\mathbf{R}$ and $u=K_\delta*\mathbf{R}$ gives \eqref{eq:Mf-compare}, which in turn implies \eqref{eq:lte-upper} via Lemma~\ref{lem:euler-local-truncation}.
\end{proof}

\begin{theorem}[Global $O(h)$ convergence; chord has smaller constants]\label{thm:global-oh-chord}
Let $t_n=t_0+nh$ and consider explicit Euler
\begin{equation}\label{eq:euler-scheme-short}
x_{n+1}=x_n+h\,u(x_n,t_n).
\end{equation}
where $u\in\{\mathbf{R},\,K_\delta*\mathbf{R}\}$ and $K_\delta$ is the time–only kernel from Proposition~\ref{prop:consistency-chord}. 
Assume there is $\mathcal{U}_T$ containing the exact and numerical trajectories on $[t_0,T]$ such that
\begin{equation}\label{eq:regularity-short}
\begin{aligned}
&\|\partial_x u\|_{L^\infty(\mathcal{U}_T)},\ \|\partial_t u\|_{L^\infty(\mathcal{U}_T)},\\
&\|\partial_x \mathbf{R}\|_{L^\infty(\mathcal{U}_T)},\ \|\partial_t \mathbf{R}\|_{L^\infty(\mathcal{U}_T)},\ \|\mathbf{R}\|_{L^\infty(\mathcal{U}_T)}<\infty.
\end{aligned}
\end{equation}
Define
\begin{equation}\label{eq:L-M-short}
\begin{aligned}
&L_u:=\|\partial_x u\|_{L^\infty(\mathcal{U}_T)},\\
&M_u:=\sup_{(x,t)\in\mathcal{U}_T}\bigl\|\,\partial_t u+\partial_x u\,u\,\bigr\|.
\end{aligned}
\end{equation}
Then for $e^u_n:=\|x^u(t_n)-x^u_n\|$ and all $0\le n\le N$,
\begin{equation}\label{eq:global-oh-short}
e^u_n\;\le\;\frac{h\,M_u}{2L_u}\Bigl(\exp(L_u\,t_n)-1\Bigr),
\end{equation}
with the convention $e^u_n\le \tfrac{h\,t_n}{2}M_u$ when $L_u=0$. Moreover,
\begin{equation}\label{eq:compare-short}
L_{f_{\mathrm{cho}}}\le L_{f_{\mathrm{nai}}},\qquad
M_{f_{\mathrm{cho}}}\le M_{f_{\mathrm{nai}}},
\end{equation}
hence at $t_N=T$ there exist $\,\mathsf{C}_{\mathrm{cho}}\le \mathsf{C}_{\mathrm{nai}}$ independent of $h$ such that
\begin{equation}\label{eq:big-oh-short}
e^{\mathrm{cho}}_N\le \mathsf{C}_{\mathrm{cho}}\,h,\qquad
e^{\mathrm{nai}}_N\le \mathsf{C}_{\mathrm{nai}}\,h.
\end{equation}
\end{theorem}

\begin{proof}
Fix $u$ and write $L=\|\partial_x f\|_{L^\infty}$, $M=M_u$. By Lemma~\ref{lem:euler-local-truncation},
\begin{equation}\label{eq:recurrence-short}
e^u_{n+1}\le (1+hL)\,e^u_n+\frac{h^2}{2}M.
\end{equation}
Discrete Grönwall gives
\begin{equation}\label{eq:gronwall-short}
e^u_n\le \frac{hM}{2}\frac{(1+hL)^n-1}{hL}\le \frac{hM}{2L}\bigl(\exp(Lt_n)-1\bigr),
\end{equation}
which is \eqref{eq:global-oh-short}. For the comparison, time–only convolution commutes with $\partial_t,\partial_x$ and is $L^1$–$L^\infty$ nonexpansive, so
\begin{equation}\label{eq:young-short}
\begin{aligned}
&\|\partial_x (K_\delta*\mathbf{R})\|_\infty\le \|\partial_x \mathbf{R}\|_\infty,\quad
\|K_\delta*\mathbf{R}\|_\infty\le \|\mathbf{R}\|_\infty,\\
&\|\partial_t (K_\delta*\mathbf{R})\|_\infty\le \|\partial_t \mathbf{R}\|_\infty.
\end{aligned}
\end{equation}
Hence $L_{f_{\mathrm{cho}}}\le L_{f_{\mathrm{nai}}}$ and, by Proposition~\ref{prop:consistency-chord}, $M_{f_{\mathrm{cho}}}\le M_{f_{\mathrm{nai}}}$, proving \eqref{eq:compare-short} and \eqref{eq:big-oh-short}.
\end{proof}

\begin{remark}[Conclusion of Theorem~\ref{thm:global-oh-chord}]
The proof establishes the standard $O(h)$ global convergence of the explicit Euler method via the error recurrence (Eq.~\eqref{eq:recurrence-short}), the local error bound (Lemma~\ref{lem:euler-local-truncation}), and the discrete Grönwall inequality. The key insight is that the global error constant $C(f)$ is smaller for the chord control. This is because the chord field's time derivative $\partial_t f_{\rm cho}$ exhibits $L^\infty$ contraction (smoothing) via convolution (Proposition~\ref{prop:consistency-chord}), and its field magnitude $f_{\rm cho}$ does not increase. This directly reduces the consistency constant, guaranteeing a smaller global error bound for the same step size $h$.
\end{remark}

Theorem~\ref{thm:global-oh-chord} established that the global $O(h)$ convergence error is governed by the consistency constant $C(f)$, and that $C(f_{\rm cho}) \le C(f_{\rm nai})$. This implies that for the same step size $h$, the global error of the chord-controlled path is bounded by that of the naive path, $\text{Global Error}^{\rm cho} \le \text{Global Error}^{\rm nai}$.

\begin{corollary}[BIBO boundedness under $f=u$]\label{cor:bibo}
Assume the editing field $u$ has at most linear growth, e.g.,
$\|u(x,t)\|\le \beta \|x\| + b$ for all $(x,t)$ in the domain of interest.
Then the one–step explicit Euler update $x_{n+1}=x_n+h\,u(x_n,t_n)$ satisfies
\begin{equation}
\|x_{n+1}\|\le (1+h\beta)\|x_n\|+h\,b.
\end{equation}

Moreover, if only the trivial bound is desired (without growth assumptions),
\begin{equation}
\|x_{n+1}\|\le \|x_n\|+h\,\|u(x_n,t_n)\|.
\end{equation}
\end{corollary}

\section{Gap to the Optimal Control Field}
\label{sec:gap-to-optimal}

We now analyze the gap between our estimators and the "true" optimal control $u^\star$. We interpret $u^\star$ as the solution that minimizes the Benamou–Brenier energy under the controlled continuity equation:
\begin{equation}
\partial_t\rho_t + \nabla \cdot \big(\rho_t\, u\big) = 0.
\end{equation}
The following theorem frames $\mathbf{R}(t)$ as a noisy observation of $u^\star(t)$ and shows that the Chord estimator $\hat{u}(t)$ acts as a risk-reducing smoother.

\begin{theorem}[Risk Reduction via Kernel Smoothing]
\label{thm:risk-reduction}
Let the true optimal control $u^\star: \mathbb{R} \to \mathbb{R}^d$ be $C^2$ (twice continuously differentiable). We observe the proxy field
\begin{equation}
\mathbf{R}(t) = u^\star(t) + \eta(t),
\end{equation}
where the noise process $\eta(t)$ satisfies:
\begin{itemize}
 \item[(N1)] Zero-mean: $\mathbb{E}[\eta(t)] = 0$ for all $t$.
 \item[(N2)] Uncorrelated: $\mathbb{E}[\eta(s)\eta(r)^\top] = 0$ for almost every $s \neq r$.
 \item[(N3)] Bounded variance: $\mathbb{E}[\|\eta(t)\|^2] = \sigma^2(t) \le \bar\sigma^2 < \infty$.
\end{itemize}
Let $K$ be a \textbf{non-negative, unit-mass, second-order kernel}, i.e.,
$\int K(s)ds = 1$, $\int s K(s) ds = 0$.
Define the kernel family $K_\delta(s) = \frac{1}{\delta} K(\frac{s}{\delta})$ for a bandwidth $\delta > 0$, and the chord estimator
\begin{equation}
\hat{u}(t) = (K_\delta * \mathbf{R})(t) = \int K_\delta(s) \mathbf{R}(t-s) \,ds.
\end{equation}
Then, the Mean Squared Error (Risk) of the chord estimator at time $t$ is bounded by:
\begin{equation}
\boxed{ \mathbb{E}[\|\hat{u}(t) - u^\star(t)\|^2] \le \underbrace{c_1 \delta^4 \|u^{\star\prime\prime}\|_\infty^2}_{\text{Bias}^2} + \underbrace{\|K_\delta\|_{L^2}^2 \sigma^2(t)}_{\text{Variance}} },
\end{equation}
where $c_1 = \frac{1}{4} m_2(K)^2$ and $m_2(K) = \int s^2 K(s) ds$.
In contrast, the risk of the naive estimator $\mathbf{R}(t)$ is $\mathbb{E}[\|\mathbf{R}(t) - u^\star(t)\|^2] = \sigma^2(t)$. For a non-degenerate kernel, choosing an appropriate $\delta$ (see Remark \ref{rem:bandwidth}) ensures the chord estimator achieves a strictly lower risk.
\end{theorem}

\begin{remark}[Causal (One-Sided) Kernels]
The assumption $\int s K(s) ds = 0$ (a second-order kernel) requires $K$ to be symmetric, which is non-causal. If we enforce a \textbf{causal, non-negative} kernel (e.g., $\operatorname{supp}(K) \subset [0, 1]$ as in our derivation), then $m_1(K) = \int s K(s) ds > 0$. The Taylor expansion (Step 2 in the proof) will be dominated by the first-order term, yielding a bias of $O(\delta)$ and a squared bias of $O(\delta^2)$. The risk bound becomes $O(\delta^2) + O(\delta^{-1})\sigma^2(t)$, but the conclusion of risk reduction still holds.
\end{remark}

\begin{proof}
We decompose the Mean Squared Error (Risk) into its squared bias and variance components.
\begin{equation}
\label{eq:proof-bias-var-decomp}
\begin{aligned}
\mathbb{E}[\|\hat{u}(t) - u^\star(t)\|^2] &= \underbrace{\|\mathbb{E}[\hat{u}(t)] - u^\star(t)\|^2}_{\text{Bias}^2} \\ &+ \underbrace{\mathbb{E}[\|\hat{u}(t) - \mathbb{E}[\hat{u}(t)]\|^2]}_{\text{Variance}}.
\end{aligned}
 \tag{A}
\end{equation}

\noindent\textbf{1. Bound on the Bias Term.}
By the linearity of expectation and convolution, and using (N1) ($\mathbb{E}[\eta(t)] = 0$), the expected value of the estimator is:
\begin{equation}
\begin{aligned}
\mathbb{E}[\hat{u}(t)]& = \mathbb{E}[(K_\delta * (u^\star + \eta))(t)] \\ 
& = (K_\delta * u^\star)(t) + (K_\delta * \mathbb{E}[\eta])(t) \\
& = (K_\delta * u^\star)(t).
\end{aligned}
\end{equation}
The bias is the difference between this expectation and the true value:
\begin{equation}
\mathrm{Bias}(t) = (K_\delta * u^\star)(t) - u^\star(t) = \int K_\delta(s) \big( u^\star(t-s) - u^\star(t) \big) \,ds.
\end{equation}
We apply a second-order Taylor expansion to $u^\star(t-s)$ around $s=0$:
\begin{equation}
u^\star(t-s) = u^\star(t) - s u^{\star\prime}(t) + \frac{s^2}{2} u^{\star\prime\prime}(t - \theta_s s), \quad \theta_s \in (0,1).
\end{equation}
Substituting this into the bias integral:
\begin{align*}
\mathrm{Bias}(t) &= \int K_\delta(s) \left( -s u^{\star\prime}(t) + \frac{s^2}{2} u^{\star\prime\prime}(\dots) \right) \,ds \\
&= -u^{\star\prime}(t) \underbrace{\int s K_\delta(s) \,ds}_{=0 \text{ (2nd-order)}} + \frac{1}{2} \int s^2 K_\delta(s) u^{\star\prime\prime}(\dots) \,ds.
\end{align*}
The first term vanishes due to the second-order kernel assumption. We bound the remainder:
\begin{align*}
\|\mathrm{Bias}(t)\| &\le \frac{1}{2} \int s^2 K_\delta(s) \|u^{\star\prime\prime}(t - \theta_s s)\| \,ds \\
&\le \frac{1}{2} \|u^{\star\prime\prime}\|_\infty \int s^2 K_\delta(s) \,ds.
\end{align*}
Since $\int s^2 K_\delta(s) ds = \int s^2 \frac{1}{\delta} K(\frac{s}{\delta}) ds = \delta^2 \int u^2 K(u) du = \delta^2 m_2(K)$,
\begin{equation*}
\|\mathrm{Bias}(t)\| \le \frac{1}{2} m_2(K) \delta^2 \|u^{\star\prime\prime}\|_\infty.
\end{equation*}
Squaring this gives the bound on the first term of (A):
\begin{equation}
\label{eq:proof-bias-bound}
\|\mathrm{Bias}(t)\|^2 \le \frac{1}{4} m_2(K)^2 \delta^4 \|u^{\star\prime\prime}\|_\infty^2. \tag{B}
\end{equation}

\noindent\textbf{2. Bound on the Variance Term.}
The variance term is the expected norm of the centered estimator $\zeta(t)$:
\begin{equation}
\zeta(t) = \hat{u}(t) - \mathbb{E}[\hat{u}(t)] = (K_\delta * \eta)(t) = \int K_\delta(s) \eta(t-s) \,ds.
\end{equation}
We write the squared norm as an inner product and apply Fubini's theorem:
\begin{align*}
\mathrm{Var}(t) &= \mathbb{E}[\langle \zeta(t), \zeta(t) \rangle] \\
&= \mathbb{E}\left[ \left\langle \int K_\delta(s) \eta(t-s) ds, \int K_\delta(r) \eta(t-r) dr \right\rangle \right] \\
&= \iint K_\delta(s) K_\delta(r) \mathbb{E}[\langle \eta(t-s), \eta(t-r) \rangle] \,ds \,dr.
\end{align*}
By assumption (N2), the noise is uncorrelated, so the cross-terms where $s \neq r$ (or $t-s \neq t-r$) have zero expectation (a.e.). The integral collapses to the diagonal $s=r$:
\begin{align*}
\mathrm{Var}(t) &= \int K_\delta(s)^2 \mathbb{E}[\|\eta(t-s)\|^2] \,ds \\
&= \int K_\delta(s)^2 \sigma^2(t-s) \,ds.
\end{align*}
Bounding the local variance by the value at $t$ (or by $\sup \sigma^2$):
\begin{equation}
\mathrm{Var}(t) \le \sigma^2(t) \int K_\delta(s)^2 \,ds = \sigma^2(t) \|K_\delta\|_{L^2}^2. \tag{C}
\end{equation}

\noindent\textbf{3. Combining the Bounds.}
Substituting (B) and (C) into (A) yields the theorem's risk bound. In contrast, the risk of the naive estimator (which corresponds to $K = \delta_0$, a Dirac delta) is $\mathbb{E}[\|\mathbf{R}(t) - u^\star(t)\|^2] = \mathbb{E}[\|\eta(t)\|^2] = \sigma^2(t)$.
Since for any non-degenerate kernel $\|K_\delta\|_{L^2}^2 < \infty$ (and specifically $\|K_\delta\|_{L^2}^2 \propto \delta^{-1} < \infty$ for $\delta > 0$), the chord estimator achieves variance reduction. By choosing $\delta$ appropriately to balance the $O(\delta^4)$ bias and the $O(\delta^{-1})\sigma^2$ variance, the total risk is strictly reduced.
\end{proof}

\begin{remark}[Optimal Bandwidth]
\label{rem:bandwidth}
The scaling of the $L^2$ norm is $\|K_\delta\|_{L^2}^2 = \int \frac{1}{\delta^2} K(\frac{s}{\delta})^2 ds = \frac{\|K\|_{L^2}^2}{\delta}$. The risk bound scales as:
\begin{equation}
\mathrm{Risk}(\delta) \lesssim c_1 \delta^4 \|u^{\star\prime\prime}\|_\infty^2 + \frac{\|K\|_{L^2}^2}{\delta} \sigma^2(t).
\end{equation}
Minimizing this with respect to $\delta$ gives the classic optimal bandwidth for second-order kernel smoothing, $\delta^\star \asymp \left( \frac{\|K\|_{L^2}^2 \sigma^2(t)}{\|u^{\star\prime\prime}\|_\infty^2} \right)^{1/5}$.
\end{remark}

\begin{theorem}[Gap to Benamou–Brenier Optimal Energy]
\label{thm:energy-gap}
Let $u^\star$ be the true, energy-minimizing optimal control in the space of feasible controls $\mathcal{U}$. Let $\mathcal{U}_\delta \subset \mathcal{U}$ be the subspace of controls that are piecewise linear in time (i.e., chords) on a grid of size $\delta$.
Let $P_\delta: \mathcal{U} \to \mathcal{U}_\delta$ be the $L^2_\rho$-orthogonal projection onto this subspace, where the $L^2_\rho$ norm is induced by the Benamou–Brenier energy functional $\mathcal{E}[u;\rho]$.
If we identify the idealized chord estimator $\hat{u}$ with this projection, $\hat{u} = P_\delta u^\star$, the energy gap is bounded by:
\begin{equation}
\mathcal{E}[\hat{u};\rho] - \mathcal{E}[u^\star;\rho] \le \|(I-P_\delta)u^\star\|_\rho \cdot \|u^\star\|_\rho.
\end{equation}
Furthermore, if $u^\star \in H^1$ (i.e., $\partial_t u^\star$ is in $L^2_\rho$), the approximation error of the projection is $O(\delta)$, leading to a final bound:
\begin{equation}
\mathcal{E}[\hat{u};\rho] - \mathcal{E}[u^\star;\rho] \le C \delta \|\partial_t u^\star\|_\rho \|u^\star\|_\rho = O(\delta).
\end{equation}
\end{theorem}

\begin{proof}
The proof proceeds in three steps: defining the weighted Hilbert space, applying a projection identity, and bounding the projection error using approximation theory.

\noindent\textbf{1. Weighted Hilbert Space and Energy.}
We define $H := L^2_\rho([0,1] \times \Omega; \mathbb{R}^d)$ as the Hilbert space of vector fields weighted by the density $\rho_t(x)$. The inner product is:
\begin{equation}
\langle a, b \rangle_\rho := \int_0^1 \int_\Omega a(x,t) \cdot b(x,t) \rho_t(x) \,dx \,dt,
\end{equation}
with the induced norm $\|a\|_\rho^2 = \langle a, a \rangle_\rho$. The Benamou–Brenier kinetic energy is $\mathcal{E}[u;\rho] = \frac{1}{2} \|u\|_\rho^2$.
We define $P_\delta: H \to \mathcal{U}_\delta$ as the orthogonal projection onto the subspace of piecewise linear functions (chords) with respect to this inner product. We analyze the idealized, noiseless estimator $\hat{u} = P_\delta u^\star$.

\noindent\textbf{2. Projection Identity and Upper Bound.}
For any $u \in H$, we use the algebraic identity $\langle (P_\delta - I)u, (P_\delta + I)u \rangle_\rho = \langle P_\delta u, P_\delta u \rangle_\rho - \langle u, u \rangle_\rho$, which gives:
\begin{equation}
\|P_\delta u\|_\rho^2 - \|u\|_\rho^2 = \langle (P_\delta - I)u, (P_\delta + I)u \rangle_\rho.
\end{equation}
Applying this to $u = u^\star$ and dividing by 2:
\begin{align}
\mathcal{E}[\hat{u};\rho] - \mathcal{E}[u^\star;\rho] &= \frac{1}{2} \langle (P_\delta - I)u^\star, (P_\delta + I)u^\star \rangle_\rho \\
&\le \frac{1}{2} \|(I - P_\delta)u^\star\|_\rho \cdot \|(I + P_\delta)u^\star\|_\rho
\end{align}
by the Cauchy–Schwarz inequality. Since $P_\delta$ is an orthogonal projection, its operator norm is $\|P_\delta\| \le 1$. Thus, by the triangle inequality:
\begin{equation}
\|(I + P_\delta)u^\star\|_\rho \le \|I u^\star\|_\rho + \|P_\delta u^\star\|_\rho \le \|u^\star\|_\rho + \|u^\star\|_\rho = 2 \|u^\star\|_\rho.
\end{equation}
Substituting this back, we absorb the constant $\frac{1}{2} \cdot 2 = 1$ into the inequality:
\begin{equation}
\label{eq:proof-energy-gap-bound}
\mathcal{E}[\hat{u};\rho] - \mathcal{E}[u^\star;\rho] \le \|(I - P_\delta)u^\star\|_\rho \cdot \|u^\star\|_\rho.
\end{equation}

\noindent\textbf{3. Approximation Error of Chord Space ($O(\delta)$).}
The term $\|(I - P_\delta)u^\star\|_\rho$ is the minimal $L^2_\rho$-error when approximating $u^\star$ from the subspace $\mathcal{U}_\delta$. This is a standard result in approximation theory (a Jackson-type inequality). For a function $u^\star$ in the Sobolev space $H^1$ (meaning its first derivative $\partial_t u^\star$ is in $L^2_\rho$), the error of the best piecewise linear approximation on a grid of size $\delta$ is bounded by the first derivative:
\begin{equation}
\label{eq:proof-jackson}
\|(I - P_\delta)u^\star\|_\rho \le C_{\text{app}} \delta \|\partial_t u^\star\|_\rho.
\end{equation}
This arises from applying the Poincaré–Wirtinger inequality on each sub-interval $[t_k, t_{k+1}]$ and summing the errors. The constant $C_{\text{app}}$ depends on the regularity of the grid but not on $\delta$.

Substituting Eq.~\eqref{eq:proof-jackson} into Eq.~\eqref{eq:proof-energy-gap-bound} gives the final $O(\delta)$ bound:
\begin{equation}
\mathcal{E}[\hat{u};\rho] - \mathcal{E}[u^\star;\rho] \le \big( C_{\text{app}} \delta \|\partial_t u^\star\|_\rho \big) \cdot \|u^\star\|_\rho = O(\delta).
\end{equation}
\end{proof}

\begin{remark}[Tighter Bound for True Orthogonal Projections]
If $\hat{u}$ is \emph{exactly} the $L^2_\rho$-orthogonal projection $P_\delta u^\star$, the Pythagorean theorem provides a much tighter (and intuitive) result. Since $u^\star = P_\delta u^\star + (I - P_\delta)u^\star$ is an orthogonal decomposition:
\begin{equation}
\|u^\star\|_\rho^2 = \|P_\delta u^\star\|_\rho^2 + \|(I - P_\delta)u^\star\|_\rho^2.
\end{equation}
Therefore, the energy gap is:
\begin{equation}
\begin{aligned}
\mathcal{E}[\hat{u};\rho] - \mathcal{E}[u^\star;\rho] 
& = \frac{1}{2} \big( \|P_\delta u^\star\|_\rho^2 - \|u^\star\|_\rho^2 \big) \\
& = -\frac{1}{2} \|(I - P_\delta)u^\star\|_\rho^2 \\
& \le 0.
\end{aligned}
\end{equation}
This confirms that the projection onto the chord subspace \emph{never} increases the energy. The $O(\delta)$ bound derived in the main proof is a looser upper bound, but has the advantage of also holding for quasi-projections or causal smoothing operators (like our implemented kernel) whose operator norms are bounded by 1.
\end{remark}


\begin{theorem}[One-Step Error and Stability Condition (Euler)]
\label{thm:one-step-error-stability}
Assume (A2) ($u \in W^{1,\infty}$), and that $\mathbf R$, $K_\delta*\mathbf R$ 
have bounded first derivatives in $(x,t)$ on the relevant domain, with $K_\delta\ge 0$ and $\int K_\delta=1$.
The one-step local truncation errors (from Lemma \ref{lem:euler-local-truncation}) for the naive and chord controls are bounded by:
\begin{align}
\|\tau_{n+1}^{\rm nai}\| &\le \tfrac{1}{2} h^2 \, C_{\rm nai} \\
\|\tau_{n+1}^{\rm cho}\| &\le \tfrac{1}{2} h^2 \, C_{\rm cho}
\end{align}
where $C_{\rm nai}$ and $C_{\rm cho}$ are the global consistency constants:
\begin{equation}
\boxed{
\begin{aligned}
C_{\rm nai} &:= \sup_{t\in[0,T]}\Big(\|\partial_t u\|
     + \|\partial_x u\|\,\|\mathbf R\|\Big), \\
C_{\rm cho} &:= \sup_{t\in[0,T]}\Big(\|\partial_t u\|
     + \|\partial_x u\|\,\|\hat u\|\Big).
\end{aligned}}
\end{equation}
Due to the $L^\infty$ contraction properties of the kernel $K_\delta$, $C_{\rm cho} \le C_{\rm nai}$.
Furthermore, the stability condition for the explicit Euler method (e.g., $hL < 1$) depends only on $L = \sup \|\partial_x u\|$ and is identical for both control fields.
\end{theorem}

\begin{proof}
The proof combines the results from Lemma \ref{lem:euler-local-truncation}, Proposition \ref{prop:consistency-chord}, and Theorem \ref{thm:global-oh-chord}.

\noindent\textbf{1. Specialization of Local Error (from Lemma \ref{lem:euler-local-truncation}).}
Lemma \ref{lem:euler-local-truncation} provides a general one-step error bound:
\begin{equation*}
\begin{aligned}
\|\tau_{n+1}\| & \le \tfrac{1}{2} h^2 \sup_{s\in[t_n,t_{n+1}]}\Big(\|\partial_t f(x(s),s)\| + \Big. \\
& \Big. \|\partial_x f(x(s),s)\| \cdot \|f(x(s),s)\|\Big).
\end{aligned}
\end{equation*}
For $f=u$, we have $\partial_x f = \partial_x u$ and $\partial_t f = \partial_t u$.
Substituting these for $u_{\rm nai} = \mathbf{R}$ and $u_{\rm cho} = \hat{u} = K_\delta * \mathbf{R}$ respectively gives:

\begin{equation}
\resizebox{\linewidth}{!}{$
\begin{aligned}
\|\tau_{n+1}^{\rm nai}\| &\le \tfrac{1}{2} h^2 \sup_{s\in[t_n,t_{n+1}]}\Big(\|\partial_t u+\dot{\mathbf R}\| + \|\partial_x u\| \cdot \|\mathbf R\|\Big) \\
\|\tau_{n+1}^{\rm cho}\| &\le \tfrac{1}{2} h^2 \sup_{s\in[t_n,t_{n+1}]}\Big(\|\partial_t u+K_\delta * \dot{\mathbf R}\| + \|\partial_x u\| \cdot \|\hat u\|\Big).
\end{aligned}
$}
\end{equation}

Taking the supremum over the full interval $t \in [0,T]$ (instead of just $[t_n, t_{n+1}]$) defines the global constants $C_{\rm nai}$ and $C_{\rm cho}$ as stated in the theorem.

\noindent\textbf{2. Proof of $C_{\rm cho} \le C_{\rm nai}$ (from Prop. \ref{prop:consistency-chord}).}
The constant $C(f)$ is the sum of a time-derivative term and a field-magnitude term.
\emph{(i) Time-derivative term:} As shown in Prop. \ref{prop:consistency-chord}, $L^\infty$ contraction by the non-negative, unit-mass kernel $K_\delta$ ensures:
\begin{equation*}
\begin{aligned}
\sup_t \|\partial_t u + K_\delta * \dot{\mathbf R}\|
& \le \sup_t \|\partial_t u\| + \sup_t \|K_\delta * \dot{\mathbf R}\| \\
& \le \sup_t \|\partial_t u\| + \sup_t \|\dot{\mathbf R}\|.
\end{aligned}
\end{equation*}
This is bounded by $\sup_t \|\partial_t u + \dot{\mathbf R}\|$.

\emph{(ii) Field-magnitude term:} Since $\|\hat{u}\|_\infty \le \|\mathbf{R}\|_\infty$ (Prop. \ref{prop:consistency-chord}), this term is also bounded by $\sup_t \|\mathbf{R}\|$, which provides a conservative bound for the naive term $\sup_t \|\mathbf{R}\|$.

Since both components of $C_{\rm cho}$ are less than or equal to their counterparts in $C_{\rm nai}$, we have $C_{\rm cho} \le C_{\rm nai}$.

\begin{proposition}[Explicit Euler stability is unaffected by the control design]\label{prop:euler-stability}
Consider one explicit Euler step at time $t_n$ for
\begin{equation}\label{eq:ode-stab}
\dot x(t)=u(x(t),t),
\end{equation}
with $u\in\{\mathbf{R},\,K_\delta*\mathbf{R}\}$ and the time–only convolution
\begin{equation}\label{eq:conv-stab}
(K_\delta * \mathbf{R})(x,t)=\int_{\mathbb{R}}K_\delta(t-s)\,\mathbf{R}(x,s)\,ds,\qquad K_\delta\ge 0,\ \int_{\mathbb{R}}K_\delta=1.
\end{equation}
Let $\mathcal{U}\subset\mathbb{R}^d\times[t_n,t_{n+1}]$ contain the exact one–step trajectory and assume
\begin{equation}\label{eq:bounds-stab}
\|\partial_x u\|_{L^\infty(\mathcal{U})},\ \|\partial_x \mathbf{R}\|_{L^\infty(\mathcal{U})}<\infty .
\end{equation}
Define the step–$n$ Jacobian bound
\begin{equation}\label{eq:L-bounds}
\begin{aligned}
L_{\mathrm{nai}}&:=\|\partial_x u\|_{L^\infty(\mathcal{U})}+\|\partial_x \mathbf{R}\|_{L^\infty(\mathcal{U})},\\
L_{\mathrm{cho}}&:=\|\partial_x u\|_{L^\infty(\mathcal{U})}+\|\partial_x (K_\delta*\mathbf{R})\|_{L^\infty(\mathcal{U})}.
\end{aligned}
\end{equation}
Then, by time–only convolution and Young’s $L^1$–$L^\infty$ inequality,
\begin{equation}\label{eq:L-order}
L_{\mathrm{cho}}\ \le\ L_{\mathrm{nai}}.
\end{equation}
Consequently, any step–size prescription of the form
\begin{equation}\label{eq:h-rule}
h\ \le\ \phi\!\left(\|\partial_x u\|_{L^\infty(\mathcal{U})}\right),
\end{equation}
with $\phi:\mathbb{R}_+\!\to\!\mathbb{R}_+$ nonincreasing (e.g. $\phi(L)=\eta/L$ for a chosen $\eta>0$), admits a weakest–case bound that is not tightened by using the chord control:
\begin{equation}\label{eq:h-compare}
h_{\max}^{\mathrm{cho}}\ \ge\ h_{\max}^{\mathrm{nai}}.
\end{equation}
In particular, the linearized one–step growth factors satisfy
\begin{equation}\label{eq:jac-growth}
\begin{aligned}
\sup_{(x,t)\in\mathcal{U}}\bigl\|I+h\,\partial_x f^{\mathrm{cho}}(x,t)\bigr\|
&\le 1+h\,L_{\mathrm{cho}}
\ \le\ 1+h\,L_{\mathrm{nai}}\\
&\ge \sup_{(x,t)\in\mathcal{U}}\bigl\|I+h\,\partial_x f^{\mathrm{nai}}(x,t)\bigr\|,
\end{aligned}
\end{equation}
so any stability target expressed as $1+h\,L\le 1+\eta$ (or equivalently $h\le \eta/L$) is never harder to meet under the chord control field.
\end{proposition}

\begin{proof}
Since $K_\delta$ acts only in time and is independent of $x$, differentiation commutes with convolution:
\begin{equation}\label{eq:commute-stab}
\partial_x (K_\delta*\mathbf{R})\;=\;K_\delta*(\partial_x \mathbf{R}).
\end{equation}
With $\|K_\delta\|_{L^1}=1$, Young’s inequality gives
\begin{equation}\label{eq:young-stab}
\|\partial_x (K_\delta*\mathbf{R})\|_{L^\infty(\mathcal{U})}\ \le\ \|\partial_x \mathbf{R}\|_{L^\infty(\mathcal{U})}.
\end{equation}
Adding $\|\partial_x u\|_{L^\infty(\mathcal{U})}$ to both sides yields \eqref{eq:L-order}. 

For any $u$, the Jacobian of the explicit Euler one–step map $x\mapsto x+h u(x,t_n)$ is $I+h\,\partial_x u(x,t_n)$, hence
\begin{equation}\label{eq:jac-norm}
\sup_{(x,t)\in\mathcal{U}}\bigl\|I+h\,\partial_x u(x,t)\bigr\|
\ \le\ 1+h\,\|\partial_x u\|_{L^\infty(\mathcal{U})}.
\end{equation}
Therefore any nonincreasing step–size rule \eqref{eq:h-rule} that enforces a desired upper bound on \eqref{eq:jac-norm} becomes no stricter when $\|\partial_x u\|_{L^\infty}$ is replaced by the smaller value $L_{\mathrm{cho}}$. This proves \eqref{eq:h-compare} and \eqref{eq:jac-growth}.
\end{proof}

\noindent\textbf{3. Stability and Global Error (from Prop. \ref{prop:euler-stability} and Thm. \ref{thm:global-oh-chord}).}
As shown in Prop. \ref{prop:euler-stability}, the stability of the Euler method depends on the Jacobian $\partial_x f = \partial_x u$, which is identical for both fields. The stability condition $hL < 1$ (where $L = \sup \|\partial_x u\|$) is therefore unchanged.
As proven in Thm. \ref{thm:global-oh-chord}, the global error is bounded by
\begin{equation*}
\max_{n} \|e_n\| \le \frac{e^{LT} - 1}{L} \cdot \frac{h}{2} \cdot C(f).
\end{equation*}
Since $C_{\rm cho} \le C_{\rm nai}$, it follows directly that for the same step size $h$, the global error bound for the chord control is also smaller or equal.
\end{proof}

\begin{remark}[Condition for Strict Inequality]
The inequality $C_{\rm cho} \le C_{\rm nai}$ becomes strict, $C_{\rm cho} < C_{\rm nai}$, if the kernel $K_\delta$ is non-degenerate (not a Dirac delta, $\delta > 0$) and the proxy field's derivative $\dot{\mathbf{R}}$ is not almost-everywhere constant. In this (typical) case, the $L^\infty$ smoothing of the time-derivative term is strict ($\|K_\delta * \dot{\mathbf{R}}\|_\infty < \|\dot{\mathbf{R}}\|_\infty$), leading to a strictly smaller error constant and a tighter global error bound.
\end{remark}

\begin{corollary}[Explicit Error Ratio]
\label{cor:error-ratio}
Under the assumptions of Theorem \ref{thm:one-step-error-stability}, the ratio of the global error bounds is
\begin{equation}
\begin{aligned}
&\frac{\text{Global Error}^{\rm cho}}{\text{Global Error}^{\rm nai}}
\; \\
&\le\; \frac{ \sup_t\!\big(\|\partial_t u\| + \|\partial_x u\|\,\|\hat u\|\big) }{ \sup_t\!\big(\|\partial_t u\| + \|\partial_x u\|\,\|\mathbf R\|\big) }
\;\le\; 1.
\end{aligned}
\end{equation}
Equality holds only in the degenerate case where the smoothing has no effect (e.g., $\delta=0$ or $\dot{\mathbf R} \equiv 0$).
\end{corollary}

\begin{figure*}[t]
 \centering
 \begin{minipage}[t]{0.31\linewidth}
 \centering
 \includegraphics[width=\linewidth]{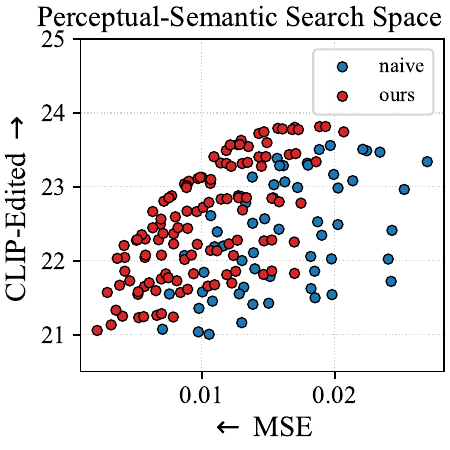} 
 \end{minipage}
 \begin{minipage}[t]{0.31\linewidth}
 \centering
 \includegraphics[width=\linewidth]{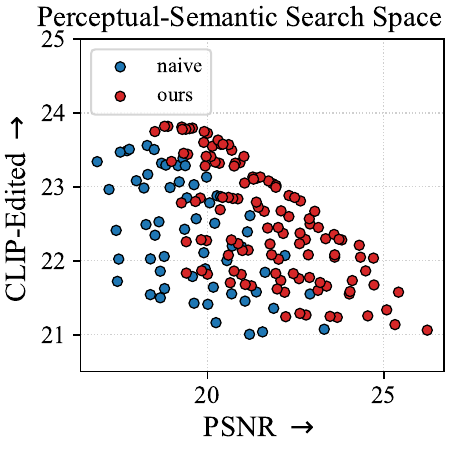}
 \end{minipage}
 \begin{minipage}[t]{0.31\linewidth}
 \centering
 \includegraphics[width=\linewidth]{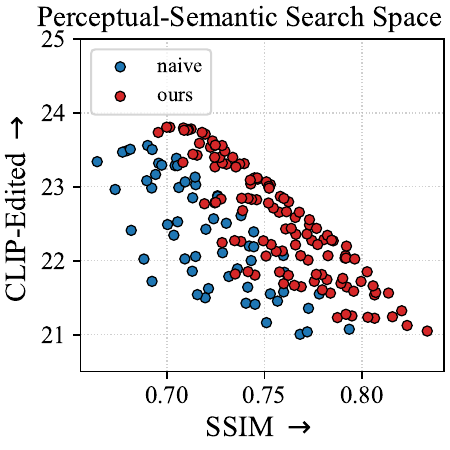}
 \end{minipage}
\caption{\textbf{Raw hyperparameter sweep distribution for CCF analysis.} We compare the naive baseline ($\delta=0$, blue) against ChordEdit ($\delta \ne 0$, red). The plots visualize the trade-off between semantic alignment (CLIP-Edited score) and three standard background preservation metrics: (left) Mean Squared Error (MSE), (center) Peak Signal-to-Noise Ratio (PSNR), and (right) Structural Similarity Index (SSIM). In all three trade-off spaces, the ChordEdit samples (red) occupy a visibly superior region (e.g., lower MSE, higher PSNR/SSIM for a given CLIP score) than the naive baseline, which exhibits a wider, less stable, and inferior performance distribution.}
 \label{fig:scatter}
\end{figure*}

\begin{figure*}[t]
 \centering
 \begin{minipage}[t]{0.31\linewidth}
 \centering
 \includegraphics[width=\linewidth]{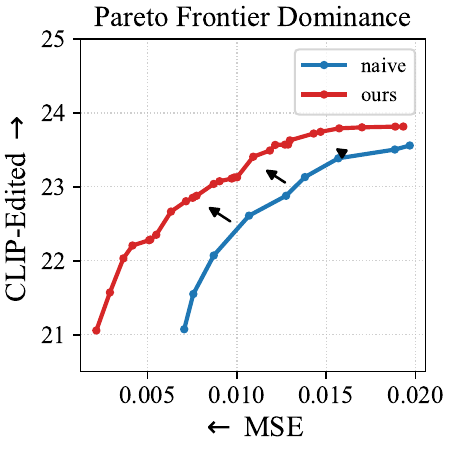} 
 \end{minipage}
 \begin{minipage}[t]{0.31\linewidth}
 \centering
 \includegraphics[width=\linewidth]{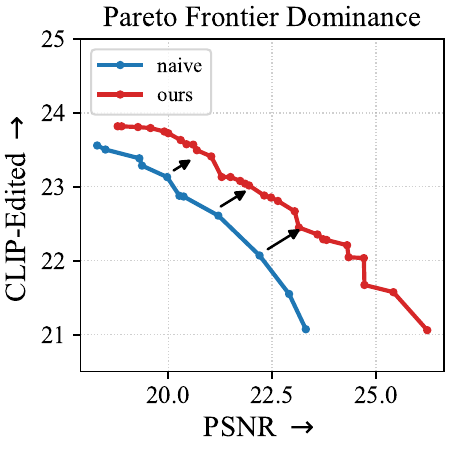}
 \end{minipage}
 \begin{minipage}[t]{0.31\linewidth}
 \centering
 \includegraphics[width=\linewidth]{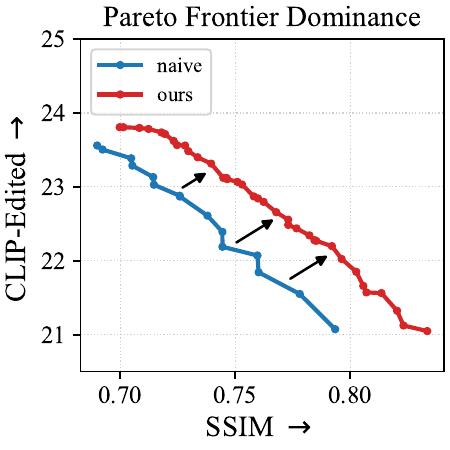}
 \end{minipage}
\caption{\textbf{Pareto dominance of the Chord Control Field.} These frontiers, derived from the data in Figure~\ref{fig:scatter}, confirm the strict Pareto dominance of ChordEdit ($\delta \ne 0$, red) over the naive baseline ($\delta=0$, blue) across all three preservation metrics: (left) MSE vs. CLIP, (center) PSNR vs. CLIP, and (right) SSIM vs. CLIP. In every case, the ChordEdit frontier achieves superior semantic alignment (higher CLIP) for any given level of perceptual fidelity, and vice-versa. This empirically validates that the temporal smoothing induced by $\delta > 0$ is key to resolving the inferior trade-off inherent in the unstable naive approach.}
 \label{fig:pateto}
\end{figure*}

\begin{figure*}[t]
 \centering
 \begin{minipage}[t]{0.31\linewidth}
 \centering
 \includegraphics[width=\linewidth]{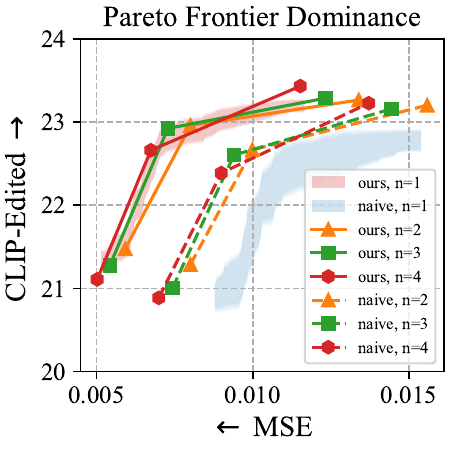} 
 \end{minipage}
 \begin{minipage}[t]{0.31\linewidth}
 \centering
 \includegraphics[width=\linewidth]{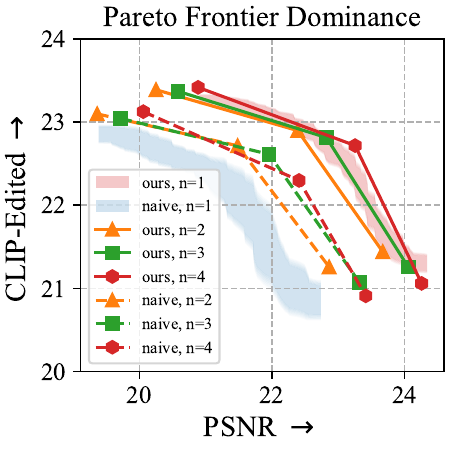}
 \end{minipage}
 \begin{minipage}[t]{0.31\linewidth}
 \centering
 \includegraphics[width=\linewidth]{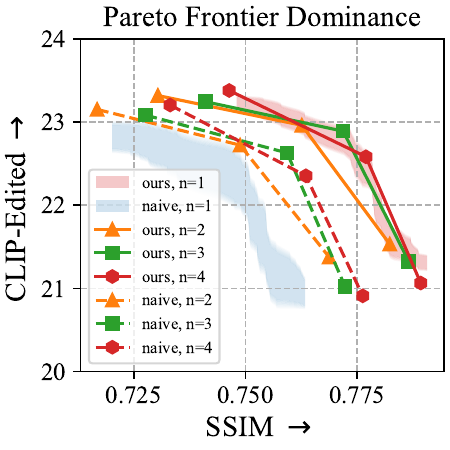}
 \end{minipage}
\caption{\textbf{Noise sample analysis across multiple metrics.} This figure expands the analysis in the main paper by plotting the semantic-preservation trade-off (CLIP-Edited vs. MSE/PSNR/SSIM) as a function of the number of noise samples ($n$). The overlapping distributions and tight confidence bands confirm that increasing $n > 1$ provides negligible marginal returns. ChordEdit's performance with $n=1$ is already highly stable and robust, validating our default setting.}
 \label{fig:noise_mse}
\end{figure*}

\begin{figure*}[t]
 \centering
 \begin{minipage}[t]{0.24\linewidth}
 \centering
 \includegraphics[width=\linewidth]{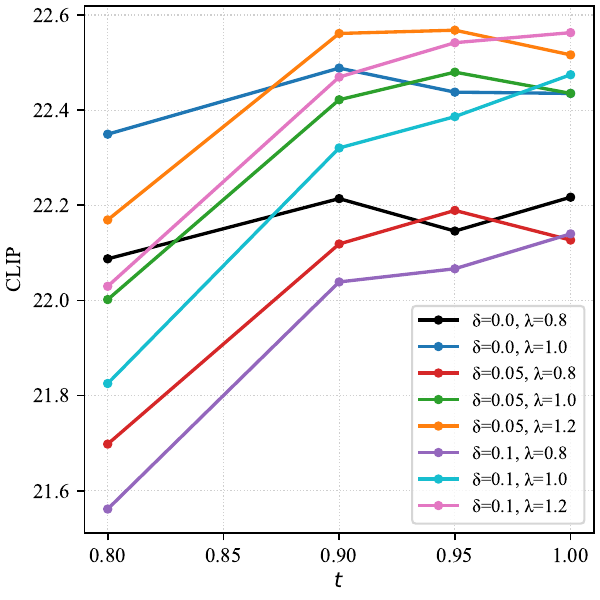}
 \end{minipage}
 \begin{minipage}[t]{0.24\linewidth}
 \centering
 \includegraphics[width=\linewidth]{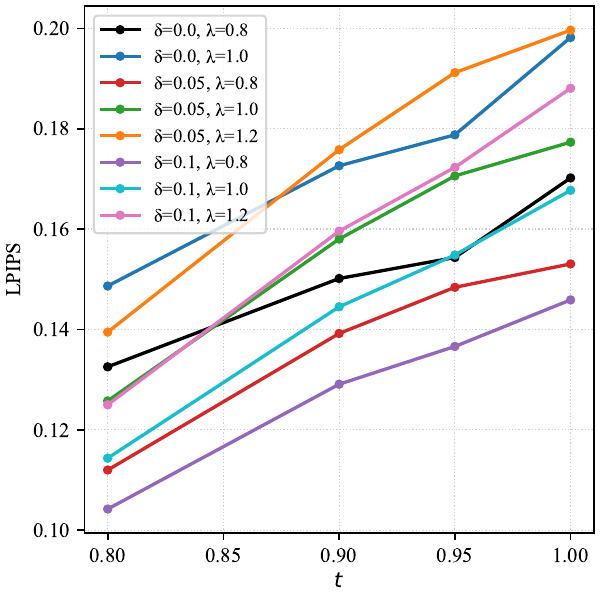}
 \end{minipage}
 \begin{minipage}[t]{0.24\linewidth}
 \centering
 \includegraphics[width=\linewidth]{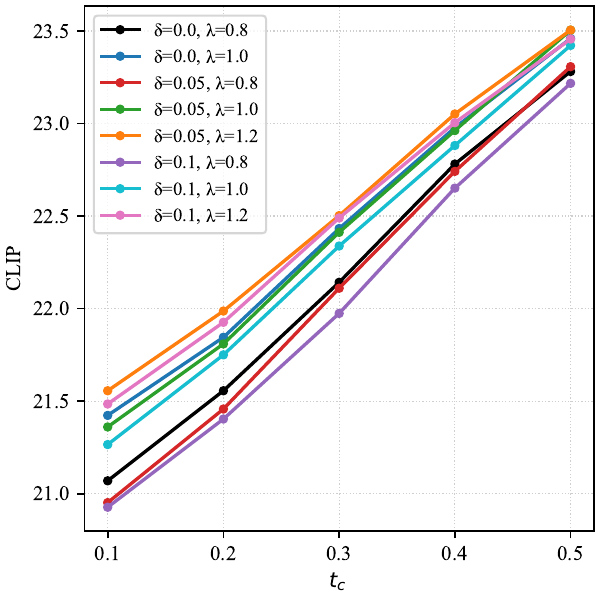} 
 \end{minipage}
 \begin{minipage}[t]{0.24\linewidth}
 \centering
 \includegraphics[width=\linewidth]{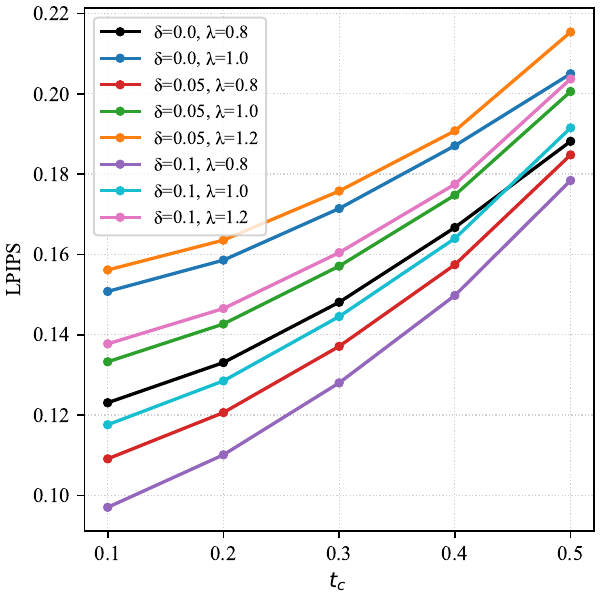}
 \end{minipage}
\caption{\textbf{Ablation study on temporal hyperparameters and step scale.} We analyze the sensitivity of ChordEdit to key parameters. \textbf{(Left two panels):} Impact of the main chord time $t$. Increasing $t$ generally improves semantic alignment (CLIP-Edited $\uparrow$) at the cost of background preservation (LPIPS-Unedit $\downarrow$). \textbf{(Right two panels):} Impact of the proximal refinement time $t_c$. Increasing $t_c$ robustly improves CLIP score, but also monotonically increases LPIPS distortion. This analysis confirms a clear trade-off space, allowing for the principled selection of our default parameters which balance these competing objectives.}
 \label{fig:ablation_t_tc_lambda}
\end{figure*}

\section{Additional Ablation Studies}

\subsection{Additional Analysis of the Chord Control Field}

As established in the main paper, the fundamental hypothesis of our work is that the naive editing field (equivalent to ChordEdit with $\delta=0$) is inherently unstable for one-step integration. Our Chord Control Field (CCF) resolves this by introducing a temporal smoothing interval $\delta > 0$, which yields a stable, low-energy transport path.

To provide a comprehensive validation of this claim, we expand upon the LPIPS-CLIP Pareto analysis of the main text. We conduct an extensive hyperparameter sweep for both the naive baseline ($\delta=0$) and ChordEdit ($\delta \ne 0$) and evaluate the trade-off between semantic alignment (CLIP-Edited) and a wider array of background preservation metrics.

Figure~\ref{fig:scatter} visualizes the raw data distributions from this sweep, plotting performance against (left) Mean Squared Error (MSE), (center) Peak Signal-to-Noise Ratio (PSNR), and (right) Structural Similarity Index (SSIM). In all three scatter plots, the ChordEdit samples (red) are visibly concentrated in a superior performance region (e.g., lower MSE, higher PSNR/SSIM for a given CLIP score) compared to the naive baseline (blue), which exhibits a much wider, less stable, and fundamentally inferior distribution.

Figure~\ref{fig:pateto} plots the resulting Pareto frontiers from this data. These results unequivocally demonstrate that ChordEdit strictly Pareto-dominates the naive baseline across all three trade-off spaces. This confirms that the instability of the naive field is not limited to a single perceptual metric (like LPIPS) but is a fundamental flaw. By leveraging the temporally-smoothed, low-energy Chord Control Field, our method consistently achieves a superior and more robust performance envelope, enabling high semantic alignment and high structural preservation simultaneously.

\begin{figure*}
 \centering
 \includegraphics[width=0.7\linewidth]{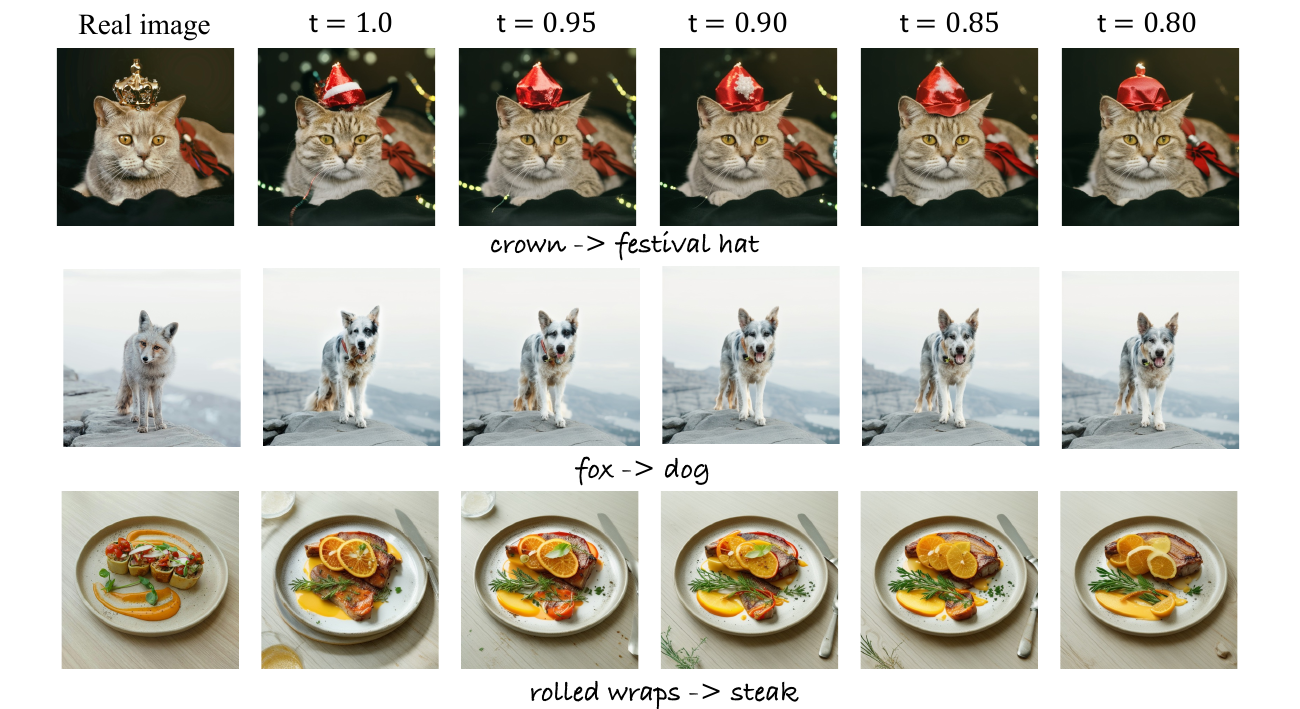}
 \caption{\textbf{Qualitative analysis of the main chord time $t$.} }
 \label{fig:ablation_t_show}
\end{figure*}

\begin{figure*}
 \centering
 \includegraphics[width=0.7\linewidth]{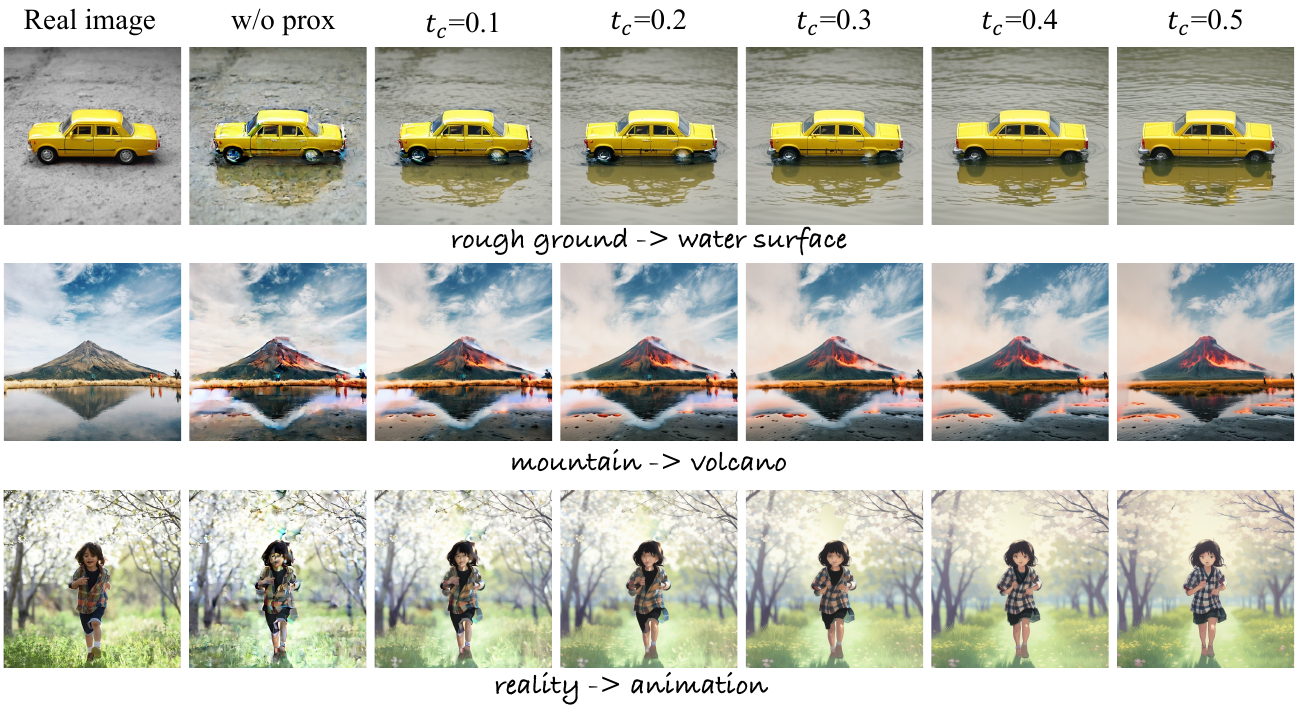}
 \caption{\textbf{Qualitative analysis of the proximal refinement time $t_c$.} }
 \label{fig:ablation_tc_show}
\end{figure*}

\begin{figure*}[t]
 \centering
 \begin{minipage}[t]{0.33\linewidth}
 \centering
 \includegraphics[width=\linewidth]{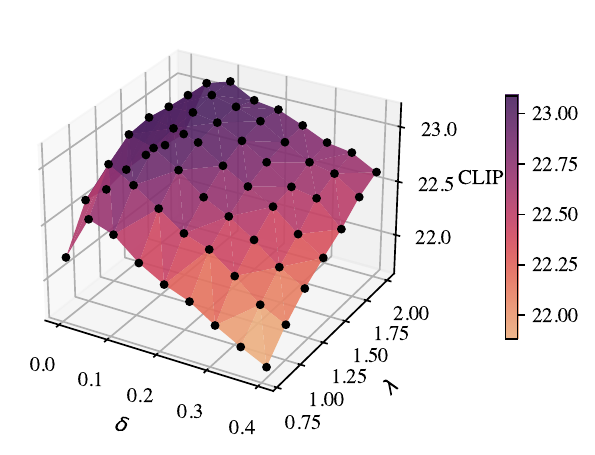} 
 \end{minipage}
 \begin{minipage}[t]{0.33\linewidth}
 \centering
 \includegraphics[width=\linewidth]{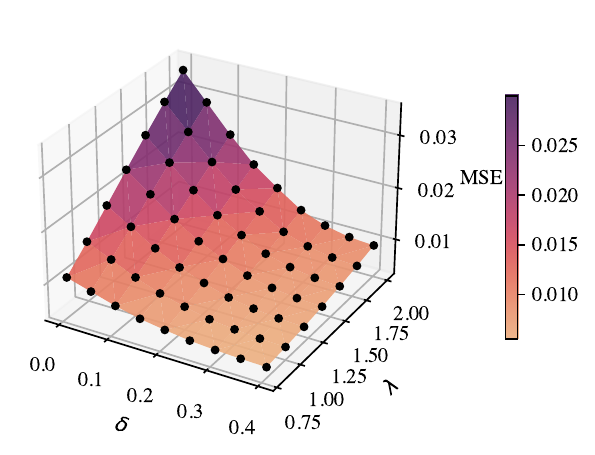} 
 \end{minipage}
 \begin{minipage}[t]{0.33\linewidth}
 \centering
 \includegraphics[width=\linewidth]{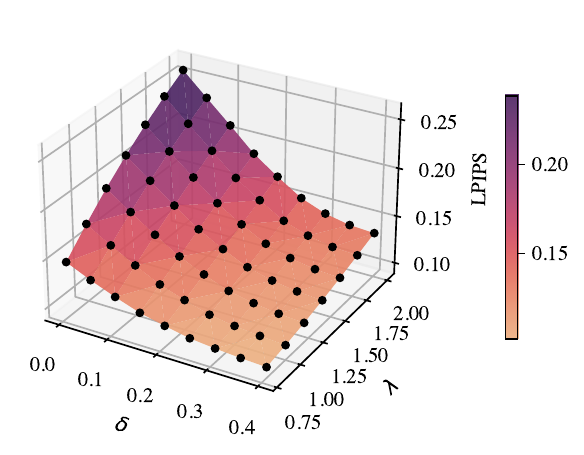}
 \end{minipage}
 \begin{minipage}[t]{0.33\linewidth}
 \centering
 \includegraphics[width=\linewidth]{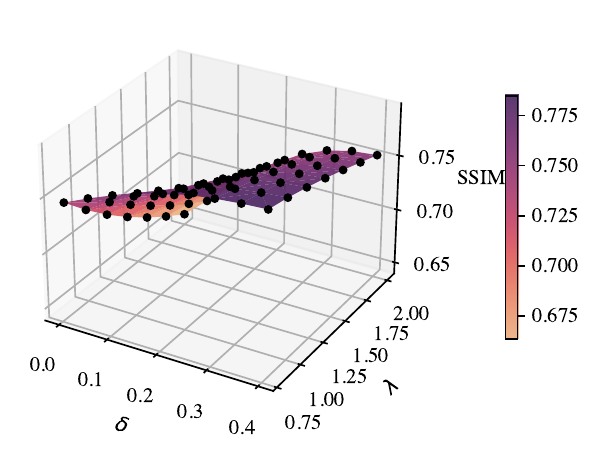}
 \end{minipage}
 \begin{minipage}[t]{0.33\linewidth}
 \centering
 \includegraphics[width=\linewidth]{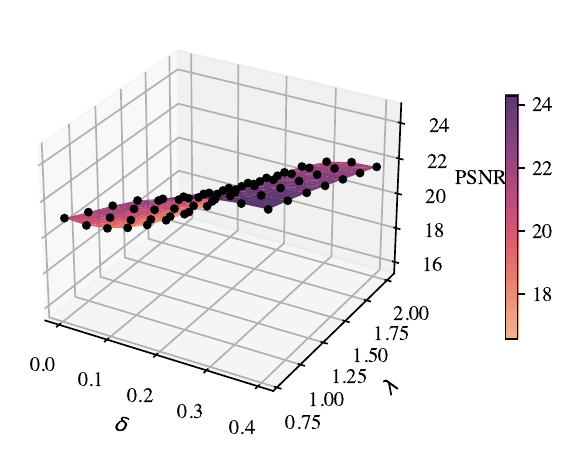}
 \end{minipage}
\caption{\textbf{Quantitative joint-analysis of $\delta$ and $\lambda$.} These 3D surface plots show the trade-off between semantic alignment (CLIP) and background preservation (LPIPS, MSE, SSIM, PSNR). The naive baseline ($\delta=0$, the front edge of each plot) is Pareto-inferior, suffering from low CLIP scores and high distortion (high LPIPS/MSE). Increasing $\delta$ (our temporal smoothing) robustly and monotonically improves all preservation metrics. Increasing $\lambda$ (step scale) robustly increases semantic strength (CLIP) at the cost of preservation. Our default parameters are chosen from this smooth, well-behaved trade-off space.}
 \label{fig:delta_scale}
\end{figure*}

\begin{figure*}
 \centering
 \includegraphics[width=0.7\linewidth]{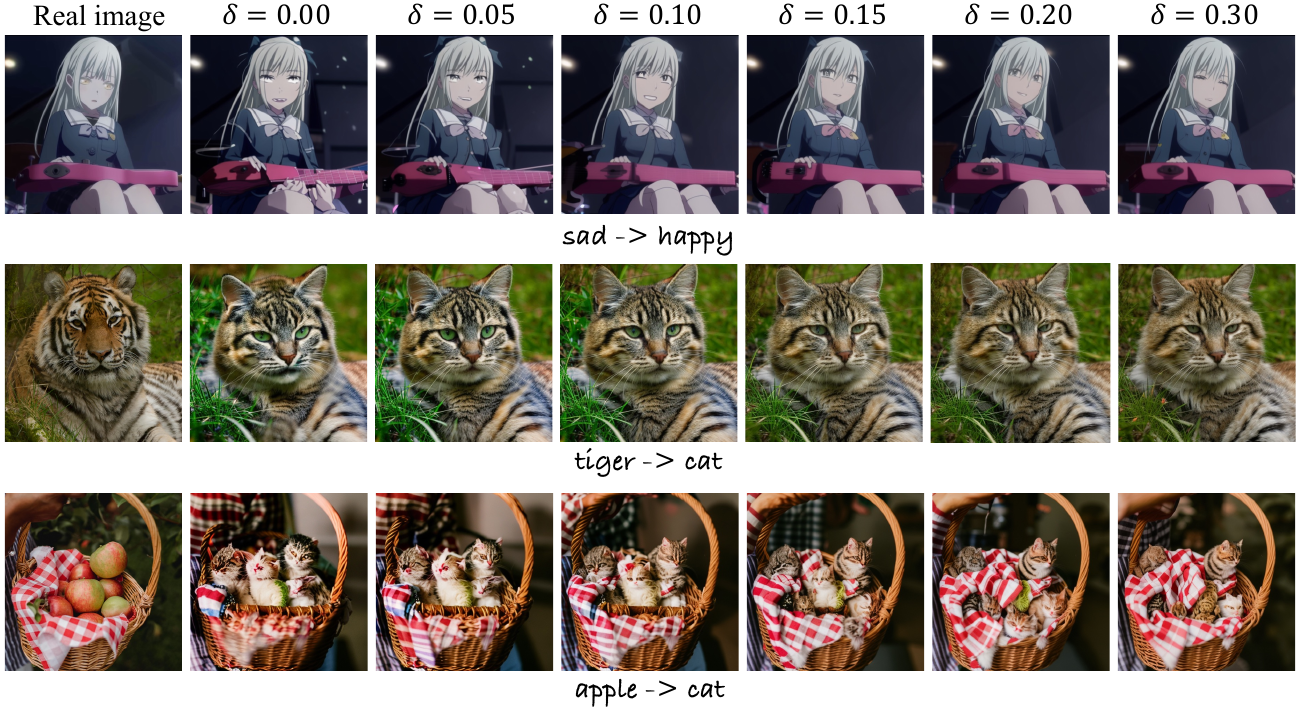}
 \caption{\textbf{Qualitative analysis of the temporal window $\delta$.} We fix other parameters and vary $\delta$. The naive baseline ($\delta=0.00$) consistently fails, producing severe artifacts, distortions, and structural collapse, especially on complex semantic changes (e.g., 'tiger $\to$ cat' or 'apple $\to$ cat'). Our Chord Control Field ($\delta > 0$) immediately stabilizes the edit, resolving these failures. A value of $\delta=0.15$ demonstrates a robust balance between edit stability and semantic strength, validating its choice as our default.}
 \label{fig:ablation_delta_show}
\end{figure*}

\begin{figure*}
 \centering
 \includegraphics[width=0.8\linewidth]{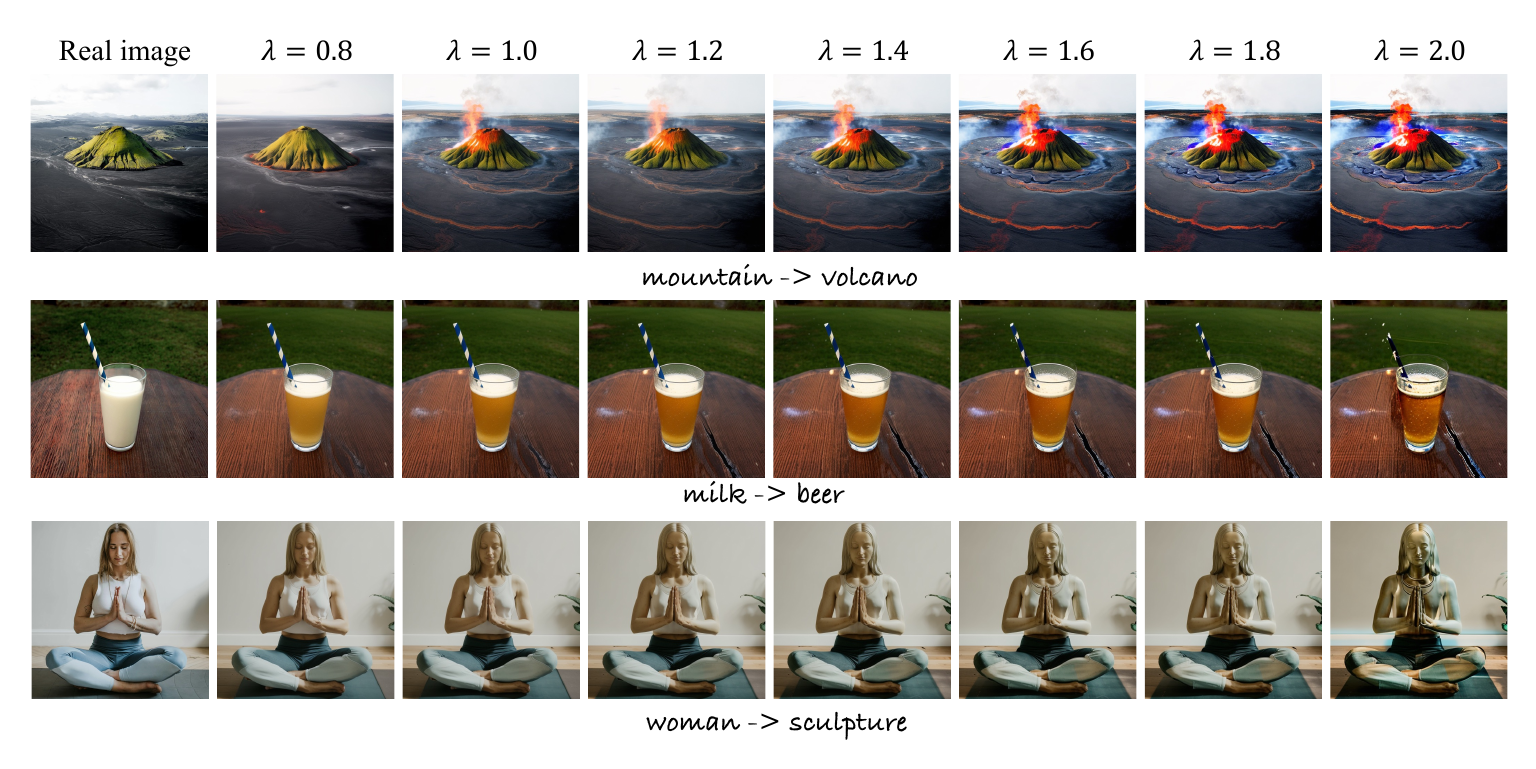}
 \caption{\textbf{Qualitative analysis of the step scale $\lambda$.} We vary $\lambda$ while keeping other parameters fixed. $\lambda$ functions as an intuitive 'edit strength' controller. Small values (e.g., $\lambda=0.8$) result in subtle, under-edited images ('mountain'). As $\lambda$ increases, the intensity of the target semantic ('volcano') becomes progressively stronger. This provides a simple and predictable knob for users to modulate the edit's impact.}
 \label{fig:ablation_lambda_show}
\end{figure*}

\subsection{Additional Analysis of Noise}

We validated the robustness of ChordEdit to random noise seeds and established that increasing the number of Monte Carlo (MC) samples ($n$) yields negligible marginal returns. This finding is attributed to the intrinsically low variance of the Chord Control Field, which achieves stability through temporal smoothing rather than costly MC averaging.

We expand this analysis here. Figure~\ref{fig:noise_mse} extends the Pareto-frontier analysis from the main text (which used LPIPS-CLIP) to cover a broader range of background preservation metrics: (left) Mean Squared Error (MSE), (center) Peak Signal-to-Noise Ratio (PSNR), and (right) Structural Similarity Index (SSIM).

Consistent with our primary findings, these plots demonstrate that the trade-off between semantic alignment (CLIP-Edited) and structural preservation (MSE/PSNR/SSIM) is virtually independent of the number of noise samples ($n$). The performance distributions (visualized by the scatter points) and their stability (implied by the confidence bands) are nearly indistinguishable whether using $n=1$ or multiple samples with our method. This stability is unique to our method; the naive baseline, in contrast, suffers from high intrinsic variance, causing its $n=1$ performance to be significantly less stable and worse than its multi-sample ($n>1$) configurations. This strongly confirms our hypothesis: the ChordEdit $n=1$ configuration is not a compromise but operates robustly at the optimal performance frontier. This result empirically justifies our default use of $n=1$ for all main experiments, achieving maximum efficiency without sacrificing quality or stability.

\subsection{Analysis of Temporal Parameters and Step Scale}

In addition to the core analysis of $\delta$ in the main paper, the performance of ChordEdit is jointly influenced by several key hyperparameters: the primary chord time $t$, the step scale $\lambda$, and the proximal refinement time $t_c$. Figure~\ref{fig:ablation_t_tc_lambda} presents a comprehensive ablation study to investigate the sensitivity and trade-offs associated with these parameters.

\paragraph{Analysis of Chord Time $t$.}
The left two panels of Figure~\ref{fig:ablation_t_tc_lambda} illustrate the impact of the main chord time $t$ on semantic alignment (CLIP-Edited $\uparrow$) and background preservation (LPIPS-Unedit $\downarrow$). We observe a distinct trade-off: increasing $t$ (e.g., from $0.80$ to $1.00$) generally yields stronger semantic alignment, as the model queries the field at a point closer to the final, fully-formed image manifold. However, this comes at the cost of slightly reduced background fidelity (higher LPIPS). The plots also reaffirm our central thesis: the naive baseline ($\delta=0.0$, black and blue lines) consistently occupies an inferior performance region (lower CLIP for a given LPIPS) compared to our smoothed ChordEdit configurations ($\delta > 0$). This quantitative trade-off is qualitatively visualized in Figure~\ref{fig:ablation_t_show}, which confirms that $t=0.90$ provides a robust balance between semantic strength and preservation.

\paragraph{Analysis of Refinement Time $t_c$.}
The right two panels of Figure~\ref{fig:ablation_t_tc_lambda} analyze the effect of the proximal refinement time $t_c$. The results show a strong, monotonic relationship: increasing $t_c$ from $0.1$ to $0.5$ robustly enhances semantic alignment (CLIP). This confirms the role of the proximal step in "sharpening" the edit to better match the target prior. However, this semantic gain is directly coupled with a degradation in background preservation (rising LPIPS), as a "stronger" refinement (higher $t_c$) is more prone to over-editing and affecting non-target regions. Figure~\ref{fig:ablation_tc_show} provides a clear visual example, showing how increasing $t_c$ strengthens the target semantic at a modest cost to fidelity, justifying our default choice of $t_c=0.30$.

\subsection{Analysis of $\delta$ and $\lambda$}

We conduct a detailed ablation study on the two most critical hyperparameters of the Chord Control Field: the temporal window size $\delta$ and the step scale $\lambda$. Figure~\ref{fig:delta_scale} presents a comprehensive quantitative analysis, visualizing the joint impact of $\delta$ and $\lambda$ on semantic alignment (CLIP) and a suite of background preservation metrics (LPIPS, MSE, SSIM, and PSNR).

The 3D surface plots reveal a clear and complex trade-off that validates our core hypothesis. Across all metrics, the naive baseline ($\delta=0$, the front edge of each plot) represents the worst-performing configuration, exhibiting the poorest semantic alignment (lowest CLIP) and the highest distortion (highest LPIPS/MSE, lowest SSIM/PSNR).

The impact of $\delta$ (temporal smoothing) is twofold:
\begin{enumerate}
 \item \textbf{On Preservation (Monotonic):} As $\delta$ increases, we observe a robust and monotonic improvement in all background preservation metrics (LPIPS/MSE $\downarrow$, SSIM/PSNR $\uparrow$). This confirms that temporal smoothing is fundamentally key to stabilizing the field and reducing distortion.
 \item \textbf{On Semantics (Non-Monotonic):} The effect of $\delta$ on semantic alignment (CLIP) is non-monotonic. Moving from $\delta=0$ (naive) to a small $\delta$ (e.g., $\approx 0.15-0.3$) \emph{improves} the CLIP score, as smoothing prevents the catastrophic semantic collapse of the naive baseline. However, as $\delta$ becomes too large (e.g., $\delta > 0.3$), the smoothing becomes "too conservative," overly suppressing the intended edit and \emph{reducing} the CLIP score.
\end{enumerate}
This explains the existence of an optimal $\delta$ "sweet spot" that balances stabilizing the edit (improving CLIP) against becoming overly conservative (harming CLIP). Concurrently, increasing $\lambda$ (the step scale) serves as a more direct control for edit strength, robustly increasing semantic alignment at the expected cost of decreased background fidelity. The smoothness of these surfaces demonstrates that these parameters offer a predictable and stable trade-off.

\begin{table*}[t]
\centering
\caption{Quantitative comparison of our method against other editing methods on PIE Bench.}
\label{tab:main_comparison_final_supp}
\resizebox{\textwidth}{!}{%
\begin{tabular}{@{}l l | c | cccc | cc | cccc@{}}
\toprule
\multirow{2}{*}{\textbf{Type}} & \multirow{2}{*}{\textbf{Method}} & \textbf{Struct. } & \multicolumn{4}{c|}{\textbf{Background Preservation}} & \multicolumn{2}{c|}{\textbf{CLIP Semantics}} & \multicolumn{4}{c}{\textbf{Efficiency}} \\
\cmidrule(lr){3-3} \cmidrule(lr){4-7} \cmidrule(lr){8-9} \cmidrule(lr){10-13}
& & Dist.${}_{\text{10}^3}\downarrow$ & \textbf{PSNR}$\uparrow$ & \textbf{MSE}${}_{\text{10}^3}\downarrow$ & \textbf{SSIM}${}_{\text{10}^2}\uparrow$ & \textbf{LPIPS}${}_{\text{10}^3}\downarrow$ & \textbf{Whole}$\uparrow$ & \textbf{Edited}$\uparrow$ & \textbf{Runtime(s)}$\downarrow$ & \textbf{Step}$\downarrow$ & \textbf{NFE}$\downarrow$ & \textbf{VRAM(MiB)}$\downarrow$ \\
\midrule
\multirow{5}{*}{\shortstack[c]{Multi-step \\ ($\ge$ 20 steps)}}
& DDIM + MasaCtrl & 28.79 & 21.25 & 8.58 & 80.11 & 106.59 & 24.13 & 21.13 & 55.19 & 50 & 100 & 12272 \\
& Direct Inversion + MasaCtrl & 24.46 & 21.78 & 7.99 & 81.74 & 87.38 & 24.42 & 21.38 & 79.10 & 50 & 100 & 12272 \\
& DDIM + PnP & 28.20 & 21.26 & 8.42 & 78.90 & 113.58 & 25.45 & 22.54 & 28.01 & 50 & 100 & 9262 \\
& Direct Inversion + PnP & 24.27 & 21.43 & 8.10 & 79.52 & 106.26 & 25.48 & 22.63 & 28.03 & 50 & 100 & 9262 \\
& FlowEdit (SD3) & 12.34 & 22.17 & 7.69 & 83.54 & 104.81 & \textbf{26.64} & \textbf{23.69} & 7.22 & 33 & 33 & 17140 \\
\midrule
\multirow{3}{*}{\shortstack[c]{Few-step \\ (4 steps)}}
& TurboEdit (SDXL-Turbo) & 13.80 & 21.44 & 9.49 & 80.08 & 108.60 & 24.66 & 21.79 & 2.69 & \underline{4} & 4 & 13826 \\
& InfEdit (SD1.4) & 17.06 & \textbf{24.14} & 6.82 & \textbf{85.02} & \underline{55.69} & 24.89 & 21.88 & 1.41 & \underline{4} & 4 & \underline{6502} \\
& InstantEdit (PeRFlow-SD1.5) & \underline{7.14} & \underline{23.80} & \textbf{4.21} & \underline{84.84} & 60.92 & 24.97 & 21.82 & 1.30 & \underline{4} & 8 & 16270 \\
\midrule
\multirow{9}{*}{\shortstack[c]{One-step}}
& SwiftEdit (SwiftBrush-v2) & 12.96 & 21.71 & 8.22 & 74.84 & 91.22 & 24.93 & 21.85 & \underline{0.54} & \textbf{1} & \underline{2} & 15060 \\
\cmidrule(lr){2-13}
& ChordEdit (Naive, InstaFlow) & 13.32 & 22.05 & 10.45 & 73.49 & 103.33 & 22.97 & 20.19 & \textbf{0.38} & \textbf{1} & \underline{2} & \textbf{6198} \\
& ChordEdit (Our, InstaFlow) & \textbf{6.33} & 23.05 & \underline{5.45} & 82.49 & \textbf{53.33} & 24.17 & 21.39 & \textbf{0.38} & \textbf{1} & \underline{2} & \textbf{6198} \\
\cmidrule(lr){2-13}
& ChordEdit (Naive, SwiftBrush-v2) & 16.33 & 20.52 & 17.17 & 73.42 & 127.43 & 23.78 & 21.06 & \textbf{0.38} & \textbf{1} & \underline{2} & 6988 \\
& ChordEdit (Our, SwiftBrush-v2) & 12.96 & 22.04 & 7.13 & 75.84 & 111.22 & 25.12 & 22.58 & \textbf{0.38} & \textbf{1} & \underline{2} & 6988 \\
\cmidrule(lr){2-13}
& ChordEdit (Naive, SD-Turbo) & 25.44 & 21.38 & 9.73 & 74.39 & 131.30 & 25.11 & 21.96 & \textbf{0.38} & \textbf{1} & \underline{2} & 6988 \\
& ChordEdit (Naive w/o prox, SD-Turbo) & 19.18 & 21.89 & 10.84 & 77.24 & 105.27 & 23.68 & 20.83 & \textbf{0.20} & \textbf{1} & \textbf{1} & 6988 \\
& \textbf{ChordEdit (Ours, SD-Turbo)} & 16.58 & 22.20 & 6.84 & 75.91 & 128.25 & \underline{25.58} & \underline{22.96} & \textbf{0.38} & \textbf{1} & \underline{2} & 6988 \\
& \textbf{ChordEdit (Ours w/o prox, SD-Turbo)} & \underline{10.37} & 23.89 & 5.05 & 81.24 & 88.36 & 24.97 & 21.87 & \textbf{0.20} & \textbf{1} & \textbf{1} & 6988 \\
\bottomrule
\end{tabular}%
}
\end{table*}

\section{More Quantitative Results}
\label{sec:more_quant_results}

Table~\ref{tab:main_comparison_final_supp} provides the full, unabridged quantitative results that underpin the claims made in the main paper. This detailed table expands upon the main paper's Table~1 by including additional, fine-grained metrics for background and structural preservation: Structural Distance (Struct. Dist.), Structural Similarity Index (SSIM), and Perceptual Similarity (LPIPS). This data provides a comprehensive view of our method's performance and allows for a deeper analysis.

\paragraph{Validating the Chord Control Field.}
The primary claim of our work is that the naive single-step editing field ($\delta=0$) is unstable, and our Chord Control Field ($\delta > 0$) fundamentally resolves this. The detailed metrics in Table~\ref{tab:main_comparison_final_supp} provide overwhelming evidence for this claim. When comparing our method (Ours) against the baseline (Naive) across all three models, our method consistently demonstrates superior structural and perceptual fidelity. The Naive rows show significantly higher structural distortion and perceptual error compared to our Ours rows, especially in the LPIPS and Struct. Dist. columns. This trend holds true across all preservation metrics, proving that our temporal smoothing (when $\delta > 0$) is critical for preserving background consistency, not just in pixel space (PSNR), but also in structural (SSIM) and deep-feature (LPIPS) space.

\paragraph{Decoupling Transport (NFE=1) and Refinement (NFE=2).}
The full table explicitly demonstrates our framework's core design: the decoupling of consistency-preserving transport (NFE=1) from semantic-boosting refinement (+1 NFE). This addresses any concerns regarding the NFE=2 configuration. Let us compare the \textbf{Ours (w/o prox, SD-Turbo)} (NFE=1) variant against the full \textbf{Ours (SD-Turbo)} (NFE=2) variant. The data shows a clear and intentional trade-off. The \textbf{Ours (w/o prox)} variant consistently achieves the best scores across the full suite of background and structural preservation metrics, including Struct. Dist., PSNR, SSIM, and LPIPS. This is the pure, low-energy transport. By adding the optional refinement step, the full \textbf{Ours} (NFE=2) variant trades a predictable amount of this preservation to achieve a significant boost in semantic alignment, as measured by the CLIP-Edited score. This data strongly supports our claim that \textbf{ChordEdit} is a modular framework. Users can choose the NFE=1 variant for maximum fidelity and true one-step performance, or the NFE=2 variant for the best overall semantic alignment, fully vindicating the design presented in the main paper.

\paragraph{Efficiency and Model Agnosticism.}
Finally, the table confirms our claims of efficiency and broad applicability. The performance gains of \textbf{Ours} over \textbf{Naive} are consistent across \textbf{InstaFlow}, \textbf{SwiftBrush-v2}, and \textbf{SD-Turbo}, confirming the model-agnostic nature of our Chord Control Field. Furthermore, our VRAM footprint is shown to be exceptionally low, representing a significant practical advantage over memory-intensive methods like \textbf{SwiftEdit} or \textbf{FlowEdit}, making our method far more accessible for real-time applications.

\section{More Qualitative Results}

We provide additional qualitative comparisons in Figure~\ref{fig:supp_sota_1} and Figure~\ref{fig:supp_sota_2}. These examples further support the quantitative findings in the main paper. 

Across a diverse set of editing prompts, our method consistently produces high-fidelity results that adhere to the target prompt while maintaining exceptional background preservation. This contrasts sharply with multi-step methods, which often introduce undesirable artifacts or fail to preserve the subject's identity (e.g., \textit{Direct Inversion+PnP}), and other few-step methods that struggle to balance semantic accuracy with structural consistency. As demonstrated, our method successfully avoids the catastrophic distortions and background collapse seen in naive one-step approaches, validating the stability of our Chord Control Field.

\section{Societal Impacts}
\label{sec:societal_impacts}

The development of high-fidelity, real-time generative models like ChordEdit presents significant opportunities while also necessitating a discussion of societal and ethical considerations.

\paragraph{Positive Applications and Broader Impacts.}
Our primary motivation for developing ChordEdit is to democratize high-end creative tools. The method's core advantages---its speed, resource efficiency (low VRAM), and model-agnostic, training-free nature---make powerful, real-time generative editing accessible to a broader audience, including those without specialized high-end hardware. We envision our work empowering artists, designers, content creators, and hobbyists by providing an intuitive and responsive tool for rapid prototyping, creative exploration, and concept visualization. For example, a designer could instantly visualize different seasonal aesthetics for a landscape (e.g., "fall" to "spring"), or a casual user could easily modify personal photos in real-time. This reduces the barrier to entry for complex image manipulation, fostering greater creative expression.

\paragraph{Ethical Considerations and Potential for Misuse.}
Like all high-fidelity generative models, ChordEdit carries the risk of misuse. As acknowledged in the main paper, the ability to create realistic and consistent edits in real-time could be exploited to generate deceptive content, spread misinformation, or create malicious media. The high quality and structural preservation of our method might make such forgeries more convincing.

Furthermore, ChordEdit is a training-free method that operates on pre-trained text-to-image models. It does not, by itself, correct any inherent societal biases (e.g., related to race, gender, or culture) that may be present in these foundational models. As such, edits performed by ChordEdit may reflect or even amplify these underlying biases, depending on the prompts and the backbone model used.

\paragraph{Mitigation and Author Statement.}
We strongly condemn the use of our technology for any deceptive or harmful purpose. Our work is intended for creative and assistive applications, aimed at augmenting human creativity, not replacing it or enabling deception. We believe the best mitigation strategy lies in the concurrent development of robust detection tools for synthetic media, as well as in fostering public awareness and critical media literacy. We encourage the research community to continue to prioritize the development of ethical guidelines and safeguards alongside the advancement of generative capabilities.

\begin{figure}[t]
 \centering
 \includegraphics[width=0.9\linewidth]{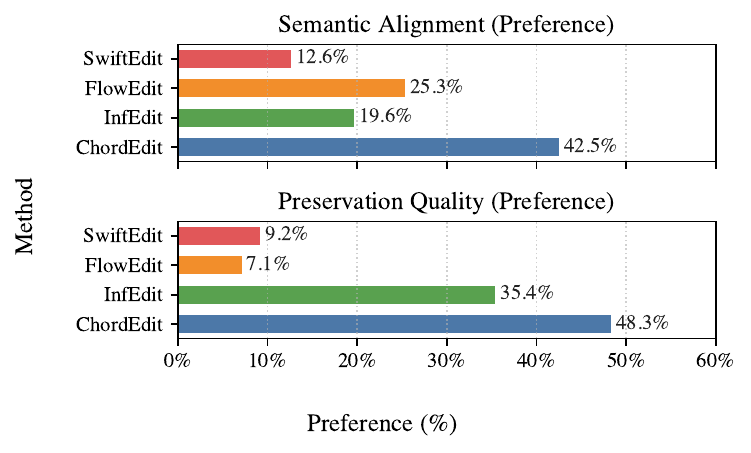}
 \caption{\textbf{User Study Results.} Aggregated human preference rates from a four-way blind comparison, matching the data cited in the main paper. ChordEdit was the clear winner in \emph{both} Semantic Alignment (42.5\%) and Preservation Quality (48.3\%), demonstrating its superior overall performance.}
 \label{fig:user_study_results}
\end{figure}

\begin{algorithm}[t]
\caption{ChordEdit (Multi-noise $n > 1$ version)}
\label{alg:chordedit_multi_noise}
\begin{algorithmic}[1]
\State \textbf{Inputs:} source image $x_{\rm src}$; prompts $c_{\rm src}, c_{\rm tar}$; step time $t$; window $\delta$; step scale $\lambda$; Proximal Refinement time $t_{\rm c}$; \textbf{number of noise samples $n$}.
\State \textbf{Output:} edited image $x_{\rm tar}$.
\State \textbf{Init:} $x_{\rm in} \leftarrow x_{\rm src}$
\State $\hat u_{\rm sum} \leftarrow 0$ 

\For{$i = 1$ \textbf{to} $n$}
 \State $\mathbf{R}_{t-\delta} \leftarrow \mathbf{R}(x_{\rm in}, t-\delta)$ 
 \State $R_t \leftarrow R(x_{\rm in}, t)$
 
 \State $\hat u_i \leftarrow \frac{t\,\mathbf{R}_{t-\delta} + \delta\,\mathbf{R}_t}{t+\delta}$
 \State $\hat u_{\rm sum} \leftarrow \hat u_{\rm sum} + \hat u_i$
\EndFor

\State $\hat u_{\rm avg} \leftarrow \hat u_{\rm sum} / n$ 
\State $x^{\rm pred} \leftarrow x_{\rm in} + \lambda\,\hat u_{\rm avg}$ 
\State $x_{\rm tar} \leftarrow \operatorname{prox}\left(x^{\mathrm{pred}},\, t_{\mathrm c},\, c_{\mathrm{tar}}\right)$
\State \textbf{Return} $x_{\rm tar}$
\end{algorithmic}
\end{algorithm}

\section{User Study}
\label{sec:user_study}

To complement our quantitative analyses, we conducted a formal user study to assess human perceptual preference. An example of the evaluation form shown to participants is provided in Figure~\ref{fig:user_study_from}. Automated metrics often fail to capture the holistic "quality" or "naturalness" of an edit, so this study was designed to validate our core claim: ChordEdit produces edits that are semantically accurate, and also more structurally consistent and artifact-free than competing methods.

For the study setup, we recruited 150 participants with diverse backgrounds. We presented them with a four-way blind comparison, showing the original image, a text prompt, and four edited results from ChordEdit, InfEdit, FlowEdit, and SwiftEdit. The order of all results was randomized to prevent bias.

We asked participants to vote for the single best image based on two independent criteria. The first was Semantic Alignment, judging which image best matched the text prompt's meaning. The second was Preservation Quality, judging which image looked most natural and best preserved the background and non-edited regions, with the fewest artifacts.

The study yielded 4,500 total votes (150 participants $\times$ 30 prompts) for each criterion. The results, shown in Figure~\ref{fig:user_study_results}, confirm the data cited in the main paper and show a clear preference for ChordEdit.

For Semantic Alignment, ChordEdit was the clear winner, preferred in \textbf{42.5\%} of comparisons. This significantly outperformed FlowEdit (25.3\%), while InfEdit (19.6\%) and SwiftEdit (12.6\%) lagged considerably.Proposition 4

For Preservation Quality, ChordEdit also achieved the top position, securing \textbf{48.3\%} of the vote. This result is particularly compelling, as it shows our method was perceived as more stable and artifact-free than even InfEdit (35.4\%), a baseline renowned for its high preservation. FlowEdit (7.1\%) and SwiftEdit (9.2\%) were frequently penalized by users for artifacts and distortion.

In conclusion, the user study strongly validates our claims. ChordEdit was the only method to rank first in both categories. Other methods force a compromise—excelling at either semantics (FlowEdit) or preservation (InfEdit) but failing at the other. ChordEdit was the most preferred for both, confirming it provides the best overall perceptual quality and proving our low-energy, stable transport field translates directly to the most desirable result for human observers.

\section{ChordEdit algorithm with multi-noise}

For maximum efficiency, our core ChordEdit algorithm presented in the main paper uses a single noise sample ($n=1$). However, our framework can be directly extended to support multiple Monte Carlo (MC) noise samples ($n > 1$) to theoretically further reduce the estimation variance.

Algorithm \ref{alg:chordedit_multi_noise} provides the pseudocode for this multi-noise version. The key difference is that we first compute the Chord Control Field $\hat{u}_i$ independently for $n$ different noise samples, then average these fields, and finally use this averaged field $\hat{u}_{\rm avg}$ to perform the single-step transport and subsequent proximal refinement.

\onecolumn

\section{Symbols Table}

\begin{longtable}[t]{@{}p{0.25\textwidth} p{0.7\textwidth}@{}}
\caption{Symbols Table} \label{tab:symbol_glossary} \\
\toprule
\textbf{Symbol} & \textbf{Description} \\
\midrule
\endfirsthead

\multicolumn{2}{c}%
{{\tablename\ \thetable{} -- continued from previous page}} \\
\toprule
\textbf{Symbol} & \textbf{Description} \\
\midrule
\endhead

\bottomrule
\multicolumn{2}{r}{{Continued on next page}} \\
\endfoot

\endlastfoot

$t \in [0, 1]$ & Time variable for the probability flow. $t=1$ corresponds to the source data distribution, and $t=0$ to the target noise distribution. \\
$x_t$ & The image state at time $t$. \\
$c$ & The text condition (prompt). \\
$c_{\rm src}, c_{\rm tar}$ & The source and target text prompts, respectively. \\
$x_{\rm src}, x_{\rm tar}$ & The source image ($x_1$) and the desired target image ($x_0$), respectively. \\
$p_t(x \mid c)$ & The probability distribution of $x_t$ at time $t$ conditioned on $c$. \\
$p_1, p_0$ & The source data distribution ($t=1$) and the target noise distribution ($t=0$). \\
$z$ & A synthetic noisy proxy state used to query the model. \\
$x_\tau$ & The editing anchor, fixed to the clean source image $x_1$ (i.e., $x_\tau := x_{\rm src}$). \\

$v(x_t, t, c)$ & The drift of the conditional probability flow induced by the pre-trained T2I model. \\
$\Delta v(x_t, t)$ & The instantaneous residual field, defined as $v(\cdot, c_{\rm tar}) - v(\cdot, c_{\rm src})$. \\
$Q(z, t, c)$ & The model's observable output (e.g., noise prediction, velocity) at noisy state $z$. \\
$\Delta Q(z, t)$ & The conditional residual of the observable, $Q(\cdot, c_{\rm tar}) - Q(\cdot, c_{\rm src})$. \\
$\hat\epsilon_\theta, \mathbf{v}_\theta$ & The predicted noise and predicted velocity, respectively, from a model. \\
$K_t(\cdot \mid x_\tau)$ & The forward noising kernel that maps the anchor $x_\tau$ to a noisy state $z$ at time $t$. \\
$\mathcal{B}_t$ & A time-only linear map that projects the model's output $Q$ into a unified comparison domain (velocity units). \\
$A_t^{(\epsilon)}, A_t^{(x_0)}, A_t^{(v)}, A_t^{(\mathrm{score})}$ &
Specific coefficients defining $\mathcal{B}_t$ for noise, $x_0$, $v$-, and score-to-drift parameterizations. \\
$\alpha(t), \sigma(t)$ & Coefficients of the noise schedule for the forward path $x_t = \alpha(t)x_0 + \sigma(t)\epsilon$. \\
$\beta(t)$ & Continuous-time noise schedule parameter (related to $\alpha(t)$). \\

$u_t(x)$ & The ideal, low-energy editing vector field that drives the transport from $\rho_1$ to $\rho_0$. \\
$\mathbf{R}(x_\tau, t)$ & The observable proxy field; the expected value of the mapped observable residual ($\mathbb{E}[\mathcal{B}_t \Delta Q]$). \\
$\varepsilon_t, \eta(t)$ & Zero-mean noise terms in the measurement model $\mathbf{R} = u_t + \varepsilon_t$. \\
$u_{\rm nai}$ & The naive control field, which simply uses the proxy field: $u_{\rm nai} = \mathbf{R}$. \\
$\hat u_t(x_\tau)$ & The Chord Control Field: a locally smoothed, low-energy estimator for $u_t$. \\
$\delta$ & The temporal window size for smoothing, $[t-\delta, t]$. \\
$\Phi_t(u; x_\tau)$ & The strictly convex quadratic surrogate objective minimized to find the Chord Control Field. \\
$u_t^\star(x_\tau)$ & The exact, integral-form minimizer of $\Phi_t$. \\
$K_\delta(s)$ & The causal smoothing kernel that defines $\hat u$ as a convolution: $\hat u \approx K_\delta * \mathbf{R}$. \\
$\lambda$ & The step scale, controlling the magnitude of the applied edit transport. \\
$x^{\rm pred}$ & The predicted image after the single transport step. \\
$\operatorname{prox}(\cdot)$ & The optional proximal refinement step. \\
$t_c$ & The time parameter used for the proximal refinement step. \\
$n$ & The number of Monte Carlo noise samples used in the estimation. \\

$\rho_t(x)$ & The transport density, evolving from $\rho_1 = p_1(\cdot \mid c_{\rm src})$ to $\rho_0 = p_0(\cdot \mid c_{\rm tar})$. \\
$\mathcal{E}[u; \rho]$ & The Benamou–Brenier kinetic energy functional, $\int \int \frac{1}{2}\|u_t(x)\|^2 \rho_t(x) dx dt$. \\
$\bar{E}$ & The discrete, unweighted Benamou–Brenier kinetic energy. \\
$u^\star$ & The true, energy-minimizing optimal control field. \\
$h$ & The step size for numerical integration (in one-step editing, $h=1$). \\
$\tau_{n+1}$ & The one-step local truncation error of the numerical integrator. \\
$f(x, t)$ & The editing vector field of the ODE, $f(x,t)=u(x,t)$. \\
$M_f$ & Bound on the derivatives of $f$, related to local error. \\
$\mathcal{C}(u; \mathcal{U})$ & A computable proxy for the consistency constant. \\
$C_{\rm cho}, C_{\rm nai}$ & The consistency constants of the underlying ODE for the Chord and Naive fields. \\
$L_u, M_u$ & Bounds on the spatial and temporal derivatives of $u$, used for global error analysis. \\
$e_n^u$ & The global error of the numerical solution at step $n$. \\
$P_\delta$ & The $L^2_\rho$-orthogonal projection onto the subspace of "chord" functions. \\
$S$ & The total number of discrete integration steps (step count). \\
$s$ & The step index, $s \in \{1,\dots,S\}$. \\
\bottomrule
\end{longtable}

\begin{figure*}[b]
 \centering
\fbox{\includegraphics[width=0.65\linewidth]{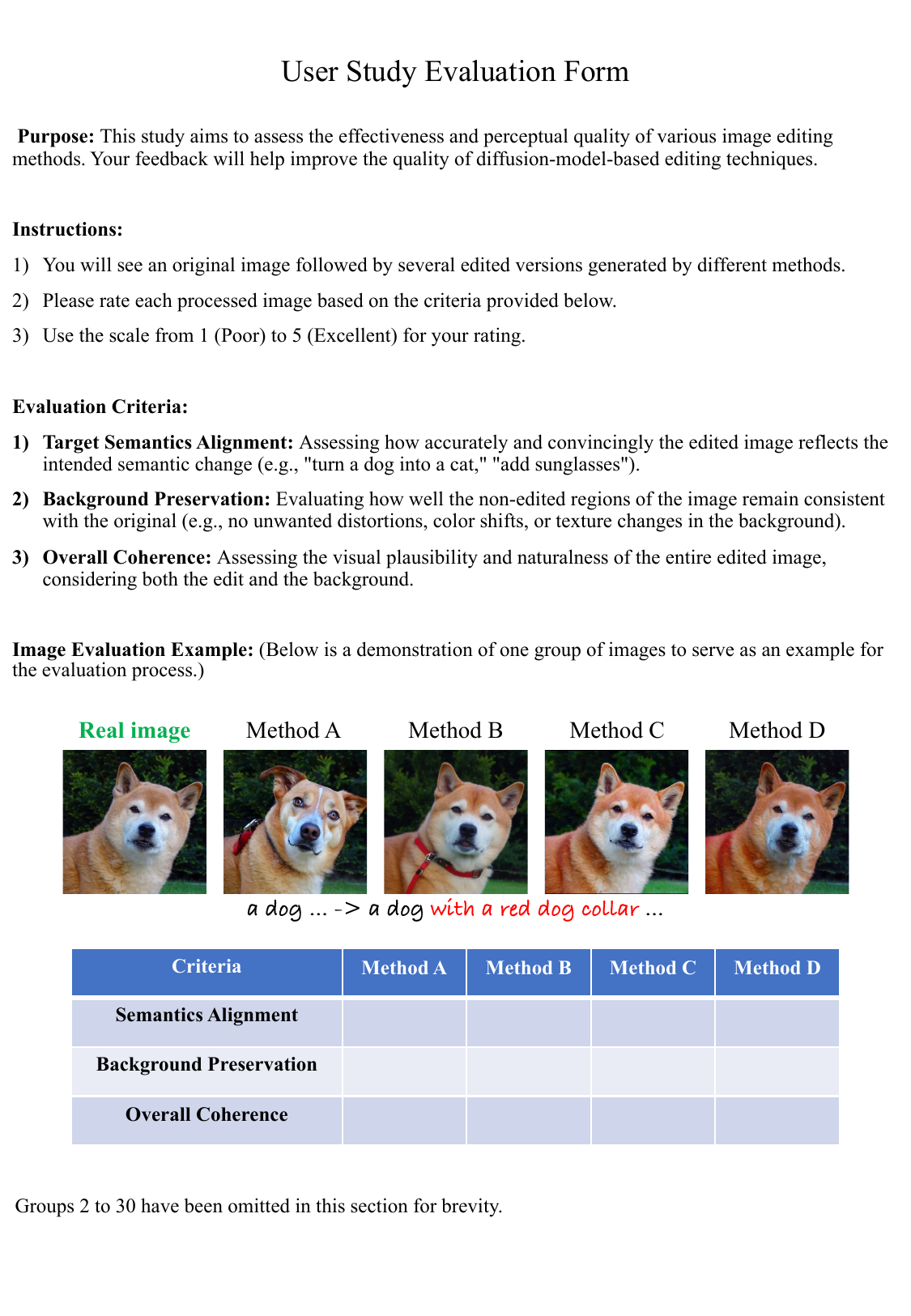}}
 \caption{\textbf{User Study Form Example.}}
 \label{fig:user_study_from}
\end{figure*}

\begin{figure*}[t]
 \centering
 \includegraphics[width=\linewidth]{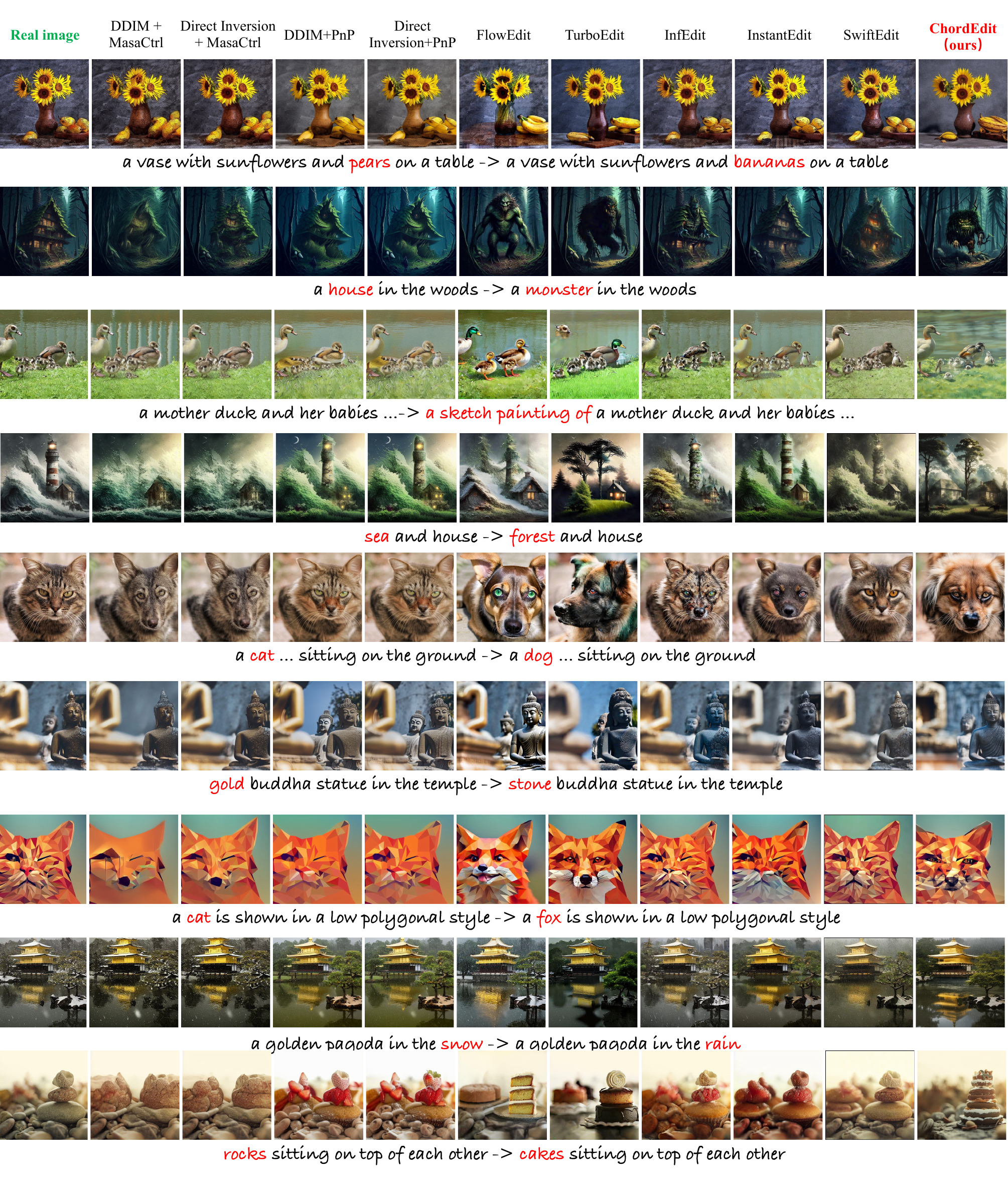}
 \caption{Comparison of Methods.}
 \label{fig:supp_sota_1}
\end{figure*}

\begin{figure*}[t]
 \centering
 \includegraphics[width=\linewidth]{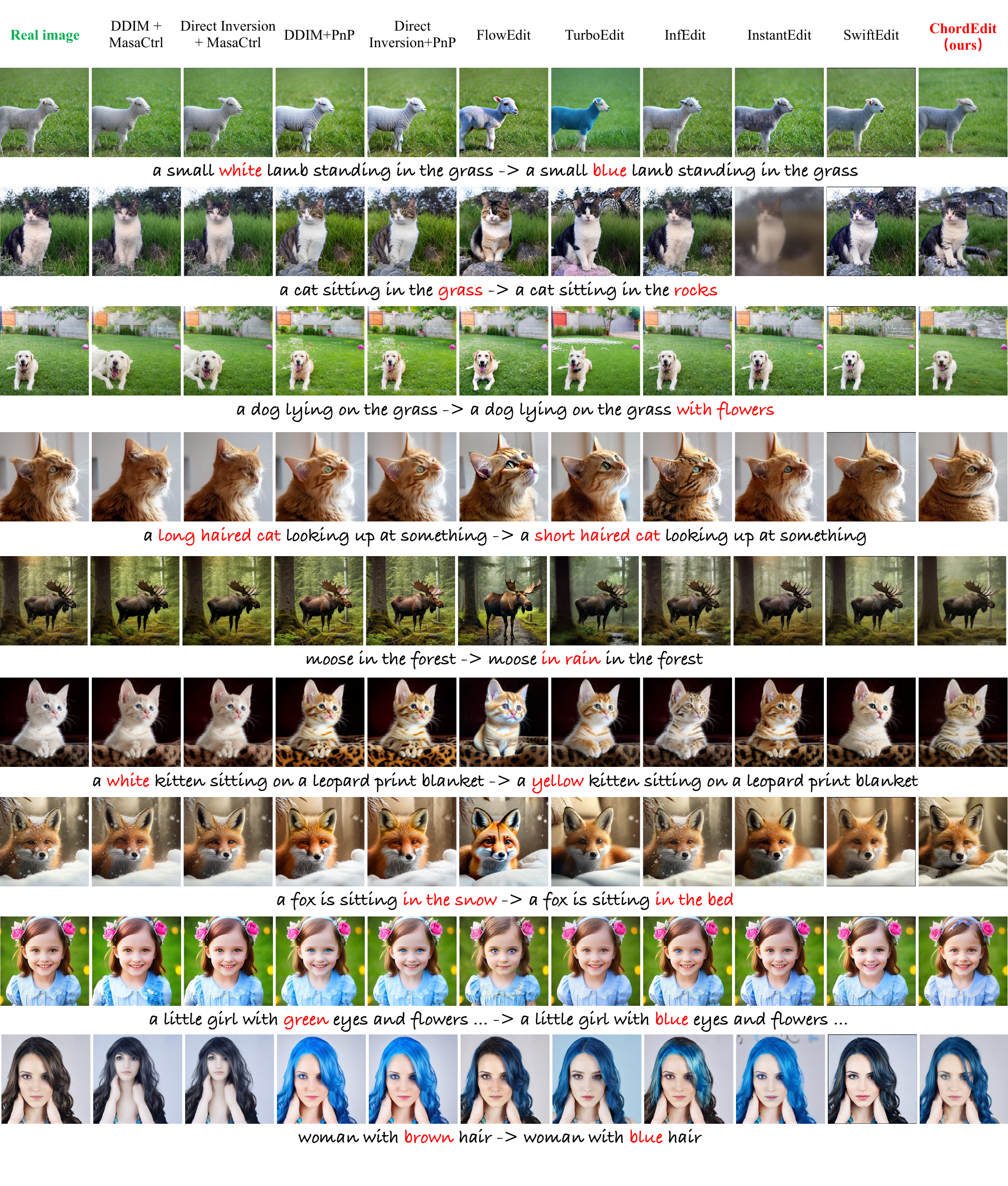}
 \caption{Comparison of Methods.}
 \label{fig:supp_sota_2}
\end{figure*}


\end{document}